\PassOptionsToPackage{table}{xcolor}
\documentclass{article}
\usepackage{iclr2027_conference,times}
\usepackage{amsmath,amsfonts,bm}

\def\eqref#1{equation~\ref{#1}}
\def\1{\bm{1}}

\DeclareMathAlphabet{\mathsfit}{\encodingdefault}{\sfdefault}{m}{sl}
\SetMathAlphabet{\mathsfit}{bold}{\encodingdefault}{\sfdefault}{bx}{n}

\usepackage{graphicx}
\usepackage{booktabs}
\usepackage{longtable}
\usepackage{amsfonts}
\usepackage{amsmath}
\usepackage{amssymb}
\usepackage{amsthm}
\usepackage{algorithm}
\usepackage{algorithmic}
\usepackage{multirow}
\usepackage{array}
\usepackage{enumitem}
\usepackage{subcaption}
\usepackage{float}
\usepackage{placeins}

\usepackage{xcolor}
\usepackage{colortbl}
\definecolor{micoRose}{HTML}{F6E0E6}
\colorlet{micoPurple}{micoRose}
\usepackage{soul}
\usepackage{nicefrac}
\usepackage{caption}
\usepackage{hyperref}
\hypersetup{hidelinks}
\usepackage{url}
\usepackage{tikz}
\usetikzlibrary{arrows.meta,positioning}

\renewcommand{\arraystretch}{0.94}

\newsavebox{\micoTableBox}
\newcommand{\micoTableFit}[1]{%
  \sbox{\micoTableBox}{#1}%
  \ifdim\wd\micoTableBox>\linewidth
    \resizebox{\linewidth}{!}{\usebox{\micoTableBox}}%
  \else
    \usebox{\micoTableBox}%
  \fi
}

\newtheorem{proposition}{Proposition}

\let\cite\citep



\definecolor{micoTitleRed}{HTML}{E7776B}
\definecolor{micoTitleGold}{HTML}{D6A127}
\newlength{\micoMarkHeight}

\DeclareRobustCommand{\micoTitleMark}{%
  \begingroup
  \settoheight{\micoMarkHeight}{M}%
  \raisebox{-0.65\micoMarkHeight}{%
    \includegraphics[height=2.3\micoMarkHeight]{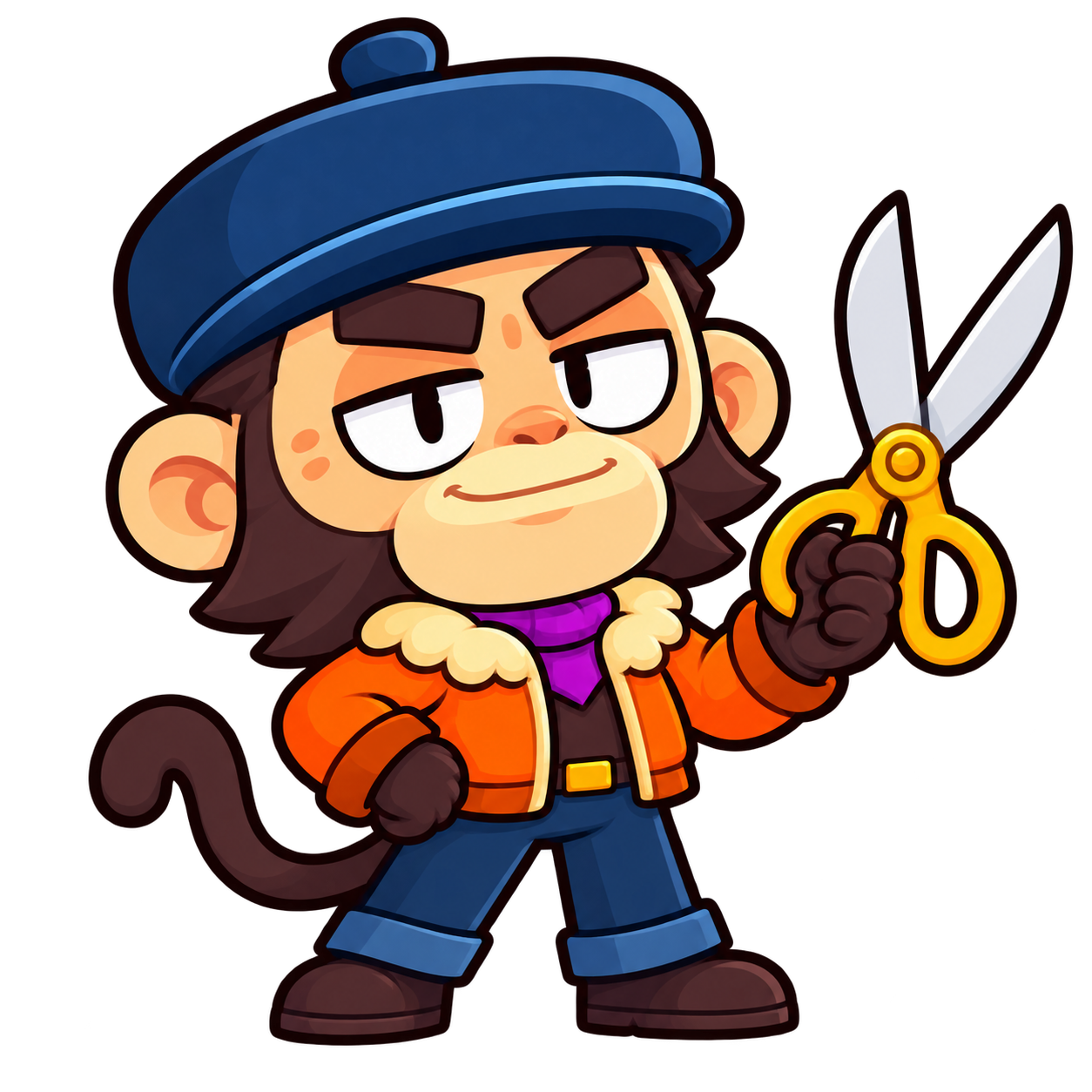}%
  }%
  \endgroup
}
\title{\texorpdfstring{\mbox{\micoTitleMark\hspace{0.22em}\textcolor{micoTitleRed}{Mi}\textcolor{micoTitleGold}{Co}}}{MiCo}: Mutual Information Coverage Optimization through Semantic Erasure Modeling for Efficient MLLM Inference}
\author{%
{\small\textbf{%
Tinghao Wang$^{1,2,*}$,
Yichen Guo$^{1,3,*}$,
Qizhe Zhang$^{1}$,
Yuan Zhang$^{1}$,
Weimin Ouyang$^{1}$,
}}\\
{\small\textbf{%
Rui Huang$^{4}$,
Jiajun Cao$^{1}$,
Sixiang Chen$^{1}$,
Hao Jiang$^{1,2}$,
Jixian Wu$^{5}$,
}}\\
{\small\textbf{%
Zheng Lu$^{2}$,
Bofan Zhu$^{2}$,
Renyuan Li$^{2}$,
Shanghang Zhang$^{1,\dagger}$
}}\\
{\small $^{1}$State Key Laboratory of Multimedia Information Processing, School of Computer Science, Peking University}\\
{\small $^{2}$University of Electronic Science and Technology of China, $^{3}$Nanyang Technological University,}\\
{\small $^{4}$The University of Hong Kong,
$^{5}$New York University}
}
\iclrfinalcopy
\begin{document}
\raggedbottom

\maketitle

\makeatletter
\begingroup
\renewcommand{\thefootnote}{}
\renewcommand{\@makefntext}[1]{\noindent#1}
\footnotetext{$^{*}$Equal contribution.
\quad $^{\dagger}$Corresponding author.}
\endgroup
\makeatother

\vspace{-3pt}
\begin{abstract}
Multimodal large language models (MLLMs) have demonstrated impressive performance in multimodal understanding, but processing large numbers of visual tokens results in high computational costs. While many methods have been proposed to reduce the number of visual tokens, most of them rely on heuristics and are prone to discarding substantial visual information during pruning, leading to degradation in model performance. In this work, by using a semantic erasure model, we derive a general mutual information coverage objective from task log-loss and propose MiCo, a training-free two-stage pruning method. MiCo first uses visual signals to select a representative candidate pool before visual tokens enter the language model, then performs task-aware subset selection within it. At each stage, suitable observable proxies instantiate the derived objective as a monotone submodular coverage function, which MiCo greedily optimizes under the token budget. MiCo is evaluated on diverse MLLMs ranging from 7B to 13B parameters across a broad range of image and video benchmarks spanning general visual reasoning, fine-grained OCR and grounding, hallucination detection, and long-video understanding. MiCo consistently achieves the best performance across nearly all evaluated models under all pruning ratios. On LLaVA-NEXT-13B, MiCo uses only 5.6\% visual tokens, retains 97.5\% of baseline performance, and achieves a 3.8$\times$ inference speedup. Our experiments demonstrate the effectiveness of MiCo and our mutual information coverage objective for visual token pruning. Our code is available at \url{https://github.com/MAX-7237/MiCo}.
\end{abstract}

\vspace{-3pt}
\begin{figure}[H]
    \centering
    \captionsetup{skip=4pt}
    \captionsetup[subfigure]{skip=2pt}
    \begin{subfigure}[b]{0.30\linewidth}
        \centering
        \includegraphics[width=\linewidth,height=1.70in,keepaspectratio]{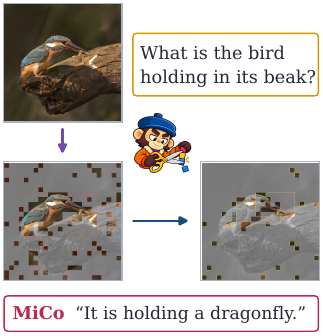}
        \caption{MiCo's pruning}
        \label{fig:teaser_process}
    \end{subfigure}\hfill
    \begin{subfigure}[b]{0.35\linewidth}
        \centering
        \includegraphics[width=\linewidth,height=1.70in,keepaspectratio]{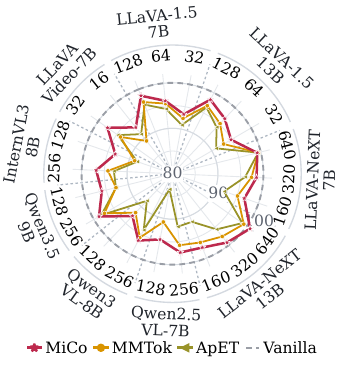}
        \caption{Cross-model performance}
        \label{fig:main_results_radar}
    \end{subfigure}\hfill
    \begin{subfigure}[b]{0.33\linewidth}
        \centering
        \includegraphics[width=\linewidth,height=1.70in,keepaspectratio]{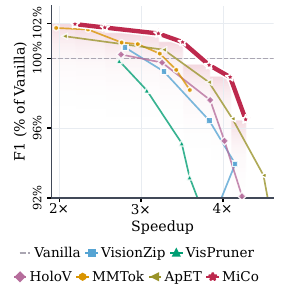}
        \caption{Quality--speed trade-off}
        \label{fig:teaser_efficiency}
    \end{subfigure}
    \caption{MiCo overview: visual token pruning, cross-model performance, and quality--speed trade-offs. Detailed caption continues on the next page.}
    \label{fig:teaser}
\end{figure}

% Continue the same figure caption on the next page, keeping all panels on page one.
\newpage
\begin{figure}[H]
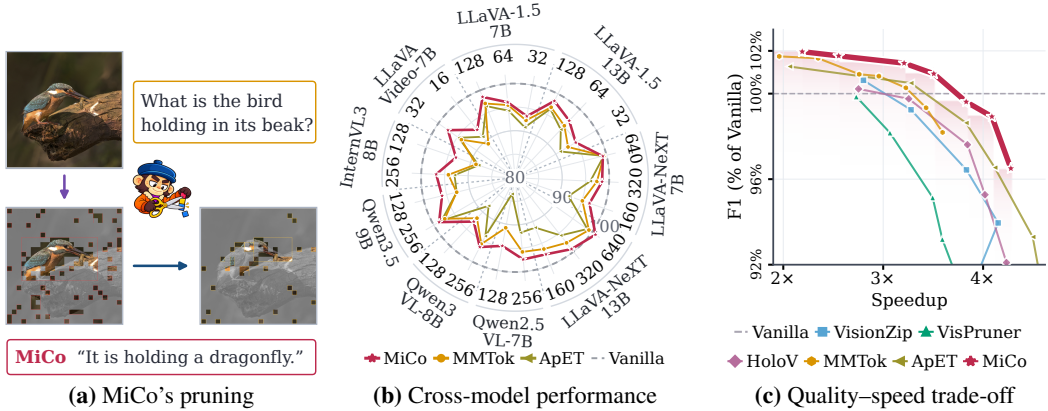

    \ContinuedFloat
    \captionsetup{skip=0pt}
    \caption[]{(continued) MiCo: (a) two-stage pruning on LLaVA-1.5-7B at $T=64$, retaining $576\!\rightarrow\!128\!\rightarrow\!46$ visual tokens for ``What is the bird holding in its beak?''; (b) relative performance across eight image models and one video model in 22 model--budget settings, normalized to each model's unpruned reference (video budgets are per frame); (c) POPE F1 retention relative to Vanilla versus one-token inference speedup on LLaVA-NeXT-13B.}
\end{figure}

\section{Introduction}

Large language models (LLMs) have demonstrated strong capabilities in language understanding and reasoning~\cite{achiam2023gpt,grattafiori2024llama,team2023gemini,yang2025qwen3}. By integrating vision encoders with LLMs, multimodal large language models (MLLMs) leverage these capabilities to understand visual content~\cite{liu2023visual,bai2023qwenvl,liu2024improved,liu2024llavanext,bai2025qwen3,zhu2025internvl3}. They convert visual inputs into token sequences that the LLM processes jointly with text. However, a single image can produce hundreds to thousands of visual tokens, and the number grows further with image resolution, the number of input images, and video frame count. Combined with the quadratic cost of self-attention, these long sequences dominate computation and memory, so reducing visual tokens while preserving model capabilities is a central challenge.

Many methods prune visual tokens to address this problem. Importance-ranking methods score tokens individually: FastV and SparseVLM use text-to-visual attention inside the language model, whereas PruMerge and VisionZip use CLS-token attention from the vision encoder~\cite{chen2024image,zhang2024sparsevlm,shang2025llava,yang2025visionzip}. Subset-construction methods instead build a representative token subset: DivPrune and DART use diversity criteria to reduce redundancy, while MMTok and CoverPruner use coverage criteria to preserve multimodal and visual information~\cite{alvar2025divprune,wen2025stop,dong2026mmtok,zhu2026coverpruner}. However, attention-based ranking often retains semantically redundant tokens, whereas the tokens kept by subset construction may be irrelevant to the current query. Even methods that combine the two paradigms~\cite{zou2025holov,deng2025scope} are motivated by heuristics rather than an analysis of the visual information the model needs, and degrade substantially at high pruning ratios.

To preserve the visual information that inference requires, we analyze how removing visual tokens changes the expected task log-loss and, through a semantic erasure model, derive a mutual information coverage objective. Its theoretical factors are not observable in a single forward pass, so we instantiate them with observable proxies and propose MiCo, a training-free two-stage method that greedily optimizes the resulting surrogate at each stage (Figure~\ref{fig:teaser}(a)); since the surrogate is monotone submodular, greedy selection attains at least a $(1-1/e)$ fraction of its optimum~\citep{nemhauser1978analysis}. Across MLLMs of diverse architectures and a broad range of image and video benchmarks, MiCo achieves the strongest performance among compared methods on nearly all models at all pruning ratios (Figure~\ref{fig:teaser}(b)) while substantially accelerating inference (Figure~\ref{fig:teaser}(c)).

Our main contributions are as follows:
\begin{itemize}[leftmargin=*,nosep]
    \item We analyze the visual information required for inference from the expected task log-loss and, via a semantic erasure model, derive a general mutual information coverage objective for visual token pruning.
    \item We instantiate the objective with accessible proxies, obtaining a tractable submodular surrogate, and propose MiCo, a training-free method that optimizes it by greedy selection.
    \item Extensive experiments across MLLMs, benchmarks, and pruning ratios show that MiCo consistently preserves performance while substantially accelerating inference.
\end{itemize}

\section{Related Work}

\subsection{Multimodal Large Language Models}

Multimodal large language models (MLLMs)~\cite{liu2023visual,bai2023qwenvl,chen2024internvl} encode images and videos into visual tokens that the LLM processes jointly with text. Visual tokens usually far outnumber text tokens: LLaVA-1.5~\cite{liu2024improved} uses 576 tokens for a 336$\times$336 image, LLaVA-NeXT~\cite{liu2024llavanext} up to 2,880 at 672$\times$672, dynamic-resolution models such as Qwen2.5-VL~\cite{Qwen2.5-VL}, Qwen3-VL~\cite{bai2025qwen3}, and InternVL3~\cite{zhu2025internvl3} scale token counts with image size, and LLaVA-Video~\cite{zhang2024llava} exceeds 10K tokens for 64 frames. These long sequences dominate prefill computation and memory, making visual token reduction essential for efficient MLLM inference.

\subsection{Visual Token Reduction for MLLMs}

Training-free pruning methods fall into two categories. The first ranks individual tokens by attention. FastV~\cite{chen2024image}, PyramidDrop~\cite{xing2024pyramiddrop}, and SparseVLM~\cite{zhang2024sparsevlm} use the attention that visual tokens receive from text inside the LLM, whereas PruMerge~\cite{shang2025llava}, VisionZip~\cite{yang2025visionzip}, and VisPruner~\cite{zhang2025vispruner} use CLS-to-patch attention from the visual encoder; VScan~\cite{zhang2025vscan} combines both attention sources. The second constructs a token subset. Diversity-based methods, including DivPrune~\cite{alvar2025divprune}, DART~\cite{wen2025stop}, and CDPruner~\cite{zhang2025cdpruner}, construct mutually distinct subsets to reduce redundancy, while coverage-based methods seek subsets that represent the original visual or multimodal token set. SCOPE, MMTok, CoverPruner, and ApET instantiate different forms of coverage or representative-set construction~\cite{deng2025scope,dong2026mmtok,zhu2026coverpruner,ma2026apet}. Attention-based selectors can retain redundant tokens from one semantic region, whereas set-level objectives without query conditioning may spend budget on task-irrelevant regions. Moving beyond such heuristic criteria, \textbf{MiCo} derives a weighted mutual information coverage objective and optimizes its observable proxy.

\section{Method}
\label{sec:method}

\begin{figure}[t]
    \centering
    \includegraphics[width=\textwidth]{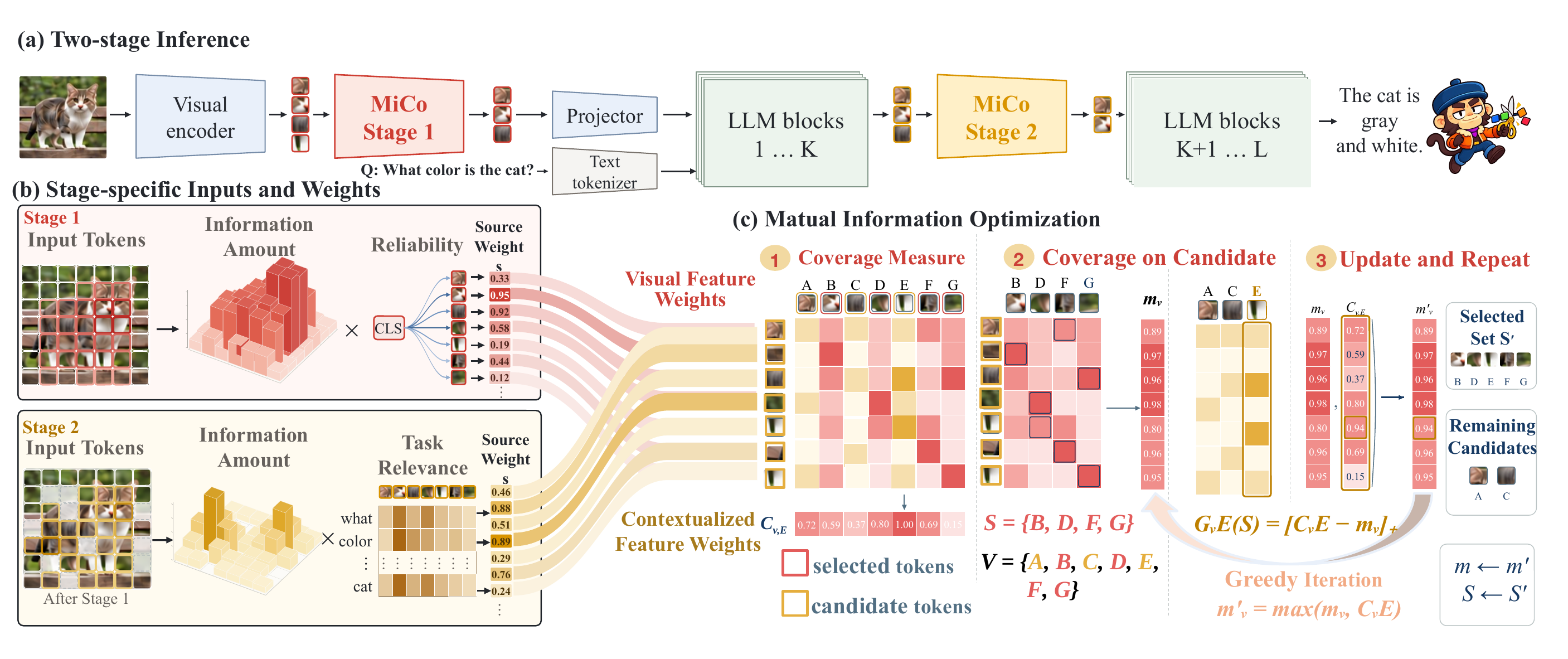}
    \caption{Overview of MiCo's two-stage pruning framework.}
    \label{fig:recover_overview}
\end{figure}

\subsection{From Task Loss to Recoverable Coverage}

\paragraph{Task loss and recoverable information.}
Let $\pi_{\mathrm{data}}$ be the true joint distribution of visual information $Z$, prompt $Q$, shared geometric context $G$ (the image layout and token positions, which pruning leaves unchanged), and discrete answer $Y$, and let $\pi_{\mathrm{pred}}$ be a conditional predictor. Fix $Q=q$ and $G=g$; all entropies and expectations are under $\pi_{\mathrm{data}}$ and are assumed finite. For expected log-loss $\mathcal L(\pi_{\mathrm{pred}};Z,q,g)=\mathbb E[-\log\pi_{\mathrm{pred}}(Y\mid Z,q,g)\mid q,g]$, the Bayes-optimal risk is $R_Z^*(q,g):=\inf_{\pi_{\mathrm{pred}}}\mathcal L(\pi_{\mathrm{pred}};Z,q,g)=H(Y\mid Z,q,g)$~\citep{gneiting2007strictly}.

Let $\mathbf U_V=(U_v)_{v\in V}$ be the complete visual source with discrete semantic variables $U_v$, and $\widehat{\mathbf U}_V(S)=(\widehat U_v(S))_{v\in V}$ the information recoverable from retained tokens $S$. Their Bayes risks are $R_{\mathrm{full}}^*$ and $R_S^*$, respectively. We model task relevance by a semantic-query index $J$ with $\Pr(J=v\mid q,g)=\Pr(J=v\mid Q=q)$, sampled independently of $(\mathbf U_V,\widehat{\mathbf U}_V(S))$ given $q,g$ for each fixed $S$ and recovery channel. Under $Y\perp\widehat{\mathbf U}_V(S)\mid\mathbf U_V,Q,G$, the log-loss/side-information identity~\citep{jiao2015justification}, entropy bound, and data-processing inequality~\citep[Theorems~3.4 and~3.7]{polyanskiy2025information} yield
\begin{equation}
\begin{aligned}
\Delta R^*(S;q,g)
&:=R_S^*(q,g)-R_{\mathrm{full}}^*(q,g)\\
&\le H(\mathbf U_V\mid q,g)
-I\!\left(U_J;\widehat U_J(S)\mid J,q,g\right).
\end{aligned}
\label{eq:task_information_risk_bound}
\end{equation}
Since $H(\mathbf U_V\mid q,g)$ is independent of $S$, maximizing the recovered information about the queried semantic unit minimizes this upper bound for a fixed query distribution.

\paragraph{Semantic erasure and source-weighted mutual-information coverage.}
To evaluate recovery at each source position, we introduce two binary gates. The coverage gate $E_v$ indicates whether $S$ supplies a usable representative, with probability $M_v(S;q,g):=\Pr(E_v=1\mid q,g)$. The reliability gate $C_v$ indicates whether the source representation faithfully carries $U_v$; for a fixed recovery channel, we treat $\Pr(C_v=1\mid q,g)$ as invariant to $S$. Conditional on $q,g$, we assume $C_v\perp E_v$ and $(C_v,E_v)\perp U_v$. The semantic erasure model is
\begin{equation}
\widehat U_v(S)=
\begin{cases}
U_v, & C_vE_v=1,\\
\bot, & C_vE_v=0.
\end{cases}
\label{eq:semantic_erasure}
\end{equation}
Here $\bot$ is an erasure symbol outside the alphabet of $U_v$. The independent gates succeed with probability $\Pr(C_v=1\mid q,g)M_v(S;q,g)$. The standard erasure-channel identity then gives $I(U_v;\widehat U_v(S)\mid q,g)=H(U_v\mid q,g)\Pr(C_v=1\mid q,g)M_v(S;q,g)$~\citep{polyanskiy2025information}. Expanding over $J$ and substituting this identity, we define the task-relevant recovered information to maximize:
\begin{equation}
\begin{aligned}
\mathcal I_{\mathrm{task}}(S;q,g)
&:=I\!\left(U_J;\widehat U_J(S)\mid J,q,g\right)\\
&=\sum_{v\in V}\Pr(J=v\mid Q=q)
I\!\left(U_v;\widehat U_v(S)\mid q,g\right)\\
&=\sum_{v\in V}
H(U_v\mid q,g)\Pr(C_v=1\mid q,g)\Pr(J=v\mid Q=q)
M_v(S;q,g).
\end{aligned}
\label{eq:weighted_coverage}
\end{equation}
In Eq.~\ref{eq:weighted_coverage}, the source contribution is weighted by
$H(U_v\mid q,g)$, which measures the information uncertainty associated with position $v$,
and by $\Pr(C_v=1\mid q,g)$ and $\Pr(J=v\mid Q=q)$, which measure,
respectively, the reliability of its representation and the likelihood that the
query requires it. The factor $M_v(S;q,g)$ measures how well the retained set
covers that position. The result is a source-weighted coverage
objective for visual-token pruning. The complete derivation is given in
Appendix~\ref{sec:appendix_logloss}--\ref{sec:appendix_erasure}.
\subsection{MiCo}
\label{sec:recover_variants}

The objective in Eq.~\ref{eq:weighted_coverage} contains three theoretical
factors that cannot be computed directly during a forward pass. We therefore
approximate them with observable proxies, denoted by
$\widehat{\mathcal H}_v$, $\widehat{\mathcal R}_v$, and
$\widehat{\mathcal P}_v$.

We use the hidden-state norm
$\widehat{\mathcal H}_v=\lVert x_v\rVert_2$ as a proxy for a token's information amount,
since its magnitude reflects the strength of the encoded visual signal. We approximate representation reliability $\widehat{\mathcal R}_v$ with the vision encoder's CLS-token attention, since it measures how strongly the representation at $v$ participates in global aggregation. At this point the visual tokens have not yet interacted with the text, so aggressive pruning could discard query-relevant tokens prematurely; Stage 1 therefore keeps a larger candidate pool, and a second selection inside the LLM uses text-to-visual attention as a query-dependent proxy for task relevance $\widehat{\mathcal P}_v$. The proxies are:
\begin{equation}
(\widehat{\mathcal H}_v,\widehat{\mathcal R}_v,\widehat{\mathcal P}_v)=
\begin{cases}
\bigl(\lVert x_v^{\mathrm{vis}}\rVert_2,
\overline{A}^{\mathrm{vis}}_{\mathrm{CLS}\to v},1\bigr),
& \text{Stage 1},\\[2pt]
\bigl(\lVert x_v^{(\ell_\star)}\rVert_2,
1,\overline{A}^{(\ell_\star)}_{\mathrm{text}\to v}\bigr),
& \text{Stage 2}.
\end{cases}
\label{eq:mico_stage_proxies}
\end{equation}
Here $x_v^{\mathrm{vis}}$ denotes the pre-projection visual
feature, and $x_v^{(\ell_\star)}$ denotes the visual feature at the
selected layer; $\overline{A}^{\mathrm{vis}}_{\mathrm{CLS}\to v}$ is the attention from the CLS token to token $v$ in the vision encoder, averaged over heads, and $\overline{A}^{(\ell_\star)}_{\mathrm{text}\to v}$ is the attention from the text tokens to token $v$ at decoder layer $\ell_\star$, averaged over text queries and heads. In both stages, we estimate coverage using the best
retained visual representative:
\begin{equation}
\widehat M_v(S)=\max_{s\in S}\left[\frac{x_v^\top x_s}{\lVert x_v\rVert_2\lVert x_s\rVert_2}\right]_+,
\label{eq:mico_coverage_proxy}
\end{equation}
where $[a]_+=\max(a,0)$ and $\widehat M_v(\varnothing)=0$. The overall proxy
objective sums the weighted coverage over all source positions:
\[
\widehat{\mathcal I}_{\mathrm{task}}(S)
=\sum_{v\in V}
\widehat{\mathcal H}_v\widehat{\mathcal R}_v\widehat{\mathcal P}_v
\widehat M_v(S).
\]
Starting from $S=\varnothing$, we use a greedy search to add the token $i^\star=\arg\max_{i\in V\setminus S}
\Delta\widehat{\mathcal I}_{\mathrm{task}}(i\mid S)$ until budget is reached, where
$\Delta\widehat{\mathcal I}_{\mathrm{task}}(i\mid S)
=\widehat{\mathcal I}_{\mathrm{task}}(S\cup\{i\})
-\widehat{\mathcal I}_{\mathrm{task}}(S)$. The objective
is normalized, monotone, and submodular, so greedy selection achieves a
$(1-1/e)$ approximation under the cardinality constraint
$|S|\le B$, where $B$ is the stage budget~\citep{krause2014submodular,nemhauser1978analysis}. Appendices~\ref{sec:appendix_coverage_guarantee} and~\ref{sec:appendix_selector} give the procedure and proof; Appendix~\ref{sec:appendix_proxy_checks} checks the proxies against their theoretical counterparts, and Table~\ref{tab:ablation_proxy_substitutions} evaluates substitutes.

\section{Experiments}

\subsection{Experimental Setup}
\label{sec:experimental_setup}

\paragraph{Models.}
We evaluate MiCo on MLLMs of different architectures, covering the LLaVA family (LLaVA-1.5-7B/13B~\cite{liu2024improved} and LLaVA-NeXT-7B/13B~\cite{liu2024llavanext}), the Qwen family (Qwen2.5-VL-7B-Instruct~\cite{Qwen2.5-VL}, Qwen3-VL-8B-Instruct~\cite{bai2025qwen3}, and Qwen3.5-9B~\cite{qwen35_9b_modelcard}), and InternVL3-8B~\cite{zhu2025internvl3}. For video understanding, we use LLaVA-Video-7B~\cite{zhang2024llava}.

\paragraph{Benchmarks.}
We use ten image tracks for the LLaVA family and eight for Qwen-VL/InternVL (e.g., GQA, POPE, MME, MMBench, SEED, and MMMU), supplementary tracks for grounding, document, chart, and OCR understanding (e.g., RefCOCO, DocVQA, OCRBench, and VTC-Bench), and four video benchmarks (VideoMME, MVBench, LongVideoBench, and MLVU); see Appendix~\ref{sec:appendix_benchmarks}.

\paragraph{Baselines.}
We compare MiCo with FastV~\cite{chen2024image}, SparseVLM~\cite{zhang2024sparsevlm}, VisionZip~\cite{yang2025visionzip}, DivPrune~\cite{alvar2025divprune}, VisPruner~\cite{zhang2025vispruner}, HoloV~\cite{zou2025holov}, VScan~\cite{zhang2025vscan}, MMTok~\cite{dong2026mmtok}, and ApET~\cite{ma2026apet}.

\paragraph{Implementation details.}
We evaluate all methods with the public evaluation stack of each model family: the official LLaVA evaluation scripts for LLaVA-1.5 and LLaVA-NeXT, VLMEvalKit~\cite{duan2024vlmevalkit} for Qwen2.5-VL, Qwen3-VL, Qwen3.5, and InternVL3, and lmms-eval~\cite{zhang2024lmmseval} for LLaVA-Video. Each model keeps its native processor, prompt format, and official scorer; we use batch size one with greedy decoding, and all pruning methods share the same token budgets and evaluation pipeline. For MiCo, Stage~1 keeps a candidate pool of $N_1=2T$ tokens on the visual-encoder features, and Stage~2 re-selects $R$ of them after decoder layer $K$, with $R$ chosen so that the average number of visual tokens over all decoder layers equals the budget $T$ (Appendix~\ref{sec:appendix_budget}); the per-model $K$ values are listed in Table~\ref{tab:layer_configurations}. Details are provided in Appendix~\ref{sec:appendix_implementation}.

\subsection{LLaVA Series}
\label{sec:llava_results}

We evaluate LLaVA-1.5 and LLaVA-NeXT at 7B and 13B scales with pruning ratios of 77.8\%, 88.9\%, and 94.4\%. MiCo achieves the highest aggregate score among reported pruners in 11 of 12 model--budget settings (Tables~\ref{tab:pruning_comparison_next13b_main} and~\ref{tab:other_models_summary}); the exception is LLaVA-NeXT-7B at 77.8\% pruning, where VScan is 0.3 points higher. At 94.4\% pruning, it retains 93.3\%/94.9\% of baseline performance on LLaVA-1.5-7B/13B and 96.7\%/97.5\% on LLaVA-NeXT-7B/13B (99.3\% on NeXT-13B at 88.9\%). Complete results are in Appendix~\ref{sec:appendix_full_results}.

\begin{table}[!htb]
    \centering
    % Source: https://zitd5je6f7j.feishu.cn/wiki/XlKVwhn5YimVnAkpOYhcOKa2n3b
% Snapshot: ccfa-workfiles/experiments/feishu-table-sync/cache/next13b_L4L3U2.json (read revision 1154; verified through workbook revision 1168).
% Main-text view of tables/llava-next-13b.tex; the appendix retains the complete table.
% Scores are unchanged; the main-text view uses its displayed method set.
\centering
\footnotesize
\setlength{\tabcolsep}{2.1pt}
\renewcommand{\arraystretch}{1.0}
\caption{Performance on LLaVA-NeXT-13B across ten benchmarks.}
\label{tab:pruning_comparison_next13b_main}
\micoTableFit{%
\begin{tabular}{l|cccccccccc|cc}
\toprule
Method & GQA & SQA-IMG & TextVQA & POPE & MME & MMB-EN & MMB-CN & SEED & AI2D & MMMU & Acc & Rel \\
\midrule
\multicolumn{13}{c}{Upper Bound --- 2880 tokens (no pruning)} \\
Vanilla & 64.4 & 73.1 & 63.2 & 85.3 & 1539.5 & 68.5 & 61.2 & 71.6 & 70.1 & 35.7 & 67.0 & 100.0\% \\
\midrule
\multicolumn{13}{c}{Retain 640 tokens (77.8\% pruned)} \\
\midrule
FastV (ECCV'24) & 60.9 & 71.7 & 60.7 & 80.2 & 1516.7 & 65.5 & 59.9 & 67.4 & 67.7 & 35.8 & 64.6 & 96.4\% \\
DivPrune (CVPR'25) & 63.5 & 72.2 & 59.2 & 86.5 & 1526.1 & 67.5 & 62.9 & 69.4 & 68.4 & \underline{37.8} & 66.4 & 99.1\% \\
VisPruner (ICCV'25) & 62.6 & 71.4 & \textbf{62.1} & 85.2 & \underline{1561.2} & 67.4 & \underline{63.1} & 68.8 & 68.2 & 37.0 & 66.4 & 99.1\% \\
VScan (TMLR'26) & 62.8 & 72.2 & 61.8 & 85.2 & 1553.6 & 67.6 & 62.7 & 68.2 & 69.9 & 36.6 & \underline{66.5} & \underline{99.2\%} \\
MMTok (ICLR'26) & 63.7 & 71.4 & 60.8 & \underline{86.8} & 1540.9 & 66.6 & 62.5 & \underline{69.7} & 68.1 & \textbf{38.3} & \underline{66.5} & \underline{99.2\%} \\
ApET (CVPR'26) & \underline{64.2} & \underline{72.5} & 58.0 & 86.4 & 1475.4 & \underline{67.7} & 62.5 & 69.5 & \textbf{70.9} & 36.7 & 66.2 & 98.8\% \\
\rowcolor{micoRose}
\textbf{MiCo} & \textbf{64.3} & \textbf{73.2} & \underline{61.9} & \textbf{87.0} & \textbf{1573.9} & \textbf{69.0} & \textbf{63.4} & \textbf{71.4} & \underline{70.2} & 37.2 & \textbf{67.6} & \textbf{100.9\%} \\
\midrule
\multicolumn{13}{c}{Retain 320 tokens (88.9\% pruned)} \\
\midrule
FastV (ECCV'24) & 54.6 & 70.5 & 55.4 & 63.6 & 1279.0 & 59.8 & 54.4 & 59.2 & 65.0 & 35.8 & 58.2 & 86.9\% \\
DivPrune (CVPR'25) & 61.8 & \underline{72.3} & 57.6 & 85.2 & 1473.0 & \underline{65.9} & 61.9 & 67.2 & 67.7 & \textbf{37.2} & 65.0 & 97.1\% \\
VisPruner (ICCV'25) & 60.8 & 70.1 & \underline{60.3} & 81.1 & 1486.0 & 65.7 & \underline{62.5} & 65.1 & 67.1 & 36.4 & 64.4 & 96.1\% \\
VScan (TMLR'26) & 61.0 & \textbf{72.4} & 59.3 & 82.0 & 1496.4 & 65.2 & 59.8 & 64.9 & 66.9 & 36.3 & 64.3 & 95.9\% \\
MMTok (ICLR'26) & \underline{62.78} & 71.49 & 59.05 & \underline{86.06} & \underline{1501.1} & 65.0 & 61.7 & \underline{67.7} & 67.9 & \textbf{37.2} & \underline{65.4} & \underline{97.6\%} \\
ApET (CVPR'26) & 61.7 & 71.4 & 54.7 & 84.1 & 1455.0 & 65.6 & 60.1 & 65.5 & \textbf{68.4} & 36.4 & 64.1 & 95.6\% \\
\rowcolor{micoRose}
\textbf{MiCo} & \textbf{63.3} & 71.3 & \textbf{60.9} & \textbf{86.5} & \textbf{1543.3} & \textbf{67.9} & \textbf{63.2} & \textbf{69.5} & \underline{68.3} & \underline{37.0} & \textbf{66.5} & \textbf{99.3\%} \\
\midrule
\multicolumn{13}{c}{Retain 160 tokens (94.4\% pruned)} \\
\midrule
DivPrune (CVPR'25) & 60 & 71.4 & 56.3 & 81.9 & 1436.7 & 65.1 & \underline{60.9} & 64.5 & \underline{67.3} & 36.6 & 63.6 & 94.9\% \\
VisPruner (ICCV'25) & 58.4 & 71.2 & \textbf{58.4} & 76.2 & 1390.8 & 64.1 & 59.6 & 61.1 & 65.5 & 36.2 & 62.0 & 92.6\% \\
VScan (TMLR'26) & 58.2 & 70.3 & 55.0 & 75.5 & 1385.1 & 61.5 & 53.4 & 61.5 & 63.4 & 34.1 & 60.2 & 89.9\% \\
MMTok (ICLR'26) & \textbf{62.02} & \textbf{72.19} & \underline{56.47} & \textbf{85.52} & \underline{1465.7} & \underline{65.3} & 60.3 & \underline{65.5} & \underline{67.3} & \textbf{37.0} & \underline{64.5} & \underline{96.3\%} \\
ApET (CVPR'26) & 59.1 & 71.4 & 52.9 & 79.6 & 1342.7 & 62.1 & 55.8 & 62.3 & 65.6 & \underline{36.8} & 61.3 & 91.4\% \\
\rowcolor{micoRose}
\textbf{MiCo} & \underline{62.0} & \underline{71.8} & \textbf{58.4} & \underline{85.0} & \textbf{1488.4} & \textbf{66.2} & \textbf{63.1} & \textbf{67.6} & \textbf{68.0} & 36.7 & \textbf{65.3} & \textbf{97.5\%} \\
\bottomrule
\end{tabular}
}

\end{table}

\subsection{Advanced Architectures}
\label{sec:advanced_results}

Table~\ref{tab:other_models_summary} compares Qwen2.5-VL-7B, Qwen3-VL-8B, Qwen3.5-9B, and InternVL3-8B at $T=256$ and $T=128$. MiCo achieves the highest aggregate score in all eight model--budget settings. At $T=128$, it retains 94.7\%, 92.8\%, 96.0\%, and 93.3\% of baseline performance, respectively. On Qwen2.5 and Qwen3, this exceeds MMTok by 3.9 and 1.0 percentage points; on InternVL3, the corresponding margin is 4.5 points. Qwen3.5 retains 98.7\%/96.0\% at $T=256/128$ versus ApET's 98.1\%/92.8\%, a larger advantage at the tighter budget. These gains across model families support MiCo's applicability beyond LLaVA (full results in Appendix~\ref{sec:appendix_full_results}).

\begin{table}[!htb]
    \centering
    % Rel values transposed from the detailed LLaVA, Qwen, and InternVL tables.
\centering
\footnotesize
\setlength{\tabcolsep}{1.2pt}
\renewcommand{\arraystretch}{1.05}
\caption{Relative performance (\%) across LLaVA, Qwen-VL, and InternVL models; Vanilla is 100\%.}
\label{tab:other_models_summary}
\begin{tabular}{>{\raggedright\arraybackslash}p{\dimexpr0.14\textwidth-2\tabcolsep\relax}*{17}{>{\centering\arraybackslash}p{\dimexpr0.86\textwidth/17-2\tabcolsep\relax}}}
\toprule
Method
  & \multicolumn{9}{c}{LLaVA}
  & \multicolumn{6}{c}{Qwen-VL}
  & \multicolumn{2}{c}{InternVL} \\
\cmidrule(lr){2-10}\cmidrule(lr){11-16}\cmidrule(lr){17-18}
Model
  & \multicolumn{3}{c}{1.5-7B}
  & \multicolumn{3}{c}{1.5-13B}
  & \multicolumn{3}{c}{NeXT-7B}
  & \multicolumn{2}{c}{2.5-VL-7B}
  & \multicolumn{2}{c}{3-VL-8B}
  & \multicolumn{2}{c}{3.5-9B}
  & \multicolumn{2}{c}{3-8B} \\
\cmidrule(lr){2-4}\cmidrule(lr){5-7}\cmidrule(lr){8-10}\cmidrule(lr){11-12}\cmidrule(lr){13-14}\cmidrule(lr){15-16}\cmidrule(lr){17-18}
Budget $T$ & 128 & 64 & 32 & 128 & 64 & 32 & 640 & 320 & 160 & 256 & 128 & 256 & 128 & 256 & 128 & 256 & 128 \\
\midrule
FastV & 92.6 & 75.5 & -- & 94.1 & 85.0 & -- & 95.4 & 80.4 & -- & 93.0 & 77.9 & 84.3 & 63.7 & 84.0 & 62.8 & \underline{95.5} & 81.5 \\
DivPrune & 96.8 & 94.4 & 91.3 & 96.7 & 93.2 & 91.2 & 99.0 & 97.2 & 94.2 & 94.5 & 90.3 & 93.6 & 88.6 & 95.9 & 91.8 & 93.3 & 87.7 \\
VisPruner & 96.6 & 93.3 & 87.8 & 96.6 & 93.6 & 88.5 & 98.7 & 95.6 & 91.0 & 91.8 & 88.0 & 94.2 & 86.8 & 97.0 & 90.1 & 89.4 & 81.1 \\
VScan & \underline{97.8} & \underline{96.0} & 91.8 & \underline{97.3} & \underline{96.0} & 88.7 & \textbf{99.5} & 96.2 & 90.3 & 94.8 & 90.6 & 95.4 & 90.1 & -- & -- & 92.9 & 87.3 \\
MMTok & 97.0 & 95.5 & \underline{92.4} & 96.8 & 95.2 & \underline{93.0} & \underline{99.4} & \underline{97.4} & \underline{95.1} & \underline{95.8} & \underline{90.8} & \underline{95.8} & \underline{91.8} & 97.3 & \underline{92.8} & 94.2 & \underline{88.8} \\
ApET & 96.4 & 94.3 & 90.9 & 96.4 & 94.1 & 91.1 & \underline{99.4} & 96.0 & 92.0 & 91.7 & 83.5 & 93.9 & 88.2 & \underline{98.1} & \underline{92.8} & 92.9 & 88.2 \\
\rowcolor{micoRose}
\textbf{MiCo} & \textbf{98.5} & \textbf{96.1} & \textbf{93.3} & \textbf{98.3} & \textbf{96.7} & \textbf{94.9} & 99.2 & \textbf{98.5} & \textbf{96.7} & \textbf{97.4} & \textbf{94.7} & \textbf{96.5} & \textbf{92.8} & \textbf{98.7} & \textbf{96.0} & \textbf{96.8} & \textbf{93.3} \\
\bottomrule
\end{tabular}

\end{table}

\subsection{LLaVA-Video}
\label{sec:video_results}

Table~\ref{tab:llava_video} compares LLaVA-Video-7B at per-frame budgets $T=32$ and $T=16$. MiCo achieves the highest Acc at both budgets and leads seven of eight task--budget comparisons; at $T=32$ it exceeds ApET, the strongest baseline, by 3.1 points on MLVU and 1.7 Acc points overall, and at $T=16$ it prunes 90.5\% of visual tokens while retaining 93.0\% of baseline performance with a 2.31$\times$ speedup.

\begin{table}[!htb]
    \centering
    % Performance from accepted four-benchmark rows; costs from the accepted native-input table.
% Only latency (complete answer), peak allocated memory and logical FLOPs are retained.
\centering
\footnotesize
\caption{Performance and efficiency on LLaVA-Video-7B.}
\label{tab:llava_video}
\label{tab:efficiency}
\setlength{\tabcolsep}{4pt}
\micoTableFit{%
\begin{tabular}{l*{9}{c}}
\toprule
\multirow{2}{*}{Method} & \multicolumn{6}{c}{Performance (\%)$\uparrow$} & \multicolumn{3}{c}{Efficiency$\downarrow$} \\
\cmidrule(lr){2-7}\cmidrule(lr){8-10}
& VideoMME & MVBench & \shortstack{LongVideo\\Bench} & MLVU & Acc & Rel & \shortstack[c]{Latency\\(ms)} & \shortstack[c]{Peak memory\\(GiB)} & \shortstack[c]{Logical\\FLOPs (T)} \\
\midrule
Vanilla & 63.63 & 58.18 & 59.01 & 67.75 & 62.14 & 100.0\% & 1855.0 & 25.59 & 243.833 \\
\midrule
\multicolumn{10}{c}{\textit{Budget $T=32$ tokens per frame}} \\
\midrule
FastV & 56.00 & 52.58 & 52.40 & 58.50 & 54.87 & 88.3\% & 975.1 & 41.30 & 75.951 \\
SparseVLM & 59.00 & 54.40 & 53.70 & 60.70 & 56.95 & 91.6\% & 1105.3 & 41.30 & 76.338 \\
DivPrune & 59.30 & 53.85 & 56.40 & 61.50 & 57.76 & 93.0\% & \textbf{881.8} & \underline{22.61} & 74.738 \\
VisPruner & \underline{60.15} & 55.20 & 56.69 & 62.46 & 58.63 & 94.3\% & \underline{884.2} & 47.96 & \textbf{74.732} \\
MMTok & 59.67 & 54.90 & 57.07 & 61.96 & 58.40 & 94.0\% & 1098.9 & \underline{22.61} & \underline{74.734} \\
ApET & 59.63 & \underline{55.33} & \underline{57.44} & \underline{62.89} & \underline{58.82} & \underline{94.7}\% & 924.1 & \underline{22.61} & \underline{74.734} \\
\rowcolor{micoRose}
\textbf{MiCo} & \textbf{61.81} & \textbf{56.78} & \textbf{57.67} & \textbf{65.99} & \textbf{60.56} & \textbf{97.5}\% & 933.8 & \textbf{21.32} & 74.834 \\
\midrule
\multicolumn{10}{c}{\textit{Budget $T=16$ tokens per frame}} \\
\midrule
FastV & 50.00 & 46.45 & 46.60 & 52.80 & 48.96 & 78.8\% & 885.0 & 41.30 & 60.937 \\
SparseVLM & 49.80 & 50.80 & 47.60 & 52.00 & 50.05 & 80.5\% & 1093.1 & 41.30 & 61.750 \\
DivPrune & 56.70 & 51.83 & 52.10 & 58.60 & 54.81 & 88.2\% & \textbf{783.7} & \underline{22.61} & 59.002 \\
VisPruner & 56.41 & 52.15 & 52.28 & 57.14 & 54.50 & 87.7\% & \underline{789.6} & 47.96 & 58.997 \\
MMTok & 55.96 & 52.18 & 54.38 & \underline{59.27} & 55.45 & 89.2\% & 902.0 & \underline{22.61} & 58.998 \\
ApET & \underline{57.04} & \underline{54.00} & \textbf{56.32} & 59.06 & \underline{56.61} & \underline{91.1}\% & 818.2 & \underline{22.61} & \underline{58.994} \\
\rowcolor{micoRose}
\textbf{MiCo} & \textbf{59.78} & \textbf{54.80} & \underline{54.75} & \textbf{61.93} & \textbf{57.82} & \textbf{93.0}\% & 804.7 & \textbf{21.32} & \textbf{58.913} \\
\bottomrule
\end{tabular}%
}

\end{table}

\begin{figure}[!htb]
    \centering
    \includegraphics[width=\linewidth]{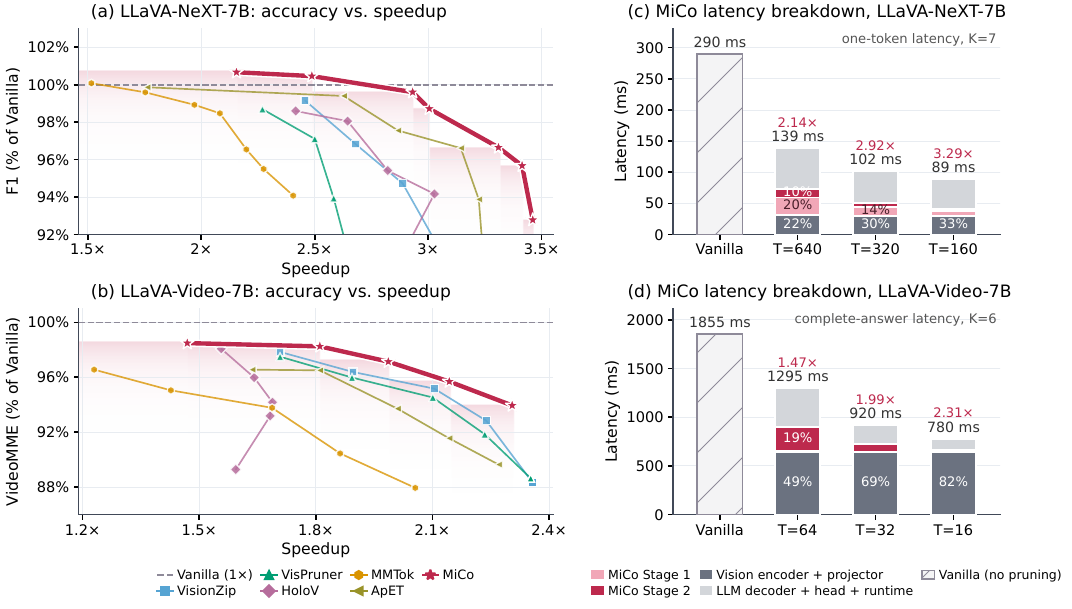}
    \caption{Performance--efficiency trade-offs (a,\,b) and MiCo's latency by component with speedups over Vanilla (c,\,d); NeXT one-token, Video complete-answer latency. Details in Appendix~\ref{sec:appendix_efficiency}.}
    \label{fig:pareto_next13b_video}
    \label{fig:video_pareto_main}
\end{figure}

\subsection{Efficiency Study}
\label{sec:efficiency}

At $T=16$ on LLaVA-Video-7B (Table~\ref{tab:llava_video}), MiCo retains 93.0\% of baseline performance with a 2.31$\times$ complete-answer speedup, 75.8\% fewer logical FLOPs, and the lowest peak memory of all compared methods (21.32 vs.\ 25.59\,GiB for Vanilla; FastV, SparseVLM, and VisPruner exceed the unpruned model). Figure~\ref{fig:teaser}(c) shows 99.62\% POPE F1 retention at a 3.83$\times$ one-token speedup on NeXT-13B; Figure~\ref{fig:pareto_next13b_video} shows 99.61\% at 2.93$\times$ on NeXT-7B and 98.25\% VideoMME retention at 1.81$\times$ on Video, together with where MiCo's latency goes. Protocols and the full NeXT-7B/13B comparisons (Table~\ref{tab:efficiency_next7b}) are in Appendix~\ref{sec:appendix_efficiency}.

\subsection{Challenging Visual Understanding}
\label{sec:challenging_results}

\paragraph{Grounding.}
\label{sec:grounding_results}

Grounding tests whether the retained tokens support spatial localization (parenthesized values are retention relative to Vanilla). On RefCOCO, RefCOCO+, and RefCOCOg, MiCo exceeds FastV, VisPruner, and ApET in all 12 combinations in Table~\ref{tab:ablation_grounding}; gRefCOCO results are in Appendix~\ref{sec:appendix_grefcoco}.

\begin{table}[!htbp]
\centering
% Source: Feishu Grounding comparison, sheet JL2Zv8, A1:E31, retrieved 2026-09-18.
\centering
\footnotesize
\setlength{\tabcolsep}{4pt}
\renewcommand{\arraystretch}{1.05}
\caption{REC scores (\%) on RefCOCO, RefCOCO+, and RefCOCOg.}
\label{tab:ablation_grounding}
\resizebox{\linewidth}{!}{%
\begin{tabular}{lcccccc}
\toprule
& \multicolumn{3}{c}{Qwen2.5-VL-7B} & \multicolumn{3}{c}{Qwen3-VL-8B} \\
\cmidrule(lr){2-4}\cmidrule(lr){5-7}
Method & RefCOCO & RefCOCO+ & RefCOCOg & RefCOCO & RefCOCO+ & RefCOCOg \\
\midrule
Vanilla & 86.8 (100.0\%) & 79.4 (100.0\%) & 85.5 (100.0\%) & 92.4 (100.0\%) & 88.7 (100.0\%) & 90.4 (100.0\%) \\
\midrule
\multicolumn{7}{c}{\textit{Retain 256 tokens (80.2\% pruned)}} \\
\midrule
FastV & 33.0 (38.0\%) & 28.3 (35.6\%) & 35.5 (41.5\%) & 67.2 (72.7\%) & 59.0 (66.5\%) & 68.1 (75.3\%) \\
VisPruner & 22.3 (25.7\%) & 17.8 (22.4\%) & 21.7 (25.4\%) & \underline{77.8} (84.2\%) & \underline{71.7} (80.8\%) & \underline{77.0} (85.2\%) \\
ApET & \underline{44.0} (50.7\%) & \underline{37.1} (46.7\%) & \underline{42.6} (49.8\%) & 51.4 (55.6\%) & 47.1 (53.1\%) & 49.7 (55.0\%) \\
\rowcolor{micoRose}
\textbf{MiCo} & \textbf{75.0} (86.4\%) & \textbf{68.1} (85.8\%) & \textbf{74.0} (86.5\%) & \textbf{82.8} (89.6\%) & \textbf{77.3} (87.1\%) & \textbf{81.9} (90.6\%) \\
\midrule
\multicolumn{7}{c}{\textit{Retain 128 tokens (90.1\% pruned)}} \\
\midrule
FastV & 7.1 (8.2\%) & 4.2 (5.3\%) & 4.9 (5.7\%) & 23.8 (25.8\%) & 19.6 (22.1\%) & 24.4 (27.0\%) \\
VisPruner & 11.3 (13.0\%) & 8.6 (10.8\%) & 10.6 (12.4\%) & \underline{49.1} (53.1\%) & \underline{43.0} (48.5\%) & \underline{47.8} (52.9\%) \\
ApET & \underline{17.9} (20.6\%) & \underline{14.6} (18.4\%) & \underline{15.4} (18.0\%) & 33.4 (36.1\%) & 29.5 (33.3\%) & 33.3 (36.8\%) \\
\rowcolor{micoRose}
\textbf{MiCo} & \textbf{50.0} (57.6\%) & \textbf{44.1} (55.5\%) & \textbf{47.7} (55.8\%) & \textbf{64.5} (69.8\%) & \textbf{57.8} (65.2\%) & \textbf{65.6} (72.6\%) \\
\bottomrule
\end{tabular}%
}

\end{table}

\paragraph{Fine-grained understanding.}
\label{sec:fine_grained_results}

Table~\ref{tab:qwen_reading_benchmarks} compares MiCo with FastV, VisPruner, and ApET on TextVQA, ChartQA, and DocVQA using Qwen2.5-VL-7B and Qwen3-VL-8B. MiCo leads all twelve displayed settings; at $T=128$ on Qwen2.5 its ChartQA score is 70.68 versus 36.68 for VisPruner, the strongest displayed baseline. The full comparison with eight external methods, including settings led by HoloV or MMTok, is in Appendix~\ref{sec:appendix_qwen_reading}.

\paragraph{VTC-Bench.}
\label{sec:vtc_results}

VTC-Bench~\citep{liao2026vtcbench} uses a fixed reference model to define downsampling-sensitive (Group A) and downsampling-robust (Group B) questions. We evaluate Qwen2.5-VL-7B on original images at $T=256/128$ (Table~\ref{tab:vtc_bench_qwen25_main}; column percentages are the reference image-area reductions used for grouping, not token budgets). MiCo outperforms FastV, VisPruner, and ApET in all 20 settings, retaining 75.8--84.9\% of baseline on Group A and 90.7--95.0\% on Group B at $T=128$. Details are in Appendix~\ref{sec:appendix_vtc_bench}.
% Sources read 2026-09-22:
% Feishu Docx WH5rdMOmEoTMKPxKoyGcuCTMnvd, revision 601,
% ChartQA/DocVQA native table doxcnmqvy6lAAphGldo3HBxTvnb.
% TextVQA: embedded spreadsheet ZWr6sFcIphhA49tbghjc7ESunkh,
% sheet pMs3TL, A1:F45 (native precision retained in the CSV below).
% Data snapshot: tables/data/qwen_reading_benchmarks.csv.
\begin{table}[!htbp]
\centering
\footnotesize
\setlength{\tabcolsep}{1.5pt}
\renewcommand{\arraystretch}{1.0}
\caption{TextVQA, ChartQA, and DocVQA scores (\%).}
\label{tab:qwen_reading_benchmarks}
\begin{tabular}{>{\raggedright\arraybackslash}p{\dimexpr0.13\linewidth-2\tabcolsep\relax}*{12}{>{\centering\arraybackslash}p{\dimexpr0.87\linewidth/12-2\tabcolsep\relax}}}
\toprule
& \multicolumn{6}{c}{Qwen2.5-VL-7B} & \multicolumn{6}{c}{Qwen3-VL-8B} \\
\cmidrule(lr){2-7}\cmidrule(lr){8-13}
& \multicolumn{3}{c}{$T=256$} & \multicolumn{3}{c}{$T=128$}
& \multicolumn{3}{c}{$T=256$} & \multicolumn{3}{c}{$T=128$} \\
\cmidrule(lr){2-4}\cmidrule(lr){5-7}\cmidrule(lr){8-10}\cmidrule(lr){11-13}
Method & Text & Chart & Doc
& Text & Chart & Doc
& Text & Chart & Doc
& Text & Chart & Doc \\
\midrule
Vanilla & 85.06 & 87.28 & 94.86 & 85.06 & 87.28 & 94.86 & 84.28 & 83.04 & 95.72 & 84.28 & 83.04 & 95.72 \\
\midrule
FastV & \underline{78.44} & \underline{68.28} & \underline{38.22} & 46.45 & 33.20 & 14.54 & 62.18 & 28.64 & 14.79 & 32.88 & 16.24 & 9.83 \\
VisPruner & 64.02 & 56.08 & 27.66 & \underline{49.40} & \underline{36.68} & \underline{19.69} & \underline{71.96} & \underline{63.52} & \underline{39.63} & \underline{54.58} & \underline{39.64} & 23.00 \\
ApET & 69.17 & 50.60 & 23.12 & 45.82 & 31.16 & 14.35 & 59.54 & 54.52 & 39.25 & 41.27 & 35.44 & \underline{25.59} \\
\midrule
\rowcolor{micoRose}
\textbf{MiCo} & \textbf{80.96} & \textbf{81.16} & \textbf{57.44} & \textbf{71.74} & \textbf{70.68} & \textbf{38.58} & \textbf{79.69} & \textbf{66.80} & \textbf{64.98} & \textbf{73.44} & \textbf{50.12} & \textbf{46.35} \\
\bottomrule
\end{tabular}
\end{table}

\begin{table}[!htb]
    % The VTC tabular is 253.5pt wide at its natural size (0.64\linewidth), so the
    % table keeps that width on the left and the figure takes the remaining column.
    \begin{minipage}[t]{0.645\linewidth}
    \vspace{0pt}
    % Compact main-text view of the accepted Group A/B CSV snapshots.
% Values are rounded from source Rel, never recomputed from rounded Acc.
\centering
\scriptsize
\setlength{\tabcolsep}{3pt}
\renewcommand{\arraystretch}{1.0}
\caption{VTC-Bench Rel (\%) on Qwen2.5-VL-7B. Entries show Group A (B): downsampling-sensitive (robust) questions under the reference model; $\Delta=|\mathrm{Rel}_{B}-\mathrm{Rel}_{A}|$ is the gap in percentage points.}
\label{tab:vtc_bench_qwen25_main}
\begin{tabular}{l*{5}{c}}
\toprule
\multirow{2}{*}{Method} & \multicolumn{5}{c}{Reference image-area reduction} \\
\cmidrule(lr){2-6}
& 75.00\% & 88.89\% & 93.75\% & 96.00\% & 99.00\% \\
\midrule
Vanilla & 100.0 {\scriptsize (100.0)} & 100.0 {\scriptsize (100.0)} & 100.0 {\scriptsize (100.0)} & 100.0 {\scriptsize (100.0)} & 100.0 {\scriptsize (100.0)} \\
\midrule
\rowcolor{black!4}
\multicolumn{6}{c}{\textbf{Token budget} $T=256$} \\
FastV & \underline{76.1} {\scriptsize (\underline{90.8})} & \underline{80.0} {\scriptsize (91.7)} & \underline{81.0} {\scriptsize (92.4)} & \underline{82.3} {\scriptsize (92.5)} & \underline{85.1} {\scriptsize (93.9)} \\
VisPruner & 74.4 {\scriptsize (89.7)} & 78.3 {\scriptsize (\underline{91.8})} & 79.3 {\scriptsize (\underline{93.2})} & 80.1 {\scriptsize (\underline{93.6})} & 83.4 {\scriptsize (\underline{94.8})} \\
ApET & 71.8 {\scriptsize (87.6)} & 76.0 {\scriptsize (90.0)} & 76.8 {\scriptsize (91.8)} & 77.4 {\scriptsize (92.5)} & 80.7 {\scriptsize (93.4)} \\
\rowcolor{micoRose}
\textbf{MiCo} & \textbf{90.1} {\scriptsize (\textbf{96.1})} & \textbf{91.6} {\scriptsize (\textbf{96.7})} & \textbf{91.9} {\scriptsize (\textbf{97.3})} & \textbf{92.4} {\scriptsize (\textbf{97.3})} & \textbf{93.3} {\scriptsize (\textbf{98.3})} \\
\midrule
\rowcolor{black!4}
\multicolumn{6}{c}{\textbf{Token budget} $T=128$} \\
FastV & 48.9 {\scriptsize (66.8)} & 50.1 {\scriptsize (70.1)} & 48.9 {\scriptsize (73.6)} & 48.9 {\scriptsize (75.0)} & 53.6 {\scriptsize (76.3)} \\
VisPruner & \underline{62.7} {\scriptsize (\underline{82.0})} & \underline{64.9} {\scriptsize (\underline{85.9})} & \underline{66.1} {\scriptsize (\underline{88.7})} & \underline{67.7} {\scriptsize (\underline{89.6})} & \underline{72.3} {\scriptsize (\underline{90.9})} \\
ApET & 56.5 {\scriptsize (77.8)} & 56.8 {\scriptsize (82.1)} & 58.4 {\scriptsize (85.6)} & 60.4 {\scriptsize (87.0)} & 66.2 {\scriptsize (88.2)} \\
\rowcolor{micoRose}
\textbf{MiCo} & \textbf{75.8} {\scriptsize (\textbf{90.7})} & \textbf{79.1} {\scriptsize (\textbf{92.6})} & \textbf{81.5} {\scriptsize (\textbf{93.4})} & \textbf{82.3} {\scriptsize (\textbf{93.7})} & \textbf{84.9} {\scriptsize (\textbf{95.0})} \\
\bottomrule
\end{tabular}

    \end{minipage}\hfill
    \begin{minipage}[t]{0.335\linewidth}
    \vspace{0pt}
    \centering
    \captionsetup{skip=4pt}
    % Rendered by visual-composer/s1-budget-sensitivity/plot_main_pair.py
    % (--metric "Average Acc" --vscan "qwen25=75.2,internvl3=74.3" --stack --compact --width 1.85 --height 1.75 --font 6.4).
    % Figure height and caption length are tuned so this column ends level with the table's bottom rule.
    \includegraphics[width=\linewidth]{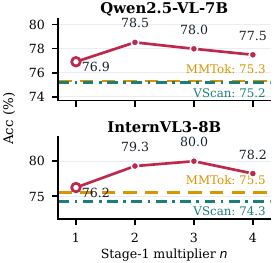}
    \captionof{figure}{Stage-1 budget sensitivity at $T=128$: eight-benchmark Acc for $n=1$--$4$ (hollow: MiCo-S1 reference) vs.\ MMTok and VScan.}
    \label{fig:s1_budget_main}
    \end{minipage}
\end{table}

\subsection{Ablation Study}
\label{sec:ablation_studies}

We ablate the components of MiCo's two-stage design and the generality of the mutual information coverage objective.

\paragraph{Stage-1 budget and objective's effectiveness.}
We change the Stage-1 pool size $nT$ at a fixed layer-average budget of $T=128$ and report the eight-benchmark Acc on Qwen2.5-VL-7B and InternVL3-8B in Figure~\ref{fig:s1_budget_main}. MiCo exceeds MMTok and VScan at every $n$ on both models, proving our gains do not depend on the pool size used in Stage~1. Specifically, at $n=1$ the pool equals the budget and MiCo reduces to its single-stage implementation (MiCo-S1); every $n>1$ improves on it, proving that the LLM-side refinement works, while $n=1$ still beats MMTok and VScan, proving that our objective is promising even in a single stage. Appendix~\ref{sec:appendix_s1_budget} gives details, and Appendix~\ref{sec:appendix_full_results} records MiCo-S1's results on more models.

\paragraph{Two-stage components.}
Figure~\ref{fig:two_stage_components} ablates the components of MiCo's two stages. The complete configuration achieves the highest seven-benchmark mean, Acc$_7=78.4$, and Acc$_7$ generally decreases as more components are disabled, showing that all signals contribute.

\begin{figure}[!htbp]
    \centering
    \captionsetup{skip=4pt}
    \includegraphics[width=0.86\linewidth]{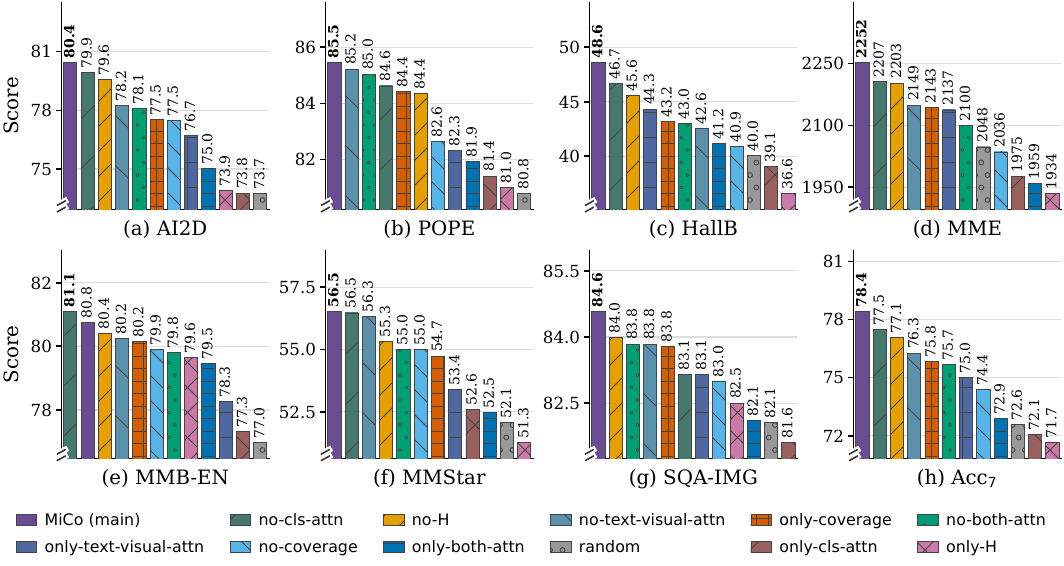}
    \caption{Two-stage component comparison on Qwen2.5-VL-7B at $T=128$.}
    \label{fig:two_stage_components}
\end{figure}

\paragraph{Proxy substitutions.}
We evaluate replacements for the information, coverage, reliability, and task-relevance proxies on Qwen2.5-VL-7B at $T=128$ in Table~\ref{tab:ablation_proxy_substitutions}. The information proxy is replaced with normalized
hidden-state entropy, the coverage and reliability proxies with similarity
measures over visual tokens, and the task-relevance proxy with similarity
measures over text tokens. MiCo achieves the highest score, and all variants except the coverage replacement still outperform
MMTok, supporting the applicability and robustness of our mutual information coverage
objective beyond MiCo's proxy choices. Appendix~\ref{sec:appendix_proxy_substitutions} details the replacements, and Appendix~\ref{sec:appendix_task_relevance} compares further task-relevance variants.

\begin{table}[!htbp]
\centering
% Sources and precision: tables/data/proxy_substitutions_20260925.json.
% MMTok is shown as the external reference; substitution rows are ordered by Acc.
\centering
\footnotesize
\caption{Proxy substitutions on Qwen2.5-VL-7B at $T=128$. The second line in
each cell reports the score difference $\Delta$ from MMTok.}
\label{tab:ablation_proxy_substitutions}
\micoTableFit{%
\begin{tabular}{l*{10}{c}}
\toprule
Replaced proxy & AI2D & POPE & HallB & MME & $\mathrm{MMB}_{\mathrm{EN}}$ & $\mathrm{MMB}_{\mathrm{CN}}$ & MMStar & SQA & Acc & Rel \\
\midrule
MMTok (ICLR'26) & 76.3 & 84.2 & 43.3 & 2120.7 & 79.8 & 77.5 & 53.7 & 81.4 & 75.3 & 90.8\% \\
\midrule
Replace Coverage & \shortstack{70.14\\[-1pt]\scriptsize{\textcolor{black!60}{$\Delta\,-6.16$}}} & \shortstack{78.10\\[-1pt]\scriptsize{\textcolor{black!60}{$\Delta\,-6.10$}}} & \shortstack{35.80\\[-1pt]\scriptsize{\textcolor{black!60}{$\Delta\,-7.50$}}} & \shortstack{1688.69\\[-1pt]\scriptsize{\textcolor{black!60}{$\Delta\,-432.01$}}} & \shortstack{75.69\\[-1pt]\scriptsize{\textcolor{black!60}{$\Delta\,-4.11$}}} & \shortstack{73.71\\[-1pt]\scriptsize{\textcolor{black!60}{$\Delta\,-3.79$}}} & \shortstack{47.20\\[-1pt]\scriptsize{\textcolor{black!60}{$\Delta\,-6.50$}}} & \shortstack{81.31\\[-1pt]\scriptsize{\textcolor{black!60}{$\Delta\,-0.09$}}} & \shortstack{68.30\\[-1pt]\scriptsize{\textcolor{black!60}{$\Delta\,-7.00$}}} & \shortstack{82.34\%\\[-1pt]\scriptsize{\textcolor{black!60}{$\Delta\,-8.46$}}} \\
Replace $H$ & \shortstack{78.47\\[-1pt]\scriptsize{\textcolor{black!60}{$\Delta\,+2.17$}}} & \shortstack{84.49\\[-1pt]\scriptsize{\textcolor{black!60}{$\Delta\,+0.29$}}} & \shortstack{45.43\\[-1pt]\scriptsize{\textcolor{black!60}{$\Delta\,+2.13$}}} & \shortstack{2152.60\\[-1pt]\scriptsize{\textcolor{black!60}{$\Delta\,+31.90$}}} & \shortstack{\underline{80.84}\\[-1pt]\scriptsize{\textcolor{black!60}{$\Delta\,+1.04$}}} & \shortstack{78.26\\[-1pt]\scriptsize{\textcolor{black!60}{$\Delta\,+0.76$}}} & \shortstack{52.60\\[-1pt]\scriptsize{\textcolor{black!60}{$\Delta\,-1.10$}}} & \shortstack{83.04\\[-1pt]\scriptsize{\textcolor{black!60}{$\Delta\,+1.64$}}} & \shortstack{76.35\\[-1pt]\scriptsize{\textcolor{black!60}{$\Delta\,+1.05$}}} & \shortstack{92.04\%\\[-1pt]\scriptsize{\textcolor{black!60}{$\Delta\,+1.24$}}} \\
Replace Relevance & \shortstack{79.08\\[-1pt]\scriptsize{\textcolor{black!60}{$\Delta\,+2.78$}}} & \shortstack{\textbf{85.54}\\[-1pt]\scriptsize{\textcolor{black!60}{$\Delta\,+1.34$}}} & \shortstack{45.88\\[-1pt]\scriptsize{\textcolor{black!60}{$\Delta\,+2.58$}}} & \shortstack{2142.01\\[-1pt]\scriptsize{\textcolor{black!60}{$\Delta\,+21.31$}}} & \shortstack{79.73\\[-1pt]\scriptsize{\textcolor{black!60}{$\Delta\,-0.07$}}} & \shortstack{\underline{78.95}\\[-1pt]\scriptsize{\textcolor{black!60}{$\Delta\,+1.45$}}} & \shortstack{56.47\\[-1pt]\scriptsize{\textcolor{black!60}{$\Delta\,+2.77$}}} & \shortstack{\underline{84.28}\\[-1pt]\scriptsize{\textcolor{black!60}{$\Delta\,+2.88$}}} & \shortstack{77.13\\[-1pt]\scriptsize{\textcolor{black!60}{$\Delta\,+1.83$}}} & \shortstack{92.98\%\\[-1pt]\scriptsize{\textcolor{black!60}{$\Delta\,+2.18$}}} \\
Replace Reliability & \shortstack{\underline{79.83}\\[-1pt]\scriptsize{\textcolor{black!60}{$\Delta\,+3.53$}}} & \shortstack{84.59\\[-1pt]\scriptsize{\textcolor{black!60}{$\Delta\,+0.39$}}} & \shortstack{\underline{46.13}\\[-1pt]\scriptsize{\textcolor{black!60}{$\Delta\,+2.83$}}} & \shortstack{\underline{2207.75}\\[-1pt]\scriptsize{\textcolor{black!60}{$\Delta\,+87.05$}}} & \shortstack{\textbf{81.10}\\[-1pt]\scriptsize{\textcolor{black!60}{$\Delta\,+1.30$}}} & \shortstack{\underline{78.95}\\[-1pt]\scriptsize{\textcolor{black!60}{$\Delta\,+1.45$}}} & \shortstack{\textbf{56.73}\\[-1pt]\scriptsize{\textcolor{black!60}{$\Delta\,+3.03$}}} & \shortstack{83.34\\[-1pt]\scriptsize{\textcolor{black!60}{$\Delta\,+1.94$}}} & \shortstack{\underline{77.63}\\[-1pt]\scriptsize{\textcolor{black!60}{$\Delta\,+2.33$}}} & \shortstack{\underline{93.59\%}\\[-1pt]\scriptsize{\textcolor{black!60}{$\Delta\,+2.79$}}} \\
\rowcolor{micoRose}
\textbf{MiCo} & \shortstack{\textbf{80.44}\\[-1pt]\scriptsize{\textcolor{black!60}{$\Delta\,+4.14$}}} & \shortstack{\underline{85.46}\\[-1pt]\scriptsize{\textcolor{black!60}{$\Delta\,+1.26$}}} & \shortstack{\textbf{48.60}\\[-1pt]\scriptsize{\textcolor{black!60}{$\Delta\,+5.30$}}} & \shortstack{\textbf{2252.37}\\[-1pt]\scriptsize{\textcolor{black!60}{$\Delta\,+131.67$}}} & \shortstack{80.76\\[-1pt]\scriptsize{\textcolor{black!60}{$\Delta\,+0.96$}}} & \shortstack{\textbf{79.12}\\[-1pt]\scriptsize{\textcolor{black!60}{$\Delta\,+1.62$}}} & \shortstack{\underline{56.53}\\[-1pt]\scriptsize{\textcolor{black!60}{$\Delta\,+2.83$}}} & \shortstack{\textbf{84.58}\\[-1pt]\scriptsize{\textcolor{black!60}{$\Delta\,+3.18$}}} & \shortstack{\textbf{78.51}\\[-1pt]\scriptsize{\textcolor{black!60}{$\Delta\,+3.21$}}} & \shortstack{\textbf{94.65\%}\\[-1pt]\scriptsize{\textcolor{black!60}{$\Delta\,+3.85$}}} \\
\bottomrule
\end{tabular}%
}

\end{table}
\FloatBarrier

\section{Conclusion}

We derive a general visual token pruning objective from task log-loss through a semantic erasure model and instantiate its factors with observable proxies, yielding a monotone submodular coverage objective that MiCo optimizes greedily in two stages. Across diverse MLLMs and image/video benchmarks, MiCo retains strong performance while reducing inference cost. We hope the objective encourages better proxies and further improves the performance--speed trade-off.

\vspace{-6pt}
\subsection*{AI use statement}
\vspace{-3pt}

We used large language models to aid and polish the writing, for retrieval and discovery of related work, and for research ideation and execution. The authors reviewed and verified all AI-assisted content and take full responsibility for the final content of this work.

\vspace{-6pt}
\subsection*{Reproducibility statement}
\vspace{-3pt}

MiCo is training-free and all experiments use publicly available pretrained models and benchmarks. Implementation details, per-model configurations, evaluation protocols, and proofs are given in the appendix. We will open-source the code and configuration files so that all reported results can be reproduced.

\vspace{-6pt}
\subsection*{Ethics statement}
\vspace{-3pt}

This work uses only publicly available pretrained models and standard public benchmarks, involves no human subjects or newly collected data, and introduces no capabilities beyond those of the underlying models. We do not foresee ethical concerns beyond those already associated with the base models.

\bibliography{main,custom}
\bibliographystyle{iclr2027_conference}

\clearpage
\appendix
% Appendix order: experimental setup, theory, implementation, model-specific implementation,
% efficiency, ablations, layer selection, complete results, supplementary benchmarks,
% limitations.
\section{Benchmark Definitions and Evaluation Protocol}
\label{sec:appendix_benchmarks}

This section defines the benchmark inventory, task coverage, and scoring conventions used in the
experiments. The two main image suites overlap: the LLaVA tables use ten benchmark tracks, while
the Qwen-VL and InternVL tables use eight. Their union is twelve core image tracks when the
English and Chinese MMBench evaluations are counted separately. We additionally report eight
supplementary image tracks (RefCOCO, RefCOCO+, RefCOCOg, gRefCOCO, ChartQA, DocVQA,
OCRBench, and VTC-Bench) and four video benchmarks (VideoMME, MVBench, LongVideoBench, and MLVU).
In total, this appendix covers twenty image tracks (twelve core and eight supplementary, counting
VTC-Bench as one track) and four video benchmarks; supplementary tracks never enter a main-table Acc.

\subsection{Core Image Benchmarks}
\label{sec:appendix_core_image_benchmarks}

The LLaVA tables use GQA, SQA-IMG, TextVQA, POPE, MME, MMBench-EN, MMBench-CN, SEED, AI2D, and
MMMU; the Qwen-VL and InternVL tables use AI2D, POPE, HallusionBench, MME, MMBench-EN, MMBench-CN,
MMStar, and SQA-IMG. A benchmark enters a model's Acc only if it appears in that model's main table.

\paragraph{General visual understanding and reasoning.}
GQA~\cite{hudson2019gqa} evaluates compositional visual reasoning over objects, attributes, and
relations grounded in scene graphs. We use the balanced evaluation split and report exact-match
accuracy; the standard split contains 12,578 questions. ScienceQA-IMG~\cite{lu2022learn} is the
image-grounded subset of the multimodal science multiple-choice benchmark. It requires selecting
an answer after combining the question, explanatory context, and a diagram or image; we report
multiple-choice accuracy on the released image subset. AI2D~\cite{kembhavi2016diagram} evaluates
question answering over scientific diagrams and their associated text. Its released collection
contains more than 5,000 diagrams and 15,000 questions; we report answer accuracy on the test split (3,088 questions). MMMU~\cite{yue2024mmmu}
contains 11.5K college-level multimodal questions spanning six broad disciplines and 30 subjects;
we use the recorded single-image evaluation protocol on the validation split and report multiple-choice accuracy. MMStar~\cite{chen2024mmstar}
is a 1,500-question diagnostic suite covering coarse and fine-grained perception, reasoning, and
knowledge-intensive visual tasks; we report its official accuracy.

\paragraph{Text and broad multimodal understanding.}
TextVQA~\cite{singh2019towards} tests reading and reasoning over scene text; the validation split
contains 5,000 questions and is scored with the VQA consensus metric. MMBench~\cite{liu2024mmbench}
evaluates 20 ability dimensions, including recognition, localization, OCR, and reasoning. MMBench-EN
and MMBench-CN are treated as two separate tracks, evaluated with the released language-specific
scorers and Circular accuracy. SEED-Bench~\cite{li2024seedbench} provides objective multiple-choice
evaluation across multimodal capability dimensions; we use its image-only evaluation subset and
report the official accuracy.

\paragraph{Hallucination diagnostics.}
POPE~\cite{li2023evaluating} probes object hallucination with binary object-presence questions over
its random, popular, and adversarial categories (3,000 released questions each). The reported
metric follows each family's official evaluator: the LLaVA tables report the average F1 score over
the three categories, as computed by the LLaVA-1.5 evaluation scripts, whereas the Qwen-VL and
InternVL tables report the overall accuracy returned by VLMEvalKit, whose runner expands the
category outputs to 9,000 scored entries. HallusionBench~\cite{guan2024hallusionbench} is a
diagnostic suite for language hallucination and visual illusion; we report VLMEvalKit's overall
average, i.e., the mean of its all-question accuracy (aAcc), figure-level accuracy (fAcc), and
question-pair accuracy (qAcc).

\paragraph{Perception and cognition.}
MME~\cite{fu2023mme} covers perception and cognition in 14 subtasks with 2,374 questions over
1,187 images (two questions per image). The MME column keeps the benchmark's native summed score,
whose scale differs by family: the LLaVA tables report the Perception score (maximum 2,000),
following the LLaVA evaluation scripts, whereas the Qwen-VL and InternVL tables report the
Perception+Cognition total (maximum 2,800) returned by VLMEvalKit. In both cases the summed score is
divided by 20 only when forming Acc, so MME enters the LLaVA aggregates on a 0--100 scale and the
Qwen-VL/InternVL aggregates on a 0--140 scale.

\subsection{Supplementary Image Benchmarks}
\label{sec:appendix_supplementary_image_benchmarks}

\paragraph{Text, chart, and document understanding.}
ChartQA~\cite{masry2022chartqa} requires visual and logical reasoning over charts. We use the test
split and its relaxed accuracy metric (2,500 questions in the evaluation record). DocVQA~\cite{mathew2021docvqa}
evaluates question answering over document images; we use the validation split and report ANLS over
5,349 questions. OCRBench~\cite{liu2023ocrbench} aggregates 29 OCR-oriented datasets covering text
recognition, scene-text VQA, document VQA, key-information extraction, and handwritten mathematical
expression recognition. We report its released aggregate score. The 100 OCRBench samples used to select the
refinement layer (Appendix~\ref{sec:appendix_image_calibration}) never enter a main-table Acc.

\paragraph{Referring expression comprehension.}
RefCOCO, RefCOCO+, and RefCOCOg evaluate localization of the region described by a natural-language
expression. RefCOCO and RefCOCO+ distinguish person-centric testA and object-centric testB, while
RefCOCOg uses longer, more descriptive expressions; we report the standard REC score on the released
evaluation split for each dataset. The parenthesized values in Table~\ref{tab:ablation_grounding}
normalize each score by the matching Vanilla result. gRefCOCO extends this setting to expressions
that refer to one object, multiple objects, or no target object~\cite{he2023grec}. We use the
released val, testA, and testB splits and report $\mathrm{Pr}@(F_1=1,\,\mathrm{GIoU}\ge0.5)$,
computed with the released GREC evaluation code; our protocol applies a generalized-IoU box-matching
threshold in place of the $\mathrm{IoU}\ge0.5$ threshold of the original metric definition.

\paragraph{Visual information preservation.}
VTC-Bench~\cite{liao2026vtcbench} measures whether a token-compressed model preserves visual
information that is sensitive to image downsampling. Its fixed Qwen2-VL reference partitions each
question into Group~A (correct at full resolution but incorrect after downsampling) or Group~B
(correct in both conditions) for reference factors $d\in\{2,3,4,5,10\}$. We reuse these groups
without modification, run Qwen2.5-VL-7B on the original images at $T=256$ and $T=128$, and
report the official task scores, the group-wise mean, and the Group~B--Group~A gap. The reference
image-area reductions (75.00\%, 88.89\%, 93.75\%, 96.00\%, and 99.00\%) define the groups; they are
not the pruning ratios of our target model.

\subsection{Video Benchmarks}
\label{sec:appendix_video_benchmarks}

All video results use the same LLaVA-Video-7B input and decoding protocol described in
Appendix~\ref{sec:appendix_budget} (video input protocol). VideoMME~\cite{fu2025video} evaluates short, medium, and long videos
across multiple visual domains; we use the no-subtitle setting and report the official aggregate
percentage. MVBench~\cite{li2024mvbench} contains 20 multiple-choice tasks spanning spatial and
temporal perception, action, object interaction, and episodic reasoning; we report the mean over its
official task scores. LongVideoBench~\cite{wu2024longvideobench} tests long-context video-language
understanding with interleaved video and text, including videos up to one hour; its official
aggregate is returned as a fraction and is converted to a percentage for the tables. MLVU~\cite{zhou2025mlvu} is a multi-task long-video suite covering diverse video genres and durations;
we report its official macro-average over the seven task categories rather than pooling all
questions. The four video scores are averaged with equal weight to form the video Acc.

\subsection{Aggregation and Decoding Conventions}
\label{sec:appendix_benchmark_protocol}

Unless a benchmark specifies a task-native metric above, scores are percentages and higher is
better. For an image model, Acc is the unweighted mean of the benchmark tracks listed in its main
table. MME contributes $\mathrm{MME}/20$ to this mean, while its native summed score remains
visible in the MME column. Rel is the aggregate Acc divided by the corresponding Vanilla Acc and
reported as a percentage. English and Chinese MMBench are separate terms in the mean; supplementary
benchmarks and calibration examples never enter the main-table Acc unless explicitly stated in a
table caption.

We use batch size one and deterministic greedy decoding without sampling. Each model family keeps
its native processor, prompt format, image geometry, and official scorer. Empty answers and
generation-cap truncations remain in the relevant denominator. For video, we follow each benchmark's
standard scoring protocol and then average the four resulting scores to obtain video Acc.

\FloatBarrier

\section{Information-Theoretic Derivations}
\label{sec:appendix_theory}

This appendix supplies the standard identities used in Section~\ref{sec:method} and their application to our recovery model, followed by the optimization properties of MiCo's coverage proxy. The log-loss and erasure-channel identities are established results. MiCo applies them to the random semantic-query model defined in the main text.

\paragraph{Assumptions and scope.}
Throughout this appendix, $Y$ and the semantic variables $U_v$ are discrete, all relevant conditional entropies and losses are finite, and $q,g$ are fixed. For each fixed retained set and recovery channel, the query index $J$ is sampled independently of the source and recovered variables. The answer-side model assumes $Y\perp\widehat{\mathbf U}_V(S)\mid\mathbf U_V,q,g$. In the erasure model, the coverage and reliability gates are conditionally independent of each other and of $U_v$; reliability is held fixed with respect to $S$. The greedy result additionally fixes the proxy weights and similarities within one selector call.

\subsection{Log-Loss, Bayes Risk, and Recoverable Information}
\label{sec:appendix_logloss}

Let $\pi_{\mathrm{data}}$ be the underlying data-generating joint distribution, including the answer $Y$, visual information $Z$, prompt $Q$, and geometric context $G$. Fix $Q=q$ and $G=g$, and use $\mid q,g$ as shorthand for conditioning on these values. A predictor supplies $\pi_{\mathrm{pred}}(\cdot\mid Z,q,g)$; the true conditional distribution $\pi_{\mathrm{data}}(\cdot\mid Z,q,g)$ is the analytical reference. All expectations and entropies below are under $\pi_{\mathrm{data}}$, and the relevant losses and entropies are assumed finite.

For a discrete answer sequence $Y=(y_1,\ldots,y_M)$, including its end-of-sequence token, the negative log-likelihood is
\begin{equation}
\begin{aligned}
\ell(\pi_{\mathrm{pred}};Y,Z,q,g)
&=-\log\pi_{\mathrm{pred}}(Y\mid Z,q,g)\\
&=-\sum_{t=1}^{M}\log\pi_{\mathrm{pred}}(y_t\mid y_{<t},Z,q,g).
\end{aligned}
\label{eq:task_nll}
\end{equation}
Adding and subtracting the log of the true conditional distribution gives the standard entropy--divergence decomposition for logarithmic loss~\citep[Example~3]{gneiting2007strictly}:
\begin{equation}
\begin{aligned}
\mathcal L(\pi_{\mathrm{pred}};Z,q,g)
&=\mathbb E_{\pi_{\mathrm{data}}}[-\log\pi_{\mathrm{pred}}(Y\mid Z,q,g)\mid q,g]\\
&=\mathbb E_{\pi_{\mathrm{data}}}[-\log\pi_{\mathrm{data}}(Y\mid Z,q,g)\mid q,g]\\
&\quad+\mathbb E_{\pi_{\mathrm{data}}}\!\left[\log\frac{\pi_{\mathrm{data}}(Y\mid Z,q,g)}{\pi_{\mathrm{pred}}(Y\mid Z,q,g)}\,\middle|\,q,g\right]\\
&=H_{\pi_{\mathrm{data}}}(Y\mid Z,q,g)\\
&\quad+\mathbb E_{Z\mid q,g}D_{\mathrm{KL}}\!\left(
\pi_{\mathrm{data}}(\cdot\mid Z,q,g)\,\Vert\,\pi_{\mathrm{pred}}(\cdot\mid Z,q,g)
\right).
\end{aligned}
\label{eq:task_loss_decomposition}
\end{equation}
The first term is the uncertainty remaining given the available information. The second is the predictor's discrepancy from the true conditional distribution. It is nonnegative and vanishes at $\pi_{\mathrm{pred}}=\pi_{\mathrm{data}}(\cdot\mid Z,q,g)$. Taking the infimum over all conditional predictive distributions yields the identity stated in Section~\ref{sec:method}:
\begin{equation}
R_Z^*(q,g):=\inf_{\pi_{\mathrm{pred}}}\mathcal L(\pi_{\mathrm{pred}};Z,q,g)=H(Y\mid Z,q,g).
\label{eq:bayes_risk_identity}
\end{equation}
This infimum is an analytical reference, not an assumption that a fixed pretrained predictor attains it.

Substituting the complete visual source $\mathbf U_V$ and its reconstruction $\widehat{\mathbf U}_V(S)$ gives
\begin{equation}
\begin{aligned}
R_{\mathrm{full}}^*(q,g)&=H(Y\mid\mathbf U_V,q,g),\\
R_S^*(q,g)&=H(Y\mid\widehat{\mathbf U}_V(S),q,g).
\end{aligned}
\label{eq:bayes_risks}
\end{equation}
Assume $Y\perp\widehat{\mathbf U}_V(S)\mid\mathbf U_V,q,g$. Applying the conditional-information identity for the benefit of side information under log-loss~\citep[Corollary~1]{jiao2015justification},
\begin{equation}
\begin{aligned}
\Delta R^*(S;q,g)
&=H(Y\mid\widehat{\mathbf U}_V(S),q,g)-H(Y\mid\mathbf U_V,q,g)\\
&=H(Y\mid\widehat{\mathbf U}_V(S),q,g)
-H(Y\mid\mathbf U_V,\widehat{\mathbf U}_V(S),q,g)\\
&=I\!\left(Y;\mathbf U_V\mid\widehat{\mathbf U}_V(S),q,g\right)\\
&\le H\!\left(\mathbf U_V\mid\widehat{\mathbf U}_V(S),q,g\right)\\
&=H(\mathbf U_V\mid q,g)
-I\!\left(\mathbf U_V;\widehat{\mathbf U}_V(S)\mid q,g\right).
\end{aligned}
\label{eq:loss_bound_derivation}
\end{equation}
The second equality uses the conditional-independence assumption. The inequality follows because the semantic variables are discrete and conditional mutual information is at most the conditional entropy of either variable~\citep[Theorem~3.4]{polyanskiy2025information}. Since $H(\mathbf U_V\mid q,g)$ does not depend on $S$, maximizing the full-source recoverable mutual information minimizes this upper bound.

The same conditional mutual information is the expected KL divergence between the Bayes-optimal full and reconstructed-source answer distributions:
\begin{equation}
\begin{aligned}
\Delta R^*(S;q,g)
&=\mathbb E_{\mathbf U_V,\widehat{\mathbf U}_V(S)\mid q,g}\\
&\quad
D_{\mathrm{KL}}\!\left(
\pi_{\mathrm{data}}(\cdot\mid\mathbf U_V,q,g)
\,\Vert\,\pi_{\mathrm{data}}(\cdot\mid\widehat{\mathbf U}_V(S),q,g)
\right).
\end{aligned}
\label{eq:bayes_output_kl}
\end{equation}
This identity concerns the true conditional distributions, not the outputs of an arbitrary fixed MLLM.

\paragraph{Task-weighted recovery and the answer-risk bound.}
The main text defines a semantic query by drawing $J$ with $\Pr(J=v\mid q,g)=\Pr(J=v\mid Q=q)$, independently of $(\mathbf U_V,\widehat{\mathbf U}_V(S))$ given $q,g$. The retained set and recovery channel are fixed before $J$ is drawn. Expanding the queried mutual information defined in \eqref{eq:weighted_coverage} gives
\begin{equation}
\begin{aligned}
\mathcal I_{\mathrm{task}}(S;q,g)
&=I\!\left(U_J;\widehat U_J(S)\mid J,q,g\right)\\
&=\sum_{v\in V}\Pr(J=v\mid q,g)
I\!\left(U_v;\widehat U_v(S)\mid J=v,q,g\right)\\
&=\sum_{v\in V}\Pr(J=v\mid Q=q)
I\!\left(U_v;\widehat U_v(S)\mid q,g\right).
\end{aligned}
\label{eq:task_information_expansion}
\end{equation}
The second equality expands the conditional mutual information over the values of $J$. The third equality uses $\Pr(J=v\mid q,g)=\Pr(J=v\mid Q=q)$ together with the conditional independence of $J$ from $(\mathbf U_V,\widehat{\mathbf U}_V(S))$ given $q,g$, which lets the conditioning on $J=v$ be dropped. Thus the relevance weights arise from averaging over the queried source.

Both $U_v$ and $\widehat U_v(S)$ are deterministic coordinate projections of their corresponding full vectors. Applying the data-processing inequality~\citep[Theorem~3.7]{polyanskiy2025information} to each projection gives
\begin{equation}
\begin{aligned}
I\!\left(U_v;\widehat U_v(S)\mid q,g\right)
&\le I\!\left(\mathbf U_V;\widehat U_v(S)\mid q,g\right)\\
&\le I\!\left(\mathbf U_V;\widehat{\mathbf U}_V(S)\mid q,g\right).
\end{aligned}
\label{eq:coordinate_information_bound}
\end{equation}
The query probabilities are nonnegative and sum to one, so
\begin{equation}
\begin{aligned}
\mathcal I_{\mathrm{task}}(S;q,g)
&\le\sum_{v\in V}\Pr(J=v\mid Q=q)
I\!\left(\mathbf U_V;\widehat{\mathbf U}_V(S)\mid q,g\right)\\
&=I\!\left(\mathbf U_V;\widehat{\mathbf U}_V(S)\mid q,g\right).
\end{aligned}
\label{eq:task_information_bound_derivation}
\end{equation}
This bounds the queried mutual information by whole-source recovered information without requiring independence among the source variables $U_v$. Combining it with \eqref{eq:loss_bound_derivation} yields
\begin{equation}
\begin{aligned}
\Delta R^*(S;q,g)
&\le H(\mathbf U_V\mid q,g)
-I\!\left(\mathbf U_V;\widehat{\mathbf U}_V(S)\mid q,g\right)\\
&\le H(\mathbf U_V\mid q,g)-\mathcal I_{\mathrm{task}}(S;q,g),
\end{aligned}
\label{eq:task_risk_bound_derivation}
\end{equation}
which establishes \eqref{eq:task_information_risk_bound}. With the query distribution fixed, the entropy term is independent of $S$, so maximizing $\mathcal I_{\mathrm{task}}$ minimizes this looser upper bound. The query distribution is a modeling choice; a retained set that minimizes this bound need not minimize the answer risk itself.

\paragraph{Log-loss interpretation of the semantic query.}

A predictor given $J,q,g$ has prior Bayes risk $R_{\mathrm{query,prior}}^*(q,g)=H(U_J\mid J,q,g)$. Observing the recovered unit reduces its risk to
\begin{equation}
\begin{aligned}
R_{\mathrm{query,recovered}}^*(S;q,g)
&=H(U_J\mid\widehat U_J(S),J,q,g)\\
&=\sum_{v\in V}\Pr(J=v\mid Q=q)H(U_v\mid q,g)
-\mathcal I_{\mathrm{task}}(S;q,g).
\end{aligned}
\label{eq:query_loss_information}
\end{equation}
The first term is the prior Bayes risk $R_{\mathrm{query,prior}}^*(q,g)$ and is independent of $S$. Maximizing $\mathcal I_{\mathrm{task}}$ therefore minimizes the Bayes log-loss of the modeled semantic query. Its prediction target is $U_J$, rather than the answer $Y$. For the answer task, \eqref{eq:task_information_bound_derivation} and \eqref{eq:task_information_risk_bound} instead establish that this criterion minimizes a looser upper bound on excess Bayes risk, not the answer loss of a fixed pretrained MLLM.

\subsection{Semantic Erasure with Coverage and Reliability Gates}
\label{sec:appendix_erasure}

Fix the prompt $Q=q$, geometric context $G=g$, and retained set $S$. The coverage probability is $M_v(S;q,g)=\Pr(E_v=1\mid q,g)$, and the representation reliability is $\Pr(C_v=1\mid q,g)$. For a fixed recovery channel, we treat this reliability probability as invariant to $S$. As stated in Section~\ref{sec:method}, $C_v$ and $E_v$ are conditionally independent, and the pair $(C_v,E_v)$ is independent of $U_v$ given $q,g$. These are assumptions of the semantic erasure model.

The event $C_vE_v=1$ is exactly the event that both binary gates equal one. Conditional independence therefore gives
\begin{equation}
\begin{aligned}
\Pr(C_vE_v=1\mid q,g)
&=\Pr(C_v=1,E_v=1\mid q,g)\\
&=\Pr(C_v=1\mid q,g)\Pr(E_v=1\mid q,g)\\
&=\Pr(C_v=1\mid q,g)M_v(S;q,g),\\
\Pr(C_vE_v=0\mid q,g)
&=1-\Pr(C_v=1\mid q,g)M_v(S;q,g).
\end{aligned}
\label{eq:joint_success_probability}
\end{equation}
Let $R_v=C_vE_v$. The erasure symbol $\bot$ lies outside the alphabet of $U_v$, so observing $\widehat U_v(S)$ also identifies $R_v$. If $R_v=1$, the output reveals $U_v$ exactly. If $R_v=0$, the output is $\bot$ and the independence assumption leaves uncertainty $H(U_v\mid q,g)$. Thus
\begin{equation}
\begin{aligned}
H\!\left(U_v\mid\widehat U_v(S),q,g\right)
&=\Pr(R_v=1\mid q,g)\cdot 0\\
&\quad+\Pr(R_v=0\mid q,g)H(U_v\mid q,g)\\
&=\left[1-\Pr(C_v=1\mid q,g)M_v(S;q,g)\right]H(U_v\mid q,g).
\end{aligned}
\label{eq:erasure_entropy}
\end{equation}
Subtracting this residual entropy yields
\begin{equation}
\begin{aligned}
I\!\left(U_v;\widehat U_v(S)\mid q,g\right)
&=H(U_v\mid q,g)-H\!\left(U_v\mid\widehat U_v(S),q,g\right)\\
&=H(U_v\mid q,g)\Pr(C_v=1\mid q,g)M_v(S;q,g),
\end{aligned}
\label{eq:erasure_information_derivation}
\end{equation}
which is the standard erasure-channel mutual-information identity~\citep[Example~33.6]{polyanskiy2025information} with success probability $\Pr(C_v=1\mid q,g)M_v(S;q,g)$. Substituting it into the random-query expansion in \eqref{eq:task_information_expansion} gives \eqref{eq:weighted_coverage}. The argument requires no independence among the source variables $U_v$. The coverage probability remains abstract in this derivation; Section~\ref{sec:recover_variants} specifies MiCo's coverage model and observable proxies.

\subsection{Coverage Properties and Greedy Approximation}
\label{sec:appendix_coverage_guarantee}

The following result applies the standard weighted facility-location and greedy-maximization arguments~\citep{krause2014submodular,nemhauser1978analysis} to MiCo's observable objective. Fix one selector call on a finite token set $V$ with features $\{x_v\}_{v\in V}$. Its source weights $\widehat W_v:=\widehat{\mathcal H}_v\widehat{\mathcal R}_v\widehat{\mathcal P}_v\ge0$ (Eq.~\ref{eq:mico_stage_proxies}) and similarities $\kappa_X(v,i)=[\cos(x_v,x_i)]_+\in[0,1]$ remain fixed as the retained set grows. The objective is
\begin{equation}
\widehat{\mathcal I}_{\mathrm{task}}(S)
=\sum_{v\in V}\widehat W_v\widehat M_v(S),
\qquad
\widehat M_v(S)=\max_{s\in S}\kappa_X(v,s),
\label{eq:proxy_coverage_objective}
\end{equation}
with $\widehat M_v(\varnothing)=0$.

\begin{proposition}[Coverage properties and greedy guarantee]
\label{prop:coverage_greedy}
The objective in \eqref{eq:proxy_coverage_objective} is normalized, nonnegative, monotone, and submodular. Let $b=\min(B,|V|)$. If $b=0$, the empty set is optimal. For $b\ge1$, let $S^\star$ maximize this objective over $S\subseteq V$ with $|S|\le b$, and let $S_b$ be the set returned after $b$ exact greedy additions starting from the empty set. Then
\begin{equation}
\begin{aligned}
\widehat{\mathcal I}_{\mathrm{task}}(S_b)
&\ge\left[1-\left(1-\frac1b\right)^b\right]
\widehat{\mathcal I}_{\mathrm{task}}(S^\star)\\
&\ge(1-1/e)\widehat{\mathcal I}_{\mathrm{task}}(S^\star).
\end{aligned}
\label{eq:coverage_greedy_ratio}
\end{equation}
\end{proposition}

\begin{proof}
Normalization follows from $\widehat M_v(\varnothing)=0$. Nonnegative weights and $0\le\widehat M_v(S)\le1$ give
\[
0\le\widehat{\mathcal I}_{\mathrm{task}}(S)\le\sum_{v\in V}\widehat W_v.
\]
For $A\subseteq D\subseteq V$, taking a maximum over the larger set gives $\widehat M_v(A)\le\widehat M_v(D)$, so the objective is monotone. To establish submodularity, consider an unselected token $i\in V\setminus D$. Its marginal gain is
\begin{equation}
\begin{aligned}
\Delta\widehat{\mathcal I}_{\mathrm{task}}(i\mid A)
&=\sum_{v\in V}\widehat W_v[\kappa_X(v,i)-\widehat M_v(A)]_+\\
&\ge\sum_{v\in V}\widehat W_v[\kappa_X(v,i)-\widehat M_v(D)]_+\\
&=\Delta\widehat{\mathcal I}_{\mathrm{task}}(i\mid D).
\end{aligned}
\label{eq:coverage_diminishing_returns}
\end{equation}
This is the diminishing-returns definition of submodularity. Multiplication by fixed nonnegative source weights preserves the inequality.

For the approximation bound, let $S_t$ denote the first $t$ greedy selections, with $S_0=\varnothing$ and $0\le t<b$. Monotonicity, diminishing returns, and the greedy choice imply
\begin{equation}
\begin{aligned}
\widehat{\mathcal I}_{\mathrm{task}}(S^\star)-\widehat{\mathcal I}_{\mathrm{task}}(S_t)
&\le\widehat{\mathcal I}_{\mathrm{task}}(S_t\cup S^\star)-\widehat{\mathcal I}_{\mathrm{task}}(S_t)\\
&\le\sum_{i\in S^\star\setminus S_t}\Delta\widehat{\mathcal I}_{\mathrm{task}}(i\mid S_t)\\
&\le b\bigl[\widehat{\mathcal I}_{\mathrm{task}}(S_{t+1})-\widehat{\mathcal I}_{\mathrm{task}}(S_t)\bigr].
\end{aligned}
\label{eq:coverage_greedy_step}
\end{equation}
Writing the remaining objective gap as $D_t=\widehat{\mathcal I}_{\mathrm{task}}(S^\star)-\widehat{\mathcal I}_{\mathrm{task}}(S_t)$ gives $D_{t+1}\le(1-1/b)D_t$. Since $D_0=\widehat{\mathcal I}_{\mathrm{task}}(S^\star)$, iterating for $b$ steps proves the first inequality in \eqref{eq:coverage_greedy_ratio}. The second follows from $(1-1/b)^b\le e^{-1}$.
\end{proof}

Greedy selection attains at least a $(1-1/e)$ fraction of the optimal coverage value. Equivalently, the objective gap is at most $\widehat{\mathcal I}_{\mathrm{task}}(S^\star)/e$. It is not a bound on answer loss or benchmark error. The proof requires nonnegative similarities, not a positive-semidefinite similarity matrix.

The guarantee is local to each fixed selector call. In Stage~2, the comparison optimum uses the candidate pool and contextualized features produced after Stage~1. When crops or tiles receive separate budgets, it applies within each group with its assigned budget. It does not establish the same approximation ratio for joint optimization of both stages or for unrestricted token allocation across groups. If $B\ge|V|$, all available tokens are retained and the solution is exact.

\FloatBarrier

\section{Implementation and Evaluation Details}
\label{sec:appendix_implementation}

\subsection{Shared MiCo Selector}
\label{sec:appendix_selector}

Both pruning stages use Algorithm~\ref{alg:weighted_coverage}. Its inputs are the visual features at the current stathree nonnegative proxy vectors, and a retained-token budget. In both stages, coverage is the positive cosine similarity between visual tokens in that stage's feature space. Text-to-visual attention affects the Stage-2 weights, not the coverage kernel. Feature magnitudes are measured before the normalization used to compute cosine similarity.

The proxy-based version of \eqref{eq:weighted_coverage} is $\widehat{\mathcal I}_{\mathrm{task}}(S;X)=\sum_{v\in V}\widehat W_v\widehat M_v(S)$, where $\widehat W_v=\widehat{\mathcal H}_v\widehat{\mathcal R}_v\widehat{\mathcal P}_v$, $\widehat M_v(S)=\max_{s\in S}\kappa_X(v,s)$, and $\widehat M_v(\varnothing)=0$. The marginal gain used to select the next token is
\begin{equation}
\Delta\widehat{\mathcal I}_{\mathrm{task}}(i\mid S;X)=\sum_{v\in V}\widehat W_v[\kappa_X(v,i)-\widehat M_v(S)]_+.
\label{eq:coverage_gain}
\end{equation}
The fixed-call assumptions, submodularity proof, and approximation guarantee are given in Appendix~\ref{sec:appendix_coverage_guarantee}; this section focuses on the executable selector.

\begin{algorithm}[!t]
\caption{MiCo}
\label{alg:weighted_coverage}
\begin{algorithmic}[1]
\REQUIRE Visual features $X=\{x_v\}_{v\in V}$; nonnegative proxies $\widehat{\mathcal H},\widehat{\mathcal R},\widehat{\mathcal P}$; retained budget $B$
\ENSURE Retained visual-token indices $S$ in their original order
\STATE $\widehat W_v\leftarrow\widehat{\mathcal H}_v\widehat{\mathcal R}_v\widehat{\mathcal P}_v$ for every $v\in V$
\STATE $\kappa(v,i)\leftarrow[\cos(x_v,x_i)]_+$ for all $v,i\in V$
\STATE $S\leftarrow\varnothing$; $\widehat M_v\leftarrow0$ for every $v\in V$
\FOR{$t=1$ to $\min(B,|V|)$}
    \STATE $i^\star\leftarrow\arg\max_{i\in V\setminus S}\sum_{v\in V}\widehat W_v[\kappa(v,i)-\widehat M_v]_+$
    \STATE $S\leftarrow S\cup\{i^\star\}$
    \STATE $\widehat M_v\leftarrow\max\{\widehat M_v,\kappa(v,i^\star)\}$ for every $v\in V$
\ENDFOR
\RETURN $S$ sorted by the original token order
\end{algorithmic}
\end{algorithm}

Table~\ref{tab:mico_stage_proxies} summarizes the stage-specific proxy assignments. The features $x_v^{\mathrm{vis}}$ and $x_v^{(\ell_\star)}$ denote visual-encoder outputs and contextualized visual states, respectively. Here $\overline A^{\mathrm{vis}}_{\mathrm{CLS}\to v}$ averages visual CLS attention over heads, and $\overline A^{(\ell_\star)}_{\mathrm{text}\to v}$ averages text-to-visual attention over available text queries and heads. For vision encoders without a CLS token, the architecture-native pooled visual-attention analogue replaces CLS attention. Setting a proxy to one gives that factor uniform weight; feature norms and attention are proxies rather than direct measurements of semantic entropy or reliability.

\begin{table}[H]
\centering
\small
\setlength{\tabcolsep}{6pt}
\caption{Stage-specific inputs to MiCo's shared weighted-coverage selector.}
\label{tab:mico_stage_proxies}
\begin{tabular}{lccc}
\toprule
\textbf{Stage} & \textbf{Energy $\widehat{\mathcal H}_v$} & \textbf{Reliability $\widehat{\mathcal R}_v$} & \textbf{Task relevance $\widehat{\mathcal P}_v$} \\
\midrule
Stage 1 & $\lVert x_v^{\mathrm{vis}}\rVert_2$ & $\overline A^{\mathrm{vis}}_{\mathrm{CLS}\to v}$ & $1$ \\
Stage 2 & $\lVert x_v^{(\ell_\star)}\rVert_2$ & $1$ & $\overline A^{(\ell_\star)}_{\mathrm{text}\to v}$ \\
\bottomrule
\end{tabular}
\end{table}

\FloatBarrier
\subsection{Empirical Checks for the Proxy Assignments}
\label{sec:appendix_proxy_checks}

The theoretical factors are latent, so these experiments validate the observable proxies against
operational references rather than claiming to measure the factors exactly. The checks below are
diagnostic evidence for the intended role of each proxy; they do not change the selector or its weights.

\paragraph{Stage-1 information amount.}
We use dense COCO-Stuff labels to construct an empirical entropy reference. The image pool is split
into disjoint A/B halves of 2,500 COCO val2017 images. For each aspect bin $a$ and native visual
position $v$, A supplies category counts $n_{a,v,c}$ over the 182 stored COCO-Stuff classes. With the
fixed Jeffreys smoothing used in the analysis,
\[
\widehat p_{a,v,c}=\frac{n_{a,v,c}+0.5}{N_{a,v}+0.5\times182},\qquad
\widehat H_{a,v}=-\sum_{c=1}^{182}\widehat p_{a,v,c}\log\widehat p_{a,v,c},
\]
where $N_{a,v}=\sum_c n_{a,v,c}$ excludes the dataset's void label. On the independent B pool,
we extract the raw pre-projection Stage-1 feature $x_{i,a,v}$ and compare $\widehat H_{a,v}$ with its
position-wise mean norm $\overline n_{a,v}=|B_{a,v}|^{-1}\sum_{i\in B_{a,v}}\lVert x_{i,a,v}\rVert_2$.
% TODO(authors): confirm which estimate was pre-specified; the text and the caption previously disagreed.
Figure~\ref{fig:proxy_h_entropy_norm} shows the two-sided 1\% trimmed fit, which gives Spearman $\rho=0.36832$; the pre-specified
primary estimate is the untrimmed $\rho=0.35943$ (Tukey $1.5\times\mathrm{IQR}$ trimming: $0.31021$). Thus the evidence is a
moderate position-level association for the Stage-1 norm, not a token-level identity with
$H(U_v\mid q,g)$: the corresponding per-token association is $\rho=0.00699$, and this diagnostic
contains no prompt variable. It therefore estimates an $H(U_v\mid G)$-type quantity and is used
only to support the task-agnostic Stage-1 proxy.

\begin{figure}[!htbp]
    \centering
    \includegraphics[width=0.96\linewidth]{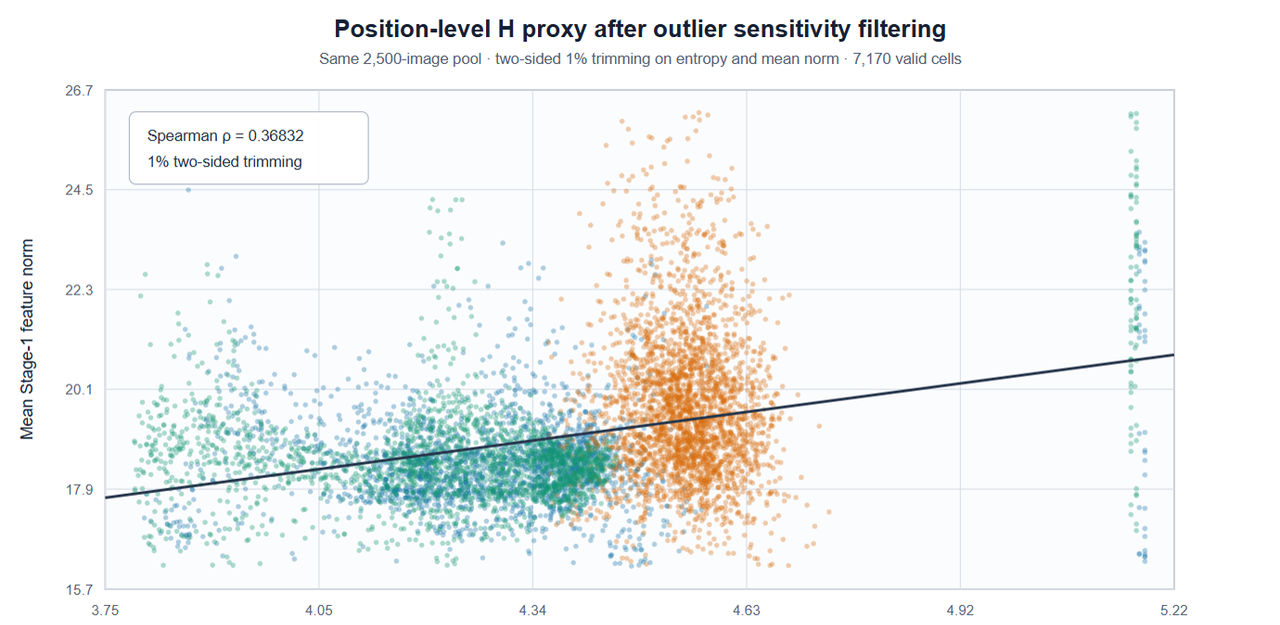}\\[-0.2em]
    {\small Empirical category entropy $H(U_v\mid G)$ (nats)}\\[-0.1em]
    {\scriptsize
      \textcolor{blue!70!black}{\rule[0.25ex]{0.8em}{0.8em}}\,aspect 1\quad
      \textcolor{yellow!75!black}{\rule[0.25ex]{0.8em}{0.8em}}\,aspect 2\quad
      \textcolor{cyan!55!black}{\rule[0.25ex]{0.8em}{0.8em}}\,aspect 3\quad
      \textcolor{black!65}{Line: least-squares trend}}
    \caption{Position-level validation of the Stage-1 information-amount proxy. Dense-label entropy from
    COCO-Stuff is compared with the mean raw Stage-1 feature norm at matched aspect bin and native
    position. The displayed fit uses a two-sided 1\% trim ($\rho=0.36832$); the pre-specified primary estimate is untrimmed ($\rho=0.35943$). This is a
    geometry-conditioned $H(U_v\mid G)$ diagnostic rather than a direct measurement of
    prompt-conditioned entropy.}
    \label{fig:proxy_h_entropy_norm}
\end{figure}

\paragraph{Stage-1 representation reliability.}
We use annotated-category decoding as an operational reference for whether a visual representation
preserves semantic content. On LLaVA-NeXT-7B, we train linear and MLP probes on 1,688 clean COCO
objects, select checkpoints on 382 validation objects, and test on 200 objects from new, disjoint
images. Each probe receives the unit-normalized mean visual feature within the target region;
neither CLS attention nor feature norm is supplied as an extra input. We use three fixed head seeds
and compare target-region masking with equal-area context masking. Decoding correctness is an
observable diagnostic, not a direct observation of the latent reliability gate $C_v$.

Figure~\ref{fig:proxy_decoding_reference} provides three complementary checks. Panel~(a) shows
that target masking reduces category accuracy from 98.0\% to 81.4\% at 75\% replacement, whereas
equal-area context masking leaves it at 97.9\%; full target masking reduces accuracy to 64.3\%.
This establishes that the decoding reference is sensitive to target content. Panels~(b) and~(c)
then test whether the reliability proxy orders successful and unsuccessful decoding trials. Within
each fixed object, masking arm, strength, and head seed, concordance is the fraction of
correct--incorrect pairs whose CLS-attention score is higher on the correct trial, with half credit
for ties. Scores average CLS-to-patch attention over heads and the fixed target region; feature norm
is a comparison score averaged over the same region. Concordance is computed from six perturbation
repeats, then averaged across defined seeds, strengths, and images, with no-association reference 0.5.
All-correct and all-incorrect strata have undefined concordance and are not assigned a chance score.

For target masking, CLS-attention concordance is 0.557 (95\% CI [0.479, 0.629]) with the linear
probe and 0.564 ([0.484, 0.636]) with the MLP, compared with feature-norm estimates of 0.459 and
0.499. Thus the positive attention trend persists under both probe capacities, providing suggestive
support for its reliability role. Both intervals include 0.5, so this is not a statistically established
above-chance association. Target concordance uses 57 and 59 informative images, respectively;
context concordance uses only 4 and 3 and is correspondingly uncertain. Confidence intervals use
1,000 image-cluster bootstrap resamples; all 200 images contribute to decoding accuracy.

\begin{figure}[!htbp]
    \centering
    \includegraphics[width=\linewidth]{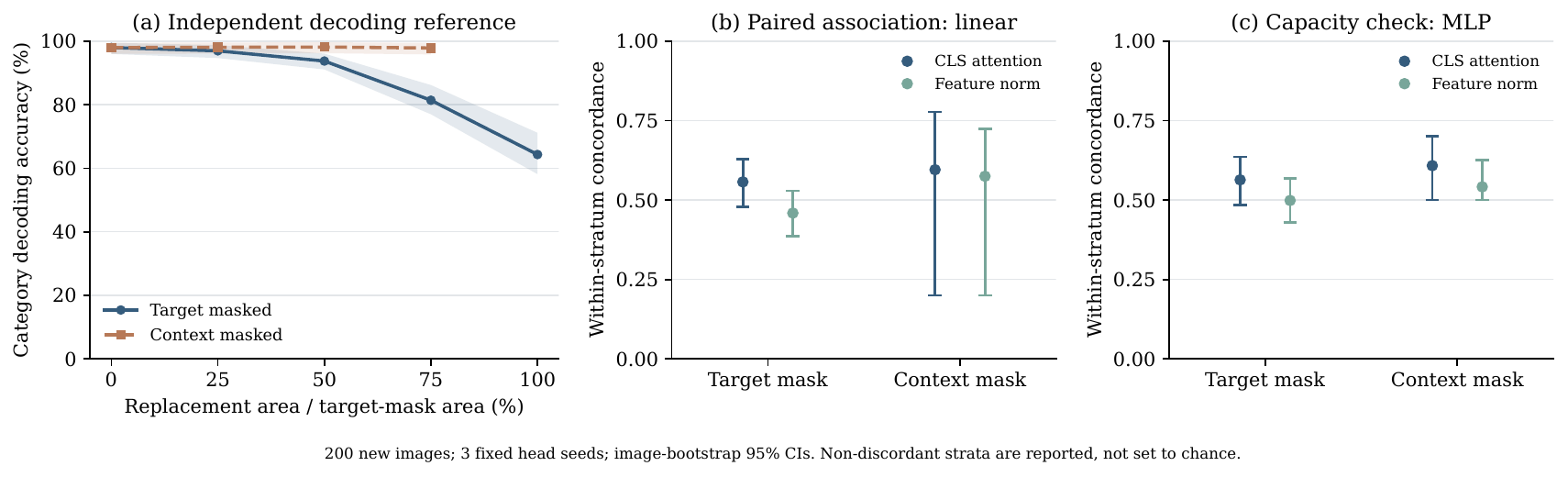}
    \caption{Three complementary diagnostics for Stage-1 reliability on LLaVA-NeXT-7B.
    (a) Category decoding is sensitive to target masking, but largely stable under context masking.
    (b) CLS attention has a positive concordance point estimate with decoding success under a linear
    probe. (c) The same direction persists with an MLP probe. Shading and error bars show
    image-bootstrap 95\% confidence intervals over 200 test images and three fixed head seeds;
    concordance excludes strata without both successes and failures. The attention trends are
    suggestive, not conclusive, because the target-mask intervals include the 0.5 reference.}
    \label{fig:proxy_decoding_reference}
\end{figure}

\paragraph{Stage-2 task relevance.}
For each held-out RefCOCO image--query pair, we convert the annotated box--cell overlap into a
normalized target distribution over the Stage-1 survivor cells. This is an operational query-relevance
target conditioned on the image and survivor pool, not a direct observation of $J$. We compare it with
the text-to-visual attention averaged over text queries and heads, without post-hoc calibration.
Figure~\ref{fig:proxy_query_attention} shows the local attention-mass interval $[0,0.08]$, using
the original rank-decile bins without rebinning. Across the complete held-out set of 480 queries from
240 images, the mean per-query Spearman correlation is $\rho=0.0697$ (image-bootstrap 95\% CI
$[0.0433,0.0976]$), indicating a weak positive spatial association. The higher-attention bins outside
this view have decreasing target mass, so the local trend does not establish full-range probability
calibration.

\begin{figure}[!htbp]
    \centering
    \includegraphics[width=0.78\linewidth]{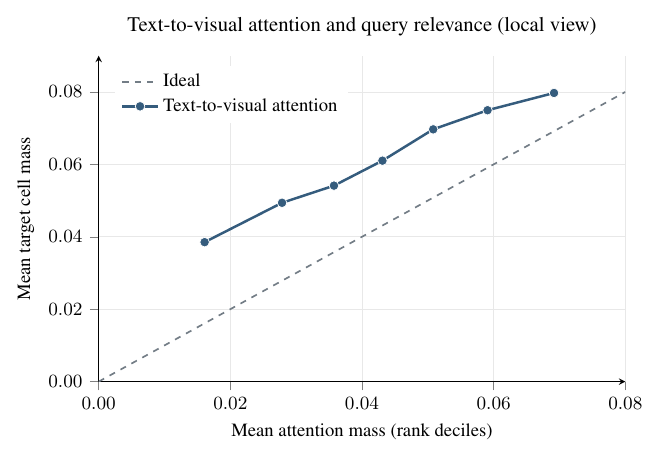}
    \caption{Local-range view of the Stage-2 task-relevance diagnostic on LLaVA-NeXT-7B,
    restricted to attention masses $0$--$0.08$. Points retain the original rank-decile means;
    the dashed line is the ideal calibration reference. The target is normalized RefCOCO box--cell
    overlap conditioned on the Stage-1 survivor pool. Higher-attention bins are outside the displayed
    range; the correlation reported in the text uses the complete held-out data.}
    \label{fig:proxy_query_attention}
\end{figure}
\FloatBarrier

\subsection{Architecture and Budget Mapping}
\label{sec:appendix_budget}

Let $N_0$ be the original visual-token count and $T$ the target layer-average count. Unless otherwise stated, Stage~1 retains a pool of $N_1=\min(2T,N_0)$ tokens. With $L$ decoder blocks, refinement uses the text-to-visual attention of the one-based block $K=\ell_\star$ and prunes immediately after it, where $1\le\ell_\star<L$, so that the first $\ell_\star$ blocks process $N_1$ visual tokens and the remaining $L-\ell_\star$ blocks process $R$. Before integer rounding and clipping, matching the target average gives
\begin{equation}
T=\frac{\ell_\star N_1+(L-\ell_\star)R}{L}.
\label{eq:layer_average_target}
\end{equation}
Solving for the post-refinement count and enforcing integer and range constraints yields
\begin{equation}
R=\operatorname{clip}\!\left(\operatorname{round}\frac{TL-\ell_\star N_1}{L-\ell_\star},1,N_1\right).
\label{eq:layer_average_budget}
\end{equation}
Rounding and clipping can make the realized layer average differ from $T$. The selector can be applied directly to budget $T$ when decoder-stage refinement is omitted. Table~\ref{tab:pruning_schedule} gives the architecture-specific schedules. Its listed $R$ assumes the pool $N_1=2T$ before integer rounding; video budgets are per frame.

\begin{table}[!htbp]
\centering
\small
\setlength{\tabcolsep}{6pt}
\caption{MiCo budgets and refinement layers by architecture; video budgets are per frame.}
\label{tab:pruning_schedule}
\begin{tabular}{lcccc}
\toprule
\textbf{Backbone} & \textbf{Decoder blocks $L$} & \textbf{Refinement $\ell_\star$} & \textbf{Stage-1 keep} & \textbf{Stage-2 keep $R$} \\
\midrule
LLaVA-1.5/NeXT-7B & 32 & 7 & $2T$ & $\operatorname{round}(18T/25)$ \\
LLaVA-1.5-13B & 40 & 8 & $2T$ & $\operatorname{round}(3T/4)$ \\
LLaVA-NeXT-13B & 40 & 14 & $2T$ & $\operatorname{round}(6T/13)$ \\
Qwen2.5-VL-7B & 28 & 2 & $2T$ & $\operatorname{round}(12T/13)$ \\
InternVL3-8B & 28 & 4 & $2T$ & $\operatorname{round}(5T/6)$ \\
Qwen3-VL-8B & 36 & 6 & $2T$ & $\operatorname{round}(4T/5)$ \\
Qwen3.5-9B & 32 & 12 & $2T$ & $\operatorname{round}_{\rm even}(2T/5)$ \\
LLaVA-Video-7B & 28 & 6 & $2T$ & $\operatorname{round}(8T/11)$ \\
\bottomrule
\end{tabular}
\end{table}

For multi-crop LLaVA-NeXT inputs and dynamically tiled InternVL inputs, the Stage~1 budget is distributed over the natural visual groups. Integer division and a minimum of one token per group can make the realized count differ slightly from $T$ or $2T$. LLaVA, Qwen3-VL, and Qwen3.5 recompute $R$ using \eqref{eq:layer_average_budget} and the realized Stage~1 survivor count. For Qwen2.5-VL and InternVL, $R$ is instead computed from the nominal $N_1=2T$ rather than the realized survivor count, so their realized layer average can differ slightly when clipping or group rounding prevents exactly $2T$ survivors.

\paragraph{Video input protocol.}
LLaVA-Video-7B is evaluated through lmms-eval with the video repository's native task prompts, generation settings, and official scorers. Each video is uniformly sampled to at most 64 frames (\texttt{force\_sample=False}, so shorter videos contribute fewer frames), and every frame is encoded by the SigLIP vision tower and spatially pooled by the model's native pooling to 169 visual tokens, so an unpruned 64-frame input carries $64\times169$ visual tokens. Video budgets are nominal per-frame layer averages: $T=64/32/16$ correspond to pruning 62.1\%/81.1\%/90.5\% of the 169 tokens per frame. MiCo applies the Stage~1 pool ($2T$ per frame) and the refinement mapping of \eqref{eq:layer_average_budget} to the concatenated frame tokens, and Table~\ref{tab:mico_video_latency_breakdown} reports the realized per-frame layer averages. Efficiency measurements fix 64 frames at $384\times384$ resolution with FP16 and SDPA on one A800 80GB GPU (Appendix~\ref{sec:appendix_efficiency}).

\subsection{Computational Complexity}
\label{sec:appendix_complexity}

Following the standard accounting for token-pruning selectors~\cite{zhang2024sparsevlm}, constructing a dense visual-token similarity matrix for $N$ tokens of width $d$ costs $O(N^2d)$ time and $O(N^2)$ memory. Exact greedy coverage with a retained budget $B$ adds $O(BN^2)$ arithmetic for the marginal updates. Thus, for the two-stage path with original token count $N_0$, Stage~1 pool $N_1$, and final budget $R$, the selector-specific work scales as $O(N_0^2d_1+N_1N_0^2+N_1^2d_2+RN_1^2)$, up to implementation constants, where $d_1$ and $d_2$ are the widths of the vision-encoder features and the decoder hidden states.
For video inputs, Stage~1 is applied to each frame separately, so the first two terms use the per-frame token count (169 for LLaVA-Video-7B) and grow only linearly with the number of frames, whereas Stage~2 operates on the concatenated survivor pool $N_1=F\cdot 2T$ over all $F$ frames; the quadratic Stage~2 terms therefore dominate the selector cost for long videos (Appendix~\ref{sec:appendix_efficiency}).

We validate this scaling on the executed implementations by counting logical dense-equivalent FLOPs, including the vision encoder, projector, decoder, output head, and both selector stages. These counts are an operation-accounting check rather than hardware-instruction counts; the per-component breakdown and synchronized latency validation are given in Appendix~\ref{sec:appendix_efficiency} and Table~\ref{tab:mico_gs_full_flops}.

\FloatBarrier

\subsection{Qwen3.5-9B: Hybrid-Decoder Adaptation}
\label{sec:appendix_qwen35_implementation}

Qwen3.5-9B~\cite{qwen35_9b_modelcard} has 32 decoder blocks, repeating three Gated DeltaNet blocks followed by one full-attention block. We retain all decoder blocks and remove only visual tokens. Both evaluated budgets ($T=256$ and $T=128$) use the same configuration: a Stage~1 pool of $N_1=\min(2T,N_0)$ tokens (multiplier $n=2$, as for the other models), refinement after the 1-based block $K=12$, which is a full-attention block, and full-attention QK relevance averaged over all non-visual prompt positions.

Stage~1 follows the CLS-free Qwen adaptation. We observe the final vision-block attention, pool its query and key axes over the native $2\times2$ merger groups, and average over query groups and heads. The coverage features are the means of the four pre-merger patch features in each group. Their norms multiply the pooled attention to form source weights for the shared greedy coverage selector. The selected group indices retain the corresponding native merged embeddings; the vision encoder and merger execute unchanged, and selection occurs before the first decoder block.

Stage~2 uses block~12's native Q/K projections, Q/K normalization, rotary positions, and grouped-query key expansion. For each non-visual prompt query and head, we normalize QK scores over the current visual keys, then average the probabilities over queries and heads. The block executes normally. The block's output visual hidden states serve as the Stage~2 coverage features, and their $\ell_2$ norms multiplied by these relevance scores give the Stage~2 source weights; the weights are divided by their sum (uniform weights are used if the sum is zero), which does not change the greedy selection. Stage~1 attention weights are not multiplied in again. We keep only the selected visual states, keep all text states, retain each surviving token's original multimodal (mRoPE) position index, and run blocks~13--32 on the shortened sequence.

The second-stage budget is the instance of \eqref{eq:layer_average_budget} with $L=32$ and $K=12$, $R=\operatorname{round}\big((32T-12N_1)/20\big)$ with $N_1=\min(2T,N_0)$, where $\operatorname{round}$ is round-half-to-even; instead of clipping, the Qwen3.5 path raises an error if $R$ falls outside $[1,N_1]$. For $N_0\ge2T$, the schedules are $512\rightarrow102$ at $T=256$ and $256\rightarrow51$ at $T=128$, realizing layer averages of 255.750 and 127.875. During cached decoding each block keeps the prefill state it computed: blocks~1--12 keep KV caches (full-attention blocks) or recurrent states (Gated DeltaNet blocks) built from the full $N_1$-token sequence and are not truncated after pruning, whereas blocks~13--32 hold caches built from the pruned sequence. Generated tokens receive position indices that continue from the original, unpruned prompt length.

The evaluated scope is one unpadded image per request, batch size one, SDPA, greedy cached generation, thinking disabled, and a 2,048-token generation cap. Empty and capped answers remain in the scoring denominator. Each configuration has 23,715 predictions: AI2D 3,088; POPE 5,127 expanded to 9,000 category entries; HallusionBench 951; MME 2,374 questions (1,187 images, two questions each); MMBench EN/CN 4,329 entries and 1,164 Circular groups each; MMStar 1,500; and ScienceQA-IMG 2,017. MMStar's released questions are used unchanged, including items whose answer options are malformed in the original release.

As with Qwen2.5-VL and Qwen3-VL, we use one fixed configuration at both budgets (here $n=2$, $K=12$; per-model layers are listed in Table~\ref{tab:layer_configurations}) and report the resulting eight-task evaluation. Acc follows the common Qwen convention (MME$/20$).

\FloatBarrier

\section{Efficiency Measurement Details}
\label{sec:appendix_efficiency}

\subsection{Measurement Protocol}

The main text reports one-token latency for the LLaVA-NeXT backbones and complete-answer latency for LLaVA-Video, together with peak allocated memory and logical FLOPs. Latency uses synchronized CUDA wall time after ten warm-up requests; FLOPs are counted in a separate pass and are not used as a latency proxy. Loading, preprocessing, and host-to-device transfer are outside the generation timer. Budgets are nominal layer-average targets; the token counts realized by the executed decoder schedules are reported for the MiCo breakdowns in Tables~\ref{tab:mico_gs_latency_components} and~\ref{tab:mico_video_latency_breakdown}.

% TODO(authors): state which attention backend the baseline rows and the NeXT-13B rows use.
The NeXT-7B MiCo rows use the default MiCo implementation, which fuses the greedy selector and runs the language model with FlashAttention2; the backend changes kernel fusion and memory traffic, but not the selector objective, token budget, or selected tokens. We report logical dense-equivalent FLOPs rather than hardware instruction counts. The accounting includes the executed vision path, projector, decoder, output head, and pruning stages; fused attention is counted by its dense-equivalent arithmetic, while indexing, sorting, synchronization, and diagnostic-only work are excluded consistently.

\subsection{Backbone-level efficiency results}

\paragraph{LLaVA-NeXT backbones.}
Table~\ref{tab:efficiency_next7b} reports median latency, peak-memory, and logical-FLOP comparisons at $T=640,320,160$ for NeXT-7B and NeXT-13B under one protocol (one A800 80GB GPU, batch size one, one generated token, ten warm-ups and 50 measured requests). The two backbones are placed side by side for direct comparison; the Pareto figures use arithmetic-mean timing, so their speedups differ slightly from median-based ratios. Under the layer-average accounting of Appendix~\ref{sec:appendix_budget}, FastV and SparseVLM have no feasible $T=160$ schedule on NeXT-7B and only a degenerate one on NeXT-13B (Appendix~\ref{sec:appendix_full_results}), so their $T=160$ rows are left blank.

% One longtable with the two backbones as column groups, so the comparison stays side by side
% but can break across pages instead of leaving a gap when the block does not fit.
% Merged from computational_efficiency_next7b.tex / computational_efficiency_next13b.tex so the
% comparison can break across pages (longtable). 7B: historical comparison rows, MiCo uses the
% refreshed GS cost controls. 13B: historical matched comparison; latency uses medians and speedups
% use the displayed Vanilla reference. Bold/underline marks (13B only) are carried over unchanged.
\begingroup
\scriptsize
\setlength{\tabcolsep}{3.5pt}
\setlength{\LTleft}{\fill}
\setlength{\LTright}{\fill}
\setlength{\LTcapwidth}{\linewidth}
\renewcommand{\arraystretch}{1.05}
\begin{longtable}{@{}l*{6}{c}@{}}
\caption{LLaVA-NeXT-7B and LLaVA-NeXT-13B efficiency comparison (median latency; speedups and reductions are relative to each backbone's Vanilla row).}
\label{tab:efficiency_next7b}\label{tab:efficiency_next13b}\\
\toprule
\multirow{2}{*}{Method} & \multicolumn{3}{c}{LLaVA-NeXT-7B} & \multicolumn{3}{c}{LLaVA-NeXT-13B} \\
\cmidrule(lr){2-4}\cmidrule(lr){5-7}
& \shortstack{Latency (ms)\\(Speedup)} & \shortstack{Peak memory (GiB)\\(Reduction)} & \shortstack{Logical FLOPs (T)\\(Reduction)} & \shortstack{Latency (ms)\\(Speedup)} & \shortstack{Peak memory (GiB)\\(Reduction)} & \shortstack{Logical FLOPs (T)\\(Reduction)} \\
\midrule
\endfirsthead
\multicolumn{7}{l}{\textit{Table \thetable\ (continued): LLaVA-NeXT-7B and LLaVA-NeXT-13B efficiency comparison}}\\
\toprule
\multirow{2}{*}{Method} & \multicolumn{3}{c}{LLaVA-NeXT-7B} & \multicolumn{3}{c}{LLaVA-NeXT-13B} \\
\cmidrule(lr){2-4}\cmidrule(lr){5-7}
& \shortstack{Latency (ms)\\(Speedup)} & \shortstack{Peak memory (GiB)\\(Reduction)} & \shortstack{Logical FLOPs (T)\\(Reduction)} & \shortstack{Latency (ms)\\(Speedup)} & \shortstack{Peak memory (GiB)\\(Reduction)} & \shortstack{Logical FLOPs (T)\\(Reduction)} \\
\midrule
\endhead
\midrule
\multicolumn{7}{r}{\textit{Continued on next page}}\\
\endfoot
\bottomrule
\endlastfoot
Vanilla & 290.0 (1.00$\times$) & 15.8 (0.0\%) & 45.6 (0.0\%) & 484.1 (1.00$\times$) & 28.4 (0.0\%) & 85.1 (0.0\%) \\
\midrule
\multicolumn{7}{c}{\textit{Budget $T=640$ tokens per image}} \\
\midrule
FastV & 113.0 (2.57$\times$) & 16.0 (-1.7\%) & 11.8 (74.1\%) & \underline{154.9} (3.13$\times$) & 28.1 (1.1\%) & 20.9 (75.4\%) \\
SparseVLM & 137.1 (2.11$\times$) & 16.0 (-1.7\%) & 11.8 (74.2\%) & 176.3 (2.75$\times$) & 28.1 (1.1\%) & 21.0 (75.3\%) \\
DivPrune & 108.7 (2.67$\times$) & 14.2 (9.7\%) & 11.7 (74.4\%) & \textbf{154.2} (3.14$\times$) & \textbf{25.9} (8.8\%) & \underline{20.8} (75.6\%) \\
VisionZip & 118.8 (2.44$\times$) & 15.2 (3.9\%) & 11.7 (74.4\%) & 173.6 (2.79$\times$) & 27.0 (4.9\%) & \underline{20.8} (75.6\%) \\
VisPruner & 128.0 (2.27$\times$) & 14.7 (7.1\%) & 11.7 (74.4\%) & 179.2 (2.70$\times$) & \underline{26.5} (6.7\%) & \underline{20.8} (75.6\%) \\
HoloV & 117.4 (2.47$\times$) & 15.2 (3.9\%) & 11.7 (74.3\%) & 176.1 (2.75$\times$) & 27.0 (4.9\%) & \underline{20.8} (75.6\%) \\
MMTok & 208.3 (1.39$\times$) & 16.1 (-2.2\%) & 11.7 (74.4\%) & 259.2 (1.87$\times$) & 28.0 (1.4\%) & \underline{20.8} (75.6\%) \\
ApET & 125.4 (2.31$\times$) & 23.3 (-48.0\%) & 10.5 (77.1\%) & 185.9 (2.60$\times$) & 37.8 (-33.1\%) & \textbf{18.7} (78.0\%) \\
\rowcolor{micoRose}
\textbf{MiCo} & 135.5 (2.14$\times$) & 13.8 (12.8\%) & 11.6 (74.5\%) & 222.4 (2.18$\times$) & 27.0 (4.9\%) & \underline{20.8} (75.6\%) \\
\midrule
\multicolumn{7}{c}{\textit{Budget $T=320$ tokens per image}} \\
\midrule
FastV & 107.0 (2.71$\times$) & 16.0 (-1.7\%) & 7.4 (83.8\%) & 128.1 (3.78$\times$) & 28.5 (-0.4\%) & 12.5 (85.3\%) \\
SparseVLM & 130.4 (2.22$\times$) & 19.0 (-20.1\%) & 7.4 (83.8\%) & 173.6 (2.79$\times$) & 32.0 (-12.7\%) & 12.3 (85.5\%) \\
DivPrune & 86.5 (3.35$\times$) & 14.0 (11.1\%) & 7.3 (84.0\%) & \textbf{112.0} (4.32$\times$) & \textbf{26.0} (8.5\%) & 12.3 (85.5\%) \\
VisionZip & 102.4 (2.83$\times$) & 15.2 (3.9\%) & 7.3 (84.1\%) & 126.5 (3.83$\times$) & 27.0 (4.9\%) & \underline{12.2} (85.7\%) \\
VisPruner & 112.4 (2.58$\times$) & 14.7 (7.1\%) & 7.3 (84.0\%) & 139.5 (3.47$\times$) & \underline{26.5} (6.7\%) & 12.3 (85.5\%) \\
HoloV & 102.6 (2.83$\times$) & 15.2 (3.9\%) & 7.3 (84.1\%) & 125.8 (3.85$\times$) & 27.0 (4.9\%) & \underline{12.2} (85.7\%) \\
MMTok & 151.9 (1.91$\times$) & 16.1 (-2.2\%) & 7.3 (84.0\%) & 181.4 (2.67$\times$) & 28.0 (1.4\%) & 12.3 (85.5\%) \\
ApET & 103.4 (2.80$\times$) & 16.0 (-1.2\%) & 6.8 (85.0\%) & \underline{124.7} (3.88$\times$) & 28.5 (-0.4\%) & \textbf{11.5} (86.5\%) \\
\rowcolor{micoRose}
\textbf{MiCo} & 99.4 (2.92$\times$) & 13.5 (14.1\%) & 7.2 (84.3\%) & 151.1 (3.20$\times$) & 27.0 (4.9\%) & \underline{12.2} (85.7\%) \\
\midrule
\multicolumn{7}{c}{\textit{Budget $T=160$ tokens per image}} \\
\midrule
FastV & -- & -- & -- & -- & -- & -- \\
SparseVLM & -- & -- & -- & -- & -- & -- \\
DivPrune & 81.9 (3.54$\times$) & 14.0 (11.2\%) & 5.1 (88.8\%) & \textbf{92.8} (5.22$\times$) & \textbf{25.9} (8.8\%) & 8.1 (90.5\%) \\
VisionZip & 97.1 (2.99$\times$) & 15.2 (3.9\%) & 5.1 (88.8\%) & 114.1 (4.24$\times$) & 27.0 (4.9\%) & \underline{8.0} (90.6\%) \\
VisPruner & 113.1 (2.56$\times$) & 14.7 (7.1\%) & 5.1 (88.8\%) & 124.6 (3.89$\times$) & 26.5 (6.7\%) & 8.1 (90.5\%) \\
HoloV & 101.1 (2.87$\times$) & 15.2 (3.9\%) & 5.1 (88.8\%) & 114.1 (4.24$\times$) & 27.0 (4.9\%) & \underline{8.0} (90.6\%) \\
MMTok & 131.9 (2.20$\times$) & 16.1 (-2.2\%) & 5.1 (88.8\%) & 144.3 (3.35$\times$) & 28.0 (1.4\%) & 8.1 (90.5\%) \\
ApET & 91.4 (3.17$\times$) & 14.0 (11.1\%) & 5.0 (89.0\%) & \underline{108.4} (4.47$\times$) & \underline{26.1} (8.1\%) & \textbf{7.9} (90.7\%) \\
\rowcolor{micoRose}
\textbf{MiCo} & 88.1 (3.29$\times$) & 13.4 (14.8\%) & 5.0 (89.0\%) & 123.6 (3.92$\times$) & 27.0 (4.9\%) & \underline{8.0} (90.6\%) \\
\end{longtable}
\endgroup

\FloatBarrier
\subsection{Component-level breakdown}

\paragraph{Component timing.}
Tables~\ref{tab:mico_gs_latency_components} and~\ref{tab:mico_video_latency_breakdown} report the NeXT-7B and Video-7B timing breakdowns with the same row structure. The vision encoder, the projector, and the three pruning spans (Stage~1 selection, Stage~2 relevance, and Stage~2 selection) are CUDA event intervals recorded inside one generate call; the remaining span covers the decoder layers, the output head, and runtime/control work, so each column sums to the profiled generate span. The headline latency and peak-memory values are those reported in Tables~\ref{tab:efficiency_next7b} and~\ref{tab:llava_video}; the component spans below are diagnostic. For NeXT-7B, the profile uses one A800 80GB GPU with $K=7$, ten warm-ups, and 50 measured MME requests per budget. Within its remaining span, the decoder layers and final norm take 45.49/48.66/61.29\,ms at $T=160/320/640$, the output head 0.42/0.53/0.82\,ms, and runtime/control 2.94/2.55/3.09\,ms; the instrumented passes take 90.17/102.20/138.91\,ms of wall time against un-instrumented reference passes of 87.95/98.75/134.92\,ms before and 87.06/100.95/135.03\,ms after profiling, so observer overhead is about 1--4\%. For Video-7B, the spans are diagnostic means over two timed requests per budget (after warm-up) on the profiling cohort, a fixed set of videos sampled at 64 frames, and the decoder was not instrumented separately ($\dagger$): its remaining span is the generate span minus the named spans.
The two backbones differ in which stage dominates the pruning cost. On NeXT-7B, Stage~1 selects $2T$ tokens from the 2,880 multi-crop tokens while Stage~2 selects $R$ from those survivors, so Stage~1 runs more greedy iterations and is the larger span (28.3 versus 13.9\,ms at $T=640$). On Video-7B, Stage~1 runs independently per frame over 169 tokens (batched across the 64 frames), whereas Stage~2 operates on the concatenated pool of $64\times2T$ survivors; its pairwise similarities and greedy loop over up to 8,192 tokens therefore dominate the pruning cost (247.6\,ms at $T=64$) and fall by roughly $3\times$ with each halving of $T$, while the per-frame Stage~1 stays at 5--11\,ms.

\begin{table}[!htbp]
\centering
\begin{minipage}[t]{0.48\textwidth}
\centering
\centering
\captionof{table}{MiCo timing on LLaVA-NeXT-7B ($K=7$). Component spans are means over 50 measured requests from paired instrumented profiles; complete-call latency is measured separately.}
\label{tab:mico_gs_latency_components}
\resizebox{\linewidth}{!}{%
\begin{tabular}{lrrr}
\toprule
Component / measurement (ms unless noted) & $T=160$ & $T=320$ & $T=640$ \\
\midrule
Vision encoder & 28.87 & 29.59 & 30.48 \\
Multimodal projector & 0.62 & 0.62 & 0.62 \\
Stage 1 selector & 7.76 & 14.23 & 28.27 \\
Stage 2 relevance (text attention) & 0.39 & 0.09 & 0.10 \\
Stage 2 selector & 2.66 & 5.47 & 13.84 \\
\textit{Pruning subtotal} & \textit{10.81} & \textit{19.79} & \textit{42.21} \\
Decoder, output head and runtime & 48.84 & 51.73 & 65.21 \\
Profiled generate span & 89.14 & 101.73 & 138.52 \\
\midrule
Formal complete-call latency & 88.1 & 99.4 & 135.5 \\
Peak allocated memory (GiB) & 13.4 & 13.5 & 13.8 \\
Actual visual tokens per image (layer average) & 159.84 & 319.69 & 640.16 \\
\bottomrule
\end{tabular}%
}

\end{minipage}\hfill
\begin{minipage}[t]{0.48\textwidth}
\centering
\centering
\captionof{table}{MiCo timing on LLaVA-Video-7B ($K=6$). Component spans are diagnostic means over two timed requests after warm-up; complete-answer latency is measured separately.}
\label{tab:mico_video_latency_breakdown}
\resizebox{\linewidth}{!}{%
\begin{tabular}{lrrr}
\toprule
Component / measurement (ms unless noted) & $T=16$ & $T=32$ & $T=64$ \\
\midrule
Vision encoder & 629.82 & 629.80 & 630.01 \\
Multimodal projector & 8.16 & 8.17 & 8.17 \\
Stage 1 selector & 4.62 & 6.65 & 10.94 \\
Stage 2 relevance (text attention) & 1.17 & 1.98 & 4.02 \\
Stage 2 selector & 22.94 & 74.28 & 247.64 \\
\textit{Pruning subtotal} & \textit{28.73} & \textit{82.91} & \textit{262.60} \\
Decoder, output head and runtime$^\dagger$ & 113.27 & 199.04 & 394.25 \\
Profiled generate span & 779.98 & 919.92 & 1295.03 \\
\midrule
Formal complete-answer latency & 804.7 & 933.8 & 1262.5 \\
Peak allocated memory (GiB) & 21.32 & 21.32 & 21.32 \\
Actual visual tokens per frame & 16.00 & 31.99 & 64.00 \\
\bottomrule
\end{tabular}%
}

\end{minipage}
\end{table}

\FloatBarrier
\paragraph{Component FLOPs.}
Tables~\ref{tab:mico_gs_full_flops} and~\ref{tab:mico_video_flops_breakdown} report the corresponding model and pruning arithmetic at the same budgets as the timing tables ($T=160/320/640$ for NeXT-7B and per-frame $T=16/32/64$ for Video-7B), separating model arithmetic from pruning arithmetic. Runtime/control work is excluded from FLOPs. The NeXT-7B counter further splits the pruning arithmetic into Stage~1 and Stage~2, matching the timing rows; the Video counter reports the Stage-1 and Stage-2 proxy computations as one row and the Triton-implemented greedy selection of both stages as a separate row. The Video breakdown is counted on the fixed 64-frame profiling cohort, so its totals differ by less than 1\% from the native-input totals reported in Table~\ref{tab:llava_video}.

\begin{table}[!htbp]
\centering
\centering
\captionof{table}{MiCo FLOPs breakdown on LLaVA-NeXT-7B ($K=7$); percentages are shares of total FLOPs. Pruning arithmetic is attributed to Stage~1 and Stage~2 separately.}
\label{tab:mico_gs_full_flops}
\resizebox{\linewidth}{!}{%
\begin{tabular}{lrrr}
\toprule
Component (GFLOPs) & $T=160$ & $T=320$ & $T=640$ \\
\midrule
Vision encoder & 1836.289 (36.63\%) & 1836.289 (25.57\%) & 1836.289 (15.78\%) \\
Multimodal projector & 120.879 (2.41\%) & 120.879 (1.68\%) & 120.879 (1.04\%) \\
Decoder layers and final normalization & 3002.246 (59.90\%) & 5136.922 (71.54\%) & 9520.805 (81.81\%) \\
Vocabulary output head & 48.350 (0.96\%) & 78.496 (1.09\%) & 139.052 (1.19\%) \\
\midrule
Stage 1: scores, similarity and selection & 3.729 (0.07\%) & 4.047 (0.06\%) & 4.684 (0.04\%) \\
Stage 2: extra attention / relevance & 0.113 (0.002\%) & 0.206 (0.003\%) & 0.391 (0.003\%) \\
Stage 2: scores, similarity and selection & 0.881 (0.02\%) & 3.651 (0.05\%) & 15.713 (0.14\%) \\
Stage 2 subtotal & 0.994 (0.02\%) & 3.857 (0.05\%) & 16.104 (0.14\%) \\
\midrule
Model subtotal & 5007.764 (99.91\%) & 7172.586 (99.89\%) & 11617.024 (99.82\%) \\
Pruning subtotal & 4.723 (0.09\%) & 7.904 (0.11\%) & 20.788 (0.18\%) \\
Runtime/control (excluded from FLOPs) & \multicolumn{3}{c}{--} \\
\textbf{Total (GFLOPs)} & \textbf{5012.486 (100.00\%)} & \textbf{7180.490 (100.00\%)} & \textbf{11637.813 (100.00\%)} \\
\textbf{Total (TFLOPs)} & \textbf{5.012486} & \textbf{7.180490} & \textbf{11.637813} \\
\bottomrule
\end{tabular}%
}

\end{table}

\FloatBarrier
\begin{table}[!htbp]
\centering
\centering
\captionof{table}{MiCo FLOPs breakdown on LLaVA-Video-7B ($K=6$) on the fixed 64-frame profiling cohort; percentages are shares of total FLOPs. The two pruning paths are counted jointly; the Triton greedy term is listed separately.}
\label{tab:mico_video_flops_breakdown}
\resizebox{\linewidth}{!}{%
\begin{tabular}{lrrr}
\toprule
Component (GFLOPs) & $T=16$ & $T=32$ & $T=64$ \\
\midrule
Vision encoder & 41228.062 (69.53\%) & 41228.062 (54.73\%) & 41228.062 (37.19\%) \\
Multimodal projector & 1585.032 (2.67\%) & 1585.032 (2.10\%) & 1585.032 (1.43\%) \\
Visual pooling / plumbing & 0.155 (0.000\%) & 0.155 (0.000\%) & 0.155 (0.000\%) \\
Decoder layers and final normalization & 15444.005 (26.05\%) & 30518.109 (40.51\%) & 63934.284 (57.68\%) \\
Vocabulary output head & 929.368 (1.57\%) & 1738.100 (2.31\%) & 3357.739 (3.03\%) \\
\midrule
Pruning operators (Stage 1 and Stage 2, jointly counted) & 109.016 (0.18\%) & 261.046 (0.35\%) & 745.496 (0.67\%) \\
Triton greedy selection & 0.088 (0.000\%) & 0.386 (0.001\%) & 1.687 (0.002\%) \\
\midrule
Model subtotal & 59186.622 (99.82\%) & 75069.459 (99.65\%) & 110105.272 (99.33\%) \\
Pruning subtotal & 109.104 (0.18\%) & 261.432 (0.35\%) & 747.183 (0.67\%) \\
Runtime/control (excluded from FLOPs) & \multicolumn{3}{c}{--} \\
\textbf{Total (GFLOPs)} & \textbf{59295.726 (100.00\%)} & \textbf{75330.891 (100.00\%)} & \textbf{110852.455 (100.00\%)} \\
\textbf{Total (TFLOPs)} & \textbf{59.296} & \textbf{75.331} & \textbf{110.852} \\
\bottomrule
\end{tabular}%
}

\end{table}

\FloatBarrier

\section{Additional Ablation Studies}
\label{sec:appendix_13b}

\subsection{Two-Stage Component Ablation}
\label{sec:appendix_components}

Table~\ref{tab:ablation_stages} reports AI2D, POPE, HallB, MME, MMB-EN, MMStar, and SQA-IMG for Figure~\ref{fig:two_stage_components}, sorted by increasing Acc$_7$. Acc$_7$ averages exactly these seven benchmarks, using POPE accuracy and MME divided by 20, and differs from the eight-benchmark Acc in the main performance tables. The main figure retains the English MMBench evaluation; Figure~\ref{fig:two_stage_components_supplement} shows MMB-CN, ChartQA, and OCRBench, none of which contributes to Acc$_7$. All means use unrounded evaluation records. Separate columns specify Stage-1 CLS-attn and Stage-2 text-to-visual attention. The run groups and their full-component references are distinguished in the table note. The random setting samples tokens uniformly without replacement; it differs from the legacy all-off setting, in which every token receives the same score and the selector degenerates to top-$k$ in index order. All settings use $K=2$ and a Stage-1 pool of $N_1=2T$. Y/N mark whether a component is enabled; Stage-1 CLS attention and Stage-2 text-to-visual attention have separate columns, whereas $H$ and coverage are toggled jointly in both stages. The twelve plotted settings retain fixed colors and hatches, with bars sorted by score within each panel. Axis breaks mark truncated ranges; individual MME panels retain the original scale.

Within the paired-removal runs, disabling coverage lowers Acc$_7$ from 77.1 to 74.4 and disabling both attention signals lowers it to 75.7, the two largest drops among the tested switches; the feature-norm term $H$ has the smallest effect in this collection. These paired-removal runs were collected with an earlier Stage-1 attention extraction, so their absolute scores sit slightly below the main-evaluation configuration (Acc$_7=78.4$); the single-component, random, and split-attention runs (the last disable only Stage-1 CLS attention or only Stage-2 text-to-visual attention) use the main-evaluation extraction. Comparisons are therefore most informative within a collection, and the table marks the collection of every row. Table~\ref{tab:ablation_stages} reports the individual switches, scores, and run distinctions. Stage-1 CLS-attn uses the native pooled analogue on Qwen2.5. MiCo does not lead every individual benchmark: no-cls-attn scores higher on MMB-EN.

\begin{table}[!htb]
\centering
\centering
\footnotesize
\setlength{\tabcolsep}{2pt}
\renewcommand{\arraystretch}{1.08}
\caption{Two-stage component scores on Qwen2.5-VL-7B at $T=128$, $K=2$, and a Stage-1 pool of $N_1=2T$.}
\label{tab:ablation_stages}
\micoTableFit{%
\begin{tabular}{lcccccccccccc}
\toprule
Setting & \shortstack{CLS-attn\\(S1)} & \shortstack{Text-visual\\attn (S2)} & $H$ & Coverage & AI2D & POPE & HallB & MME & \shortstack{MMB\\EN} & MMStar & \shortstack{SQA\\IMG} & Acc$_7$ \\
\midrule
all-off & N & N & N & N & 70.5 & 67.8 & 34.6 & 1662.24 & 72.4 & 46.7 & 77.6 & 64.7 \\
only-H & N & N & Y & N & 73.9 & 81.0 & 36.6 & 1933.90 & 79.6 & 51.3 & 82.5 & 71.7 \\
only-cls-attn & Y & N & N & N & 73.8 & 81.4 & 39.1 & 1975.37 & 77.3 & 52.6 & 81.6 & 72.1 \\
random & N & N & N & N & 73.7 & 80.8 & 40.0 & 2047.53 & 77.0 & 52.1 & 82.1 & 72.6 \\
only-both-attn & Y & Y & N & N & 75.0 & 81.9 & 41.2 & 1958.55 & 79.5 & 52.5 & 82.1 & 72.9 \\
no-coverage & Y & Y & Y & N & 77.5 & 82.6 & 40.9 & 2035.51 & 79.9 & 55.0 & 83.0 & 74.4 \\
only-text-visual-attn & N & Y & N & N & 76.7 & 82.3 & 44.3 & 2137.36 & 78.3 & 53.4 & 83.1 & 75.0 \\
no-both-attn & N & N & Y & Y & 78.1 & 85.0 & 43.0 & 2099.88 & 79.8 & 55.0 & 83.8 & 75.7 \\
only-coverage & N & N & N & Y & 77.5 & 84.4 & 43.2 & 2142.60 & 80.2 & 54.7 & 83.8 & 75.8 \\
no-text-visual-attn & Y & N & Y & Y & 78.2 & 85.2 & 42.6 & 2149.09 & 80.2 & 56.3 & 83.8 & 76.3 \\
no-H & Y & Y & N & Y & 79.6 & 84.4 & 45.6 & 2202.81 & 80.4 & 55.3 & 84.0 & 77.1 \\
no-cls-attn & N & Y & Y & Y & 79.9 & 84.6 & 46.7 & 2206.65 & 81.1 & 56.5 & 83.1 & 77.5 \\
\rowcolor{micoPurple}
\textbf{MiCo} & Y & Y & Y & Y & 80.4 & 85.5 & 48.6 & 2252.37 & 80.8 & 56.5 & 84.6 & \textbf{78.4} \\
\bottomrule
\end{tabular}%
}
\par\smallskip
\begin{minipage}{\linewidth}
\scriptsize
\textit{Note.} CLS-attn denotes Qwen2.5's native pooled visual-attention analogue, not a literal CLS token. Text-visual-attn denotes text-to-visual attention in Stage 2; both-attn denotes both signals. MiCo uses the main-evaluation record. The no-H, no-both-attn, no-coverage, and all-off settings come from an earlier paired-removal run with a different Stage-1 attention extraction, whose full-component reference has Acc$_7=77.1$; the only-component, random, and split-attention runs use the main-evaluation extraction. Comparisons are most informative within one run. Scores and means use unrounded evaluation records; table values are rounded for display. In only-cls-attn and only-text-visual-attn, the stage without attention samples uniformly without replacement; random does so in both stages (one run, seed 0). All-off uses equal-score top-$k$ selection and remains appendix-only.
\end{minipage}

\end{table}

\begin{figure}[!htb]
    \centering
    \includegraphics[width=\linewidth]{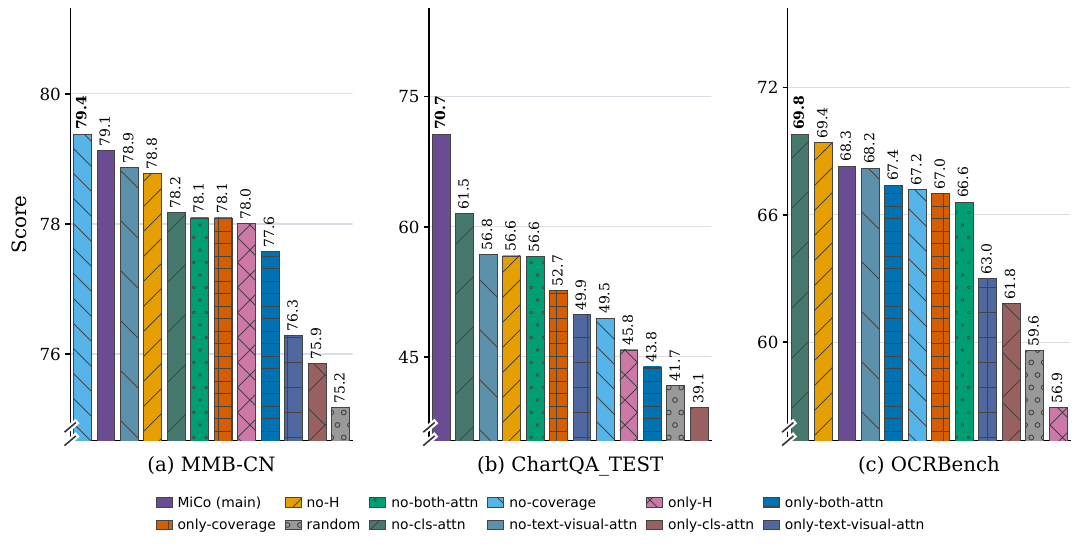}
    \caption{Additional component comparisons on Qwen2.5-VL-7B at $T=128$; these tasks are excluded from Acc$_7$.}
    \label{fig:two_stage_components_supplement}
\end{figure}
% No barrier here: the supplement figure is tall, so let the next subsection's text
% fill the page while the figure floats to the top of the following one.

\subsection{Proxy Substitutions}
\label{sec:appendix_proxy_substitutions}

Table~\ref{tab:ablation_proxy_substitutions} compares four proxy replacements on Qwen2.5-VL-7B at $T=128$. % TODO(authors): confirm the three similarity definitions below match the implementation.
The $H$ replacement scores each token by its normalized hidden-state entropy instead of its feature norm, in both stages. The coverage replacement drops the pairwise coverage term and scores each token by its similarity to the mean visual feature (visual-mean similarity), in both stages. The reliability replacement replaces the Stage-1 CLS-attention proxy with each token's similarity to the global pooled visual token (global-token similarity), and the relevance replacement replaces the Stage-2 text-to-visual attention with each token's mean similarity to the text tokens (mean text--visual similarity). Acc averages AI2D, POPE accuracy, HallB, MME$/20$, MMBench-EN/CN, MMStar, and SQA-IMG; Rel normalizes Acc by the unpruned reference of 82.95. Rows are ordered by increasing Acc, with MiCo last. Each column highlights the best and second-best distinct displayed values; MiCo leads the aggregate, but not every individual benchmark. The complete task-relevance comparison below retains the additional uniform, inverse (one minus) mean text--visual similarity, and final-token alternatives.
\FloatBarrier

\subsection{Stage-2 Task-Relevance Ablation}
\label{sec:appendix_task_relevance}

We replace only the Stage-2 task-relevance weight on Qwen2.5 at $T=128$ (Table~\ref{tab:ablation_text_rater}). All variants use $K=2$ and a Stage-1 pool of $2T$; uniform relevance sets the weight to one. Averaging text-to-visual attention over all text queries gives the highest Acc, 78.5, compared with 76.6 for uniform relevance, 77.1 for mean text--visual similarity, and 77.2 for its complement. Using only the final text token reaches 78.3. These results favor attention-based task relevance over raw cross-modal similarity, with a small further gain from aggregating text queries. All five variants exceed MMTok's Acc of 75.3, the strongest external baseline at this budget.

The aggregate uses the eight-benchmark Acc of the Qwen main tables, with MME divided by 20. POPE reports overall accuracy, MMBench uses Circular evaluation, and SQA uses image-question accuracy.

\begin{table}[!htbp]
    \centering
    % Rel uses unrounded taskrel Acc8 from figures/data/pope_accuracy_ablations.json
% divided by the Qwen2.5 main-table unpruned reference, 82.95.
\centering
\footnotesize
\setlength{\tabcolsep}{2pt}
\renewcommand{\arraystretch}{1.08}
\caption{Stage-2 task-relevance ablation on Qwen2.5-VL-7B at $T=128$, $K=2$, and a Stage-1 pool of $N_1=2T$.}
\label{tab:ablation_text_rater}
\begin{tabular}{p{\dimexpr0.21\linewidth-2\tabcolsep\relax}
  >{\centering\arraybackslash}p{\dimexpr0.05\linewidth-2\tabcolsep\relax}
  *{3}{>{\centering\arraybackslash}p{\dimexpr0.067\linewidth-2\tabcolsep\relax}}
  >{\centering\arraybackslash}p{\dimexpr0.105\linewidth-2\tabcolsep\relax}
  *{2}{>{\centering\arraybackslash}p{\dimexpr0.067\linewidth-2\tabcolsep\relax}}
  >{\centering\arraybackslash}p{\dimexpr0.09\linewidth-2\tabcolsep\relax}
  >{\centering\arraybackslash}p{\dimexpr0.07\linewidth-2\tabcolsep\relax}
  >{\centering\arraybackslash}p{\dimexpr0.065\linewidth-2\tabcolsep\relax}
  >{\centering\arraybackslash}p{\dimexpr0.075\linewidth-2\tabcolsep\relax}}
\toprule
Method / relevance & $T$ & AI2D & POPE & HallB & MME & \shortstack{MMB\\EN} & \shortstack{MMB\\CN} & MMStar & \shortstack{SQA\\IMG} & Acc & Rel \\
\midrule
MMTok (ICLR'26) & 128 & 76.3 & 84.2 & 43.3 & 2120.7 & 79.8 & 77.5 & 53.7 & 81.4 & 75.3 & 90.8\% \\
\midrule
uniform & 128 & 78.2 & 85.2 & 42.6 & 2149.09 & 80.2 & 78.9 & \underline{56.3} & 83.8 & 76.6 & 92.3\% \\
text-vision-sim & 128 & 79.1 & 85.5 & 45.9 & 2142.01 & 79.7 & \underline{79} & \textbf{56.5} & \underline{84.3} & 77.1 & 93.0\% \\
1-(text-vision-sim) & 128 & 79.4 & \textbf{86.2} & 45.9 & 2170.42 & 80.3 & 78.7 & 55.5 & 83.4 & 77.2 & 93.1\% \\
last-token-attn & 128 & \textbf{80.5} & \underline{85.6} & \underline{48.3} & \underline{2235.97} & \underline{80.6} & \underline{79} & 56.1 & \underline{84.3} & \underline{78.3} & \underline{94.4\%} \\
\rowcolor{micoPurple}
\textbf{MiCo} (all-text-attn) & 128 & \underline{80.4} & 85.5 & \textbf{48.6} & \textbf{2252.37} & \textbf{80.8} & \textbf{79.1} & \textbf{56.5} & \textbf{84.6} & \textbf{78.5} & \textbf{94.7\%} \\
\bottomrule
\end{tabular}

\end{table}
\FloatBarrier

\subsection{Stage-1 Budget Ablations}
\label{sec:appendix_s1_budget}

To test the trade-off between a broad candidate pool and later-layer compression, we vary the Stage~1 pool $N_1=nT$ for $n=1,\dots,5$ on Qwen2.5-VL-7B while keeping the layer-average budget fixed (the figures show $n=1$--$4$; the tables also include $n=5$). Increasing the initial pool leaves fewer tokens for the remaining decoder layers under this constraint. Qwen2.5 uses 28 decoder blocks and $K=2$; for each $n$, the Stage-2 count $R$ is recomputed from \eqref{eq:layer_average_budget} with $N_1=nT$, so the layer average stays at $T$ (Appendix~\ref{sec:appendix_budget}).

Figure~\ref{fig:s1_budget_sensitivity} shows four illustrative benchmarks and the full eight-benchmark mean at $T=128$. All four pool sizes ($n=1$--$4$) exceed the strongest external baseline on the displayed metrics. The eight-benchmark mean is higher at $n=2$--$4$ than at $n=1$ and peaks at $n=2$, supporting a moderate initial candidate pool. This improvement is not uniform on each task: MMB-EN at $n=3$ is 0.09 points below $n=1$. The tables below report all benchmark results at both budgets, including $n=5$, alongside the InternVL3 comparison.

\begin{figure}[!htbp]
    \centering
    \includegraphics[width=\linewidth]{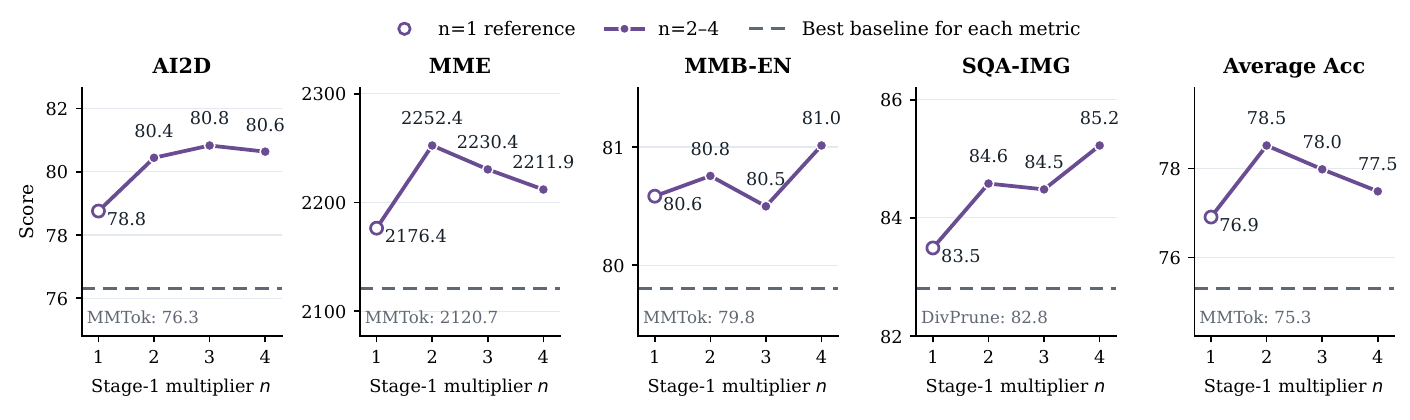}
    \caption{Stage-1 budget sensitivity on Qwen2.5-VL-7B at $T=128$: four illustrative tasks and the eight-task mean.}
    \label{fig:s1_budget_sensitivity}
\end{figure}

Tables~\ref{tab:ablation_hyperparams} and~\ref{tab:ablation_s1_intern} report the full Stage-1 budget ablations for Qwen2.5-VL-7B and InternVL3-8B at $T=256$ and $T=128$, including $n=5$. Figures~\ref{fig:s1_budget_sensitivity} and~\ref{fig:s1_budget_sensitivity_internvl3} display $n=1$--$4$ for Qwen2.5 and InternVL3, respectively. The $n=1$ rows repeat the main-evaluation MiCo-S1 records of Tables~\ref{tab:pruning_comparison_qwen25_full} and~\ref{tab:pruning_comparison_internvl3_full}, which both figures also use as their $n=1$ reference. InternVL3 uses 28 decoder blocks and $K=4$, and its displayed Average Acc peaks at $n=3$. Its plotted $n=2$ point is the main-evaluation MiCo record from Table~\ref{tab:pruning_comparison_internvl3_full} (Acc 79.3); Table~\ref{tab:ablation_s1_intern} instead lists the separately collected budget-ablation run for that setting (Acc 79.2).

Each figure displays four illustrative benchmarks chosen after inspecting the results, with different selections for the two models. These panels do not summarize every task: the complete eight-benchmark mean is shown separately, and the tables retain all benchmark scores. On both models, all four pool sizes exceed the strongest external baseline in this mean, and $n=2$--$4$ score higher than $n=1$. Individual tasks do not always follow this pattern; for example, Qwen2.5 MMB-EN at $n=3$ is 0.09 points below $n=1$. Dashed lines show the strongest external pruner per metric; hollow $n=1$ points identify the main-evaluation MiCo-S1 reference. Panel ranges differ, and MME retains its original scale outside the mean. Acc uses unrounded scores with MME$/20$. InternVL distributes the token budget evenly across its dynamic tiles, assigning any remainder to the first tiles.

\begin{figure}[!htbp]
    \centering
    \includegraphics[width=\linewidth]{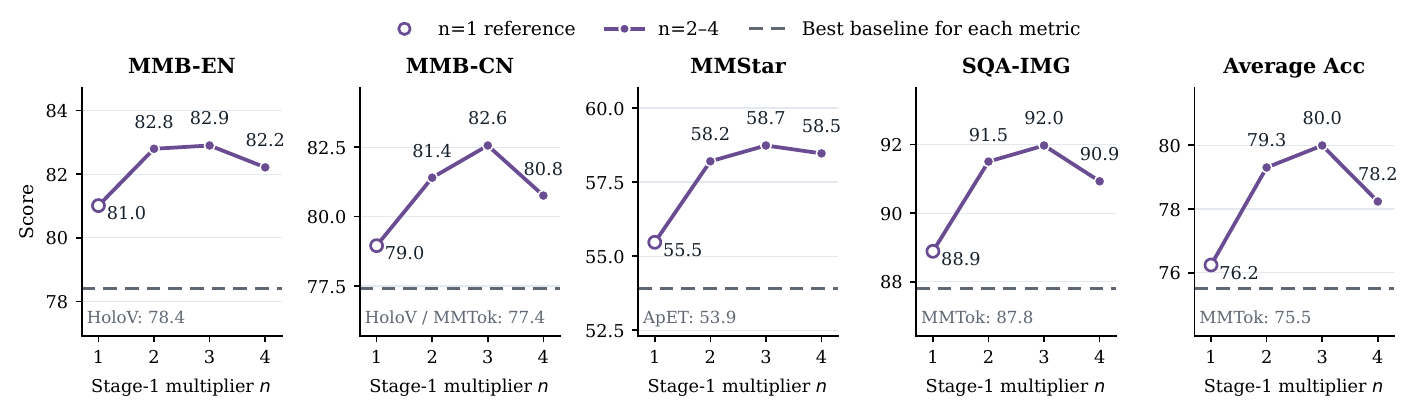}
    \caption{Stage-1 budget sensitivity on InternVL3-8B at $T=128$: four illustrative tasks and the eight-task mean.}
    \label{fig:s1_budget_sensitivity_internvl3}
\end{figure}

\begin{table}[H]
\centering
\footnotesize
\setlength{\tabcolsep}{2pt}
\renewcommand{\arraystretch}{1.08}
\caption{Stage-1 budget ablation on Qwen2.5-VL-7B with refinement at $K=2$.}
\label{tab:ablation_hyperparams}
\micoTableFit{%
\begin{tabular}{lcccccccccccc}
\toprule
S1 setting & $T$ & \shortstack{S1\\tokens} & S2 $R$ & AI2D & POPE & HallB & MME & \shortstack{MMB\\EN} & \shortstack{MMB\\CN} & MMStar & \shortstack{SQA\\IMG} & Acc \\
\midrule
n=1 (T) & 256 & 256 & 256 & 81.7 & 86.1 & 48.6 & 2282.2 & 82.0 & 80.5 & 59.5 & 85.5 & 79.8 \\
n=2 (2T) & 256 & 512 & 236 & 83.8 & 86.3 & 51.2 & 2291.0 & 83.2 & 81.3 & 60.1 & 86.2 & 80.8 \\
n=3 (3T) & 256 & 768 & 217 & 82.4 & 86.4 & 50.6 & 2289.3 & 82.3 & 81.4 & 60.7 & 86.4 & 80.6 \\
n=4 (4T) & 256 & 1024 & 197 & 82.4 & 86.1 & 50.0 & 2266.5 & 82.4 & 81.3 & 59.6 & 86.3 & 80.2 \\
n=5 (5T) & 256 & 1280 & 177 & 82.0 & 85.6 & 49.9 & 2250.1 & 81.9 & 80.6 & 59.1 & 85.8 & 79.7 \\
n=1 (T) & 128 & 128 & 128 & 78.8 & 84.3 & 46.3 & 2176.4 & 80.6 & 78.4 & 54.6 & 83.5 & 76.9 \\
n=2 (2T) & 128 & 256 & 118 & 80.4 & 85.5 & 48.6 & 2252.4 & 80.8 & 79.1 & 56.5 & 84.6 & 78.5 \\
n=3 (3T) & 128 & 384 & 108 & 80.8 & 85.0 & 46.5 & 2230.4 & 80.5 & 78.7 & 56.3 & 84.5 & 78.0 \\
n=4 (4T) & 128 & 512 & 98 & 80.6 & 84.2 & 45.3 & 2211.9 & 81.0 & 78.4 & 54.5 & 85.2 & 77.5 \\
n=5 (5T) & 128 & 640 & 89 & 79.9 & 84.0 & 45.4 & 2203.2 & 80.2 & 77.8 & 54.9 & 85.1 & 77.2 \\
\bottomrule
\end{tabular}%
}
\end{table}

\begin{table}[H]
\centering
\footnotesize
\setlength{\tabcolsep}{2pt}
\renewcommand{\arraystretch}{1.08}
\caption{Stage-1 budget ablation on InternVL3-8B with refinement at $K=4$.}
\label{tab:ablation_s1_intern}
\micoTableFit{%
\begin{tabular}{lcccccccccccc}
\toprule
S1 setting & $T$ & \shortstack{S1\\tokens} & S2 $R$ & AI2D & POPE & HallB & MME & \shortstack{MMB\\EN} & \shortstack{MMB\\CN} & MMStar & \shortstack{SQA\\IMG} & Acc \\
\midrule
n=1 (T) & 256 & 256 & 256 & 80.1 & 90.8 & 44.1 & 2279.1 & 83.7 & 83.2 & 59.7 & 93.4 & 81.1 \\
n=2 (2T) & 256 & 512 & 213 & 83.0 & 90.2 & 45.1 & 2293.7 & 85.0 & 84.5 & 63.0 & 95.6 & 82.6 \\
n=3 (3T) & 256 & 768 & 171 & 82.9 & 90.0 & 45.3 & 2302.5 & 85.3 & 84.7 & 63.5 & 95.6 & 82.8 \\
n=4 (4T) & 256 & 1024 & 128 & 82.1 & 89.9 & 43.2 & 2281.9 & 84.8 & 83.7 & 61.7 & 94.6 & 81.8 \\
n=5 (5T) & 256 & 1280 & 85 & 79.7 & 89.0 & 40.7 & 2260.5 & 82.6 & 83.0 & 59.5 & 92.6 & 80.0 \\
n=1 (T) & 128 & 128 & 128 & 73.9 & 89.3 & 36.8 & 2112.5 & 81.0 & 79.0 & 55.5 & 88.9 & 76.2 \\
n=2 (2T) & 128 & 256 & 107 & 77.5 & 89.7 & 40.4 & 2252.3 & 82.6 & 81.4 & 57.8 & 91.7 & 79.2 \\
n=3 (3T) & 128 & 384 & 85 & 79.4 & 89.6 & 40.6 & 2284.1 & 82.9 & 82.6 & 58.7 & 92.0 & 80.0 \\
n=4 (4T) & 128 & 512 & 64 & 78.2 & 88.6 & 38.7 & 2159.8 & 82.2 & 80.8 & 58.5 & 90.9 & 78.2 \\
n=5 (5T) & 128 & 640 & 43 & 75.6 & 87.8 & 37.5 & 2149.2 & 80.0 & 78.4 & 54.5 & 88.3 & 76.2 \\
\bottomrule
\end{tabular}%
}
\end{table}

\FloatBarrier

\section{Refinement-Layer Selection and Sensitivity}
\label{sec:appendix_layer_selection}

\subsection{Layer Configuration Summary}
\label{sec:appendix_layer_config}

We use one fixed, architecture-specific refinement layer in each reported configuration. Table~\ref{tab:layer_configurations} summarizes the selected layers, and Appendix~\ref{sec:appendix_image_calibration} gives the selection procedure and representative examples.

\begin{table}[!htbp]
    \centering
    \footnotesize
    \setlength{\tabcolsep}{3pt}
    \renewcommand{\arraystretch}{1.12}
    \caption{Refinement-layer configurations used in the reported experiments.}
    \label{tab:layer_configurations}
    \begin{tabular}{p{0.30\linewidth}p{0.60\linewidth}}
    \toprule
    \textbf{Model} & \textbf{Refinement-layer configuration} \\
    \midrule
    LLaVA-1.5-7B & $K=7$ \\
    LLaVA-1.5-13B & $K=8$ \\
    LLaVA-NeXT-7B & $K=7$ \\
    LLaVA-NeXT-13B & $K=14$ \\
    Qwen3-VL-8B & $K=6$ \\
    InternVL3-8B & $K=4$ \\
    Qwen2.5-VL-7B & $K=2$ \\
    Qwen3.5-9B & $K=12$ \\
    LLaVA-Video-7B & $K=6$ \\
    \bottomrule
    \end{tabular}
\end{table}

Detailed selection examples are given in Appendix~\ref{sec:appendix_image_calibration}, while the complete LLaVA sweep tables are collected with the full image-model results.

\subsection{Image-Model Calibration}
\label{sec:appendix_image_calibration}

For the representative OCR-based checks on Qwen3-VL and InternVL3, we hold the same 100 OCRBench samples and target budget $T=128$ fixed, sweep the one-based decoder-layer cutoff $K$ over candidate blocks, run the same two-stage pruning pipeline for each candidate, and select the layer with the highest OCR accuracy. The LLaVA calibration plots use the corresponding fixed sample and budget settings stated in the captions. Answer-token KL divergence is reported as a supplementary diagnostic. Figure~\ref{fig:ocrbench_layer_selection} shows the size-grouped LLaVA calibration, while Figures~\ref{fig:ocrbench_layer_selection_qwen} and~\ref{fig:ocrbench_layer_selection_intern} give representative Qwen3-VL-8B and InternVL3-8B examples at $T=128$. Accuracy curves count correct answers; stars mark the selected maxima, and hollow entries mark the MiCo-S1 (Stage-1-only) references. Layer indices are one-based, with pruning after $K$.

\begin{figure}[!htbp]
    \centering
    \includegraphics[width=\linewidth]{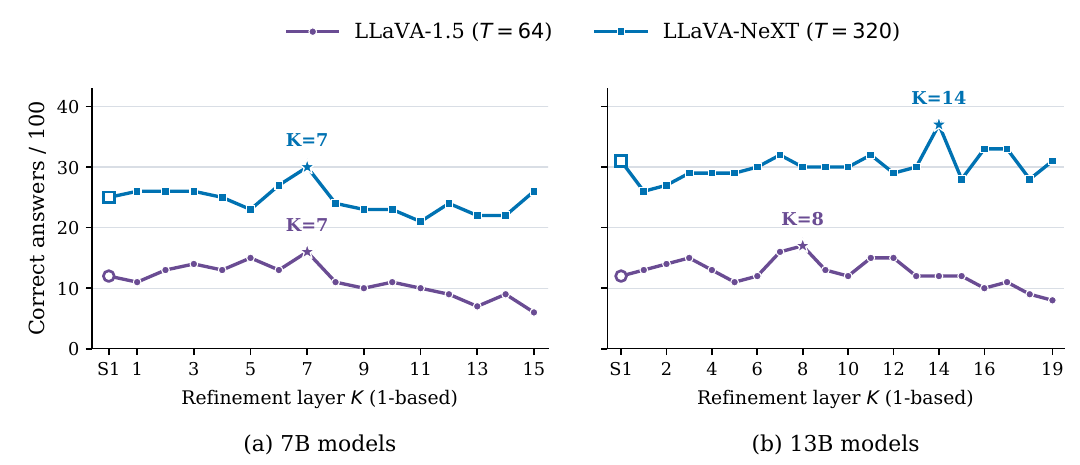}
    \caption{Refinement-layer selection on 100 OCRBench samples: LLaVA-1.5 at $T=64$ and LLaVA-NeXT at $T=320$.}
    \label{fig:ocrbench_layer_selection}
\end{figure}

\begin{figure}[!htbp]
    \centering
    \begin{subfigure}[t]{0.49\linewidth}
        \centering
        \includegraphics[width=\linewidth]{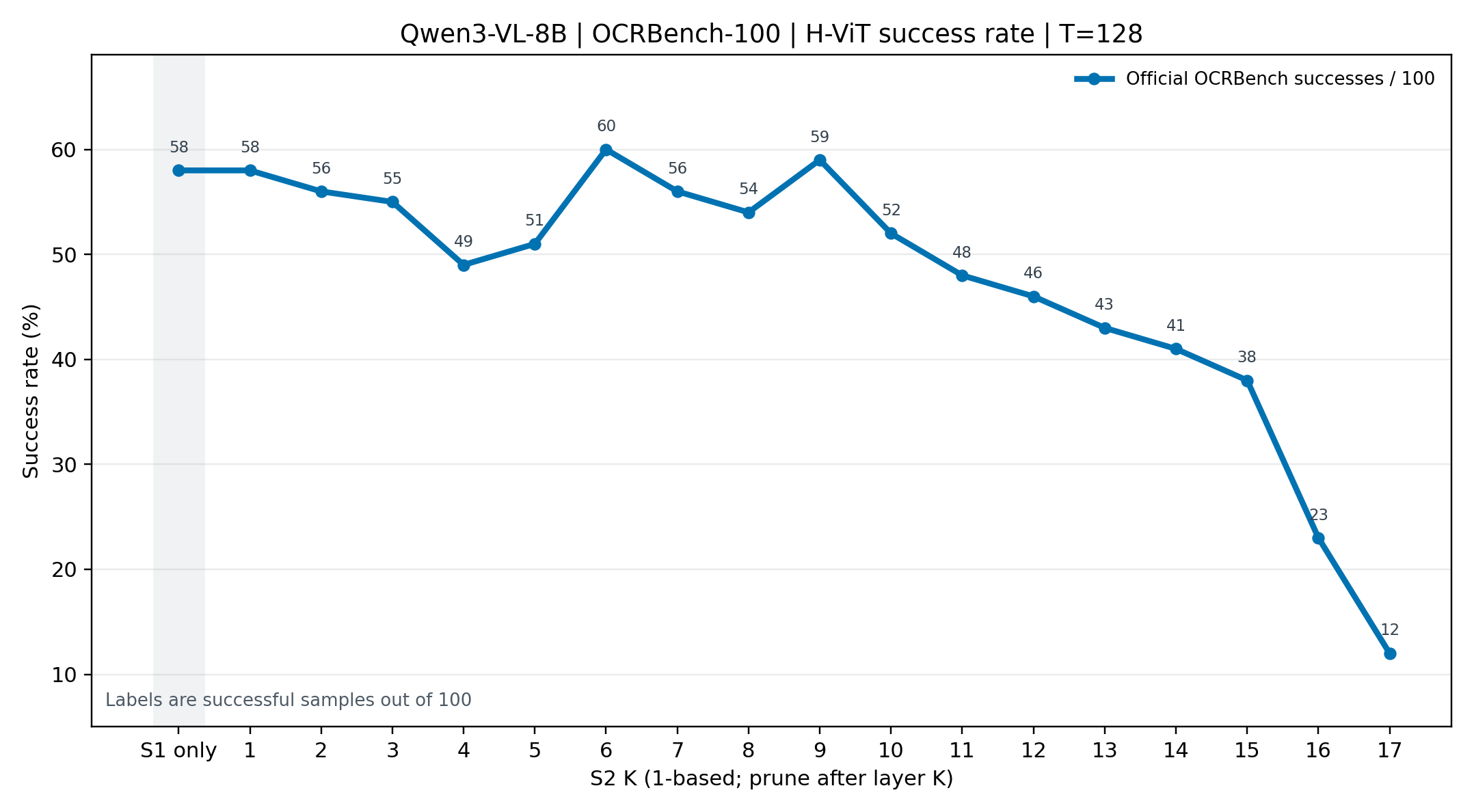}
        \caption{Qwen3-VL-8B}
        \label{fig:ocrbench_layer_selection_qwen}
    \end{subfigure}\hfill
    \begin{subfigure}[t]{0.49\linewidth}
        \centering
        \includegraphics[width=\linewidth]{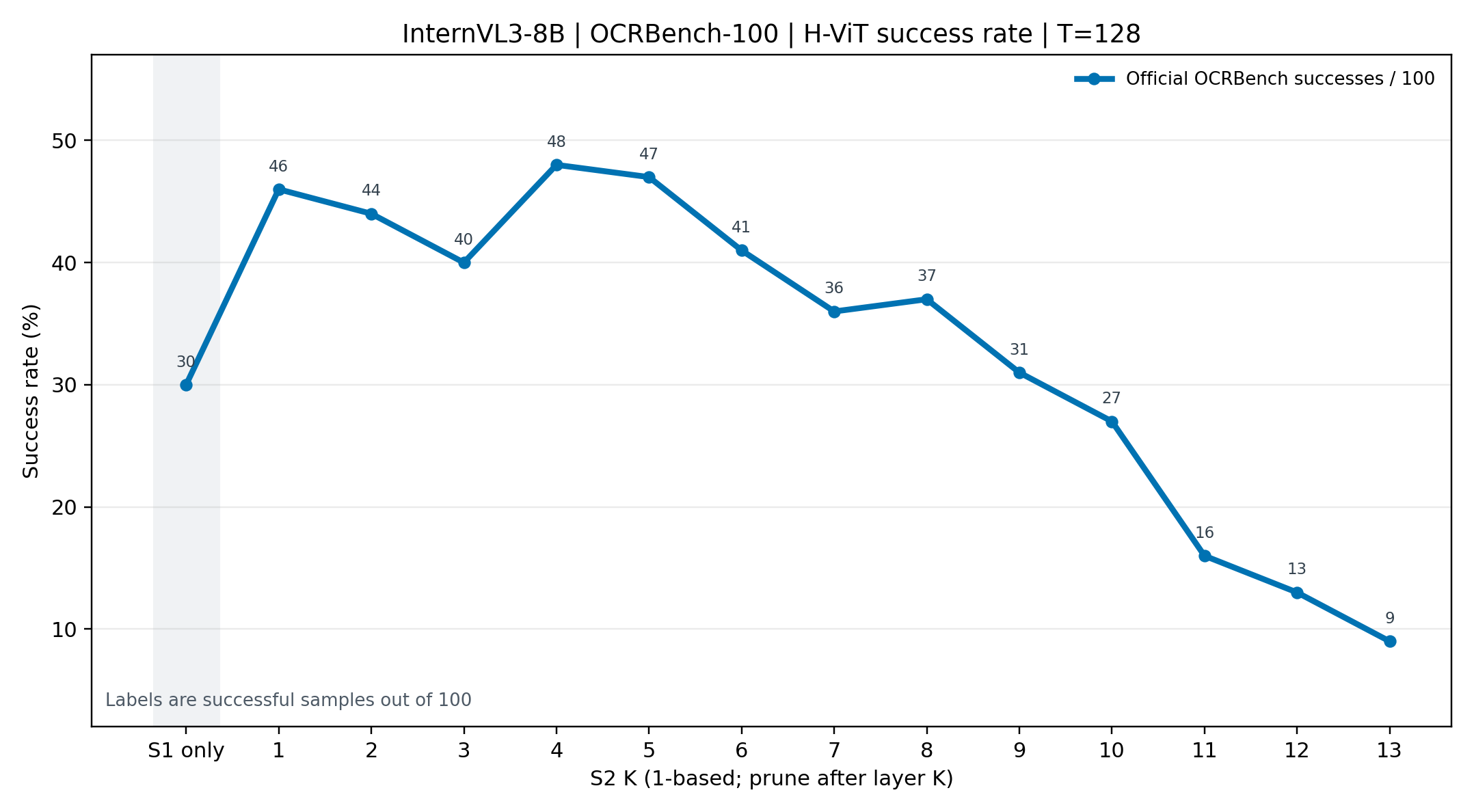}
        \caption{InternVL3-8B}
        \label{fig:ocrbench_layer_selection_intern}
    \end{subfigure}
    \caption{Representative refinement-layer selection for Qwen3-VL-8B and InternVL3-8B on 100 OCRBench samples at $T=128$.}
    \label{fig:ocrbench_layer_selection_qwen_intern}
\end{figure}

\FloatBarrier
\subsection{LLaVA Refinement-Layer Sweeps}
\label{sec:appendix_llava_layers}

Figures~\ref{fig:llava_layer_sensitivity}, \ref{fig:llava_layer_sensitivity_64_320}, and~\ref{fig:llava_layer_sensitivity_128_640} report the refinement-layer sweeps at $T=32/160$, $64/320$, and $128/640$, respectively, for LLaVA-1.5/LLaVA-NeXT. Each figure groups models by parameter scale, with one curve for LLaVA-1.5 and one for LLaVA-NeXT. The selected layers remain $K=7/8$ for LLaVA-1.5-7B/13B and $K=7/14$ for LLaVA-NeXT-7B/13B; they are not reselected from the downstream scores. The left/right panels show 7B/13B models and sweep $K=1$--$15$/$1$--$19$, respectively. Diamonds mark calibration-selected layers; stars mark sweep maxima, including ties. Curves use the ten-benchmark mean Acc$_{10}$ defined in Appendix~\ref{sec:appendix_full_results}, and refinement occurs after the one-based layer $K$.

At the most aggressive budgets ($T=32$ for LLaVA-1.5 and $T=160$ for LLaVA-NeXT), the calibration-selected layers on LLaVA-NeXT-7B/13B are within 0.23/0.05 Acc points of the observed sweep maxima; the corresponding gaps on LLaVA-1.5-7B/13B are 1.19/0.45 points. The calibration-selected layer therefore need not maximize the downstream ten-benchmark mean, although it is close to the observed peak on both NeXT models.

% The three sweep figures are rendered by visual-composer/llava-layer-sweeps/plot_layer_sweeps.py
% (7.2 x 2.0 in, shown at full text width with one-line captions): flat enough that all three sit
% close together under the section text.
{\centering
\includegraphics[width=\linewidth]{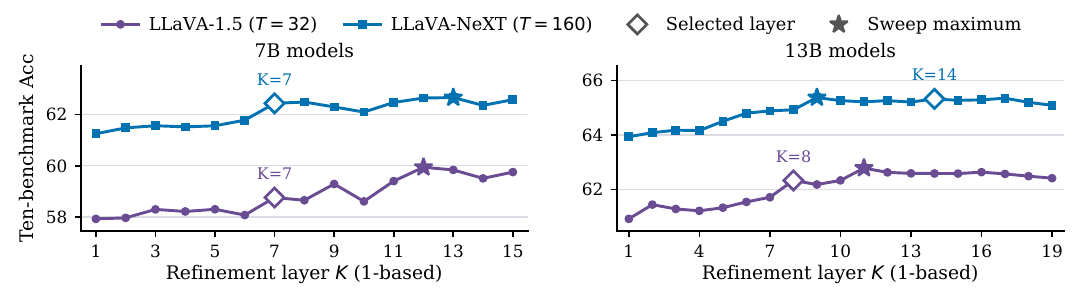}\par
\captionof{figure}{Refinement-layer sweeps at 94.4\% pruning: LLaVA-1.5 ($T=32$) and LLaVA-NeXT ($T=160$).}
\label{fig:llava_layer_sensitivity}\par}

{\centering
\includegraphics[width=\linewidth]{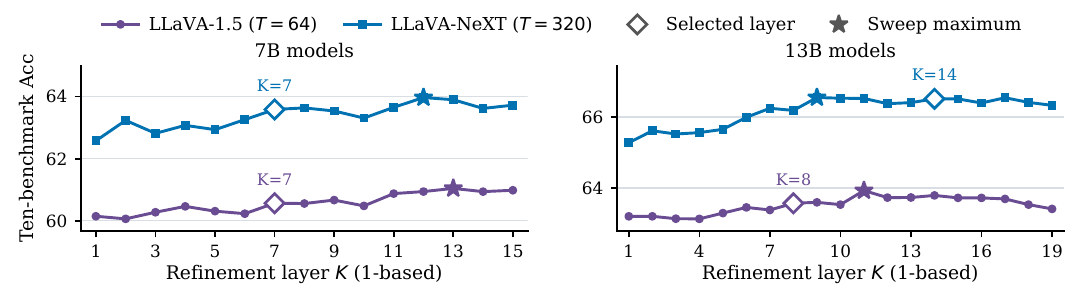}\par
\captionof{figure}{Refinement-layer sweeps at 88.9\% pruning: LLaVA-1.5 ($T=64$) and LLaVA-NeXT ($T=320$).}
\label{fig:llava_layer_sensitivity_64_320}\par}

{\centering
\includegraphics[width=\linewidth]{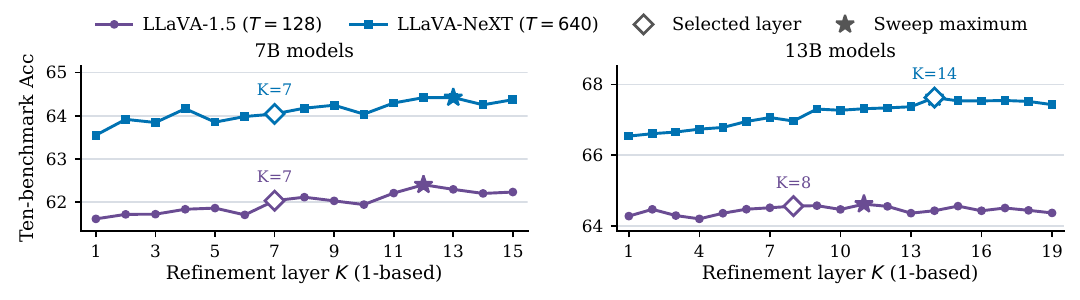}\par
\captionof{figure}{Refinement-layer sweeps at 77.8\% pruning: LLaVA-1.5 ($T=128$) and LLaVA-NeXT ($T=640$).}
\label{fig:llava_layer_sensitivity_128_640}\par}
\FloatBarrier

\subsection{Answer-Token KL Diagnostics}

Figure~\ref{fig:ocrbench_layer_selection_kl} reports answer-token KL divergence as a complementary diagnostic, showing the change in KL from the unpruned baseline relative to the Stage-1-only reference. For LLaVA, the refinement layers were selected from the OCRBench accuracy curves at $T=64$ (LLaVA-1.5) and $T=320$ (LLaVA-NeXT) in Figure~\ref{fig:ocrbench_layer_selection}; the KL diagnostic was run at the most aggressive budgets, $T=32$ and $T=160$, to check the selected layers there. The representative Qwen3 and InternVL3 panels use $T=128$ in both the accuracy and KL figures. Each $\Delta$KL subtracts Stage-1-only KL from the refined model's answer-token KL to the unpruned baseline; lower is better. Shading denotes paired-bootstrap 95\% confidence intervals.

\begin{figure}[!htbp]
    \centering
    \begin{subfigure}{0.24\linewidth}
        \centering
        \includegraphics[width=\linewidth]{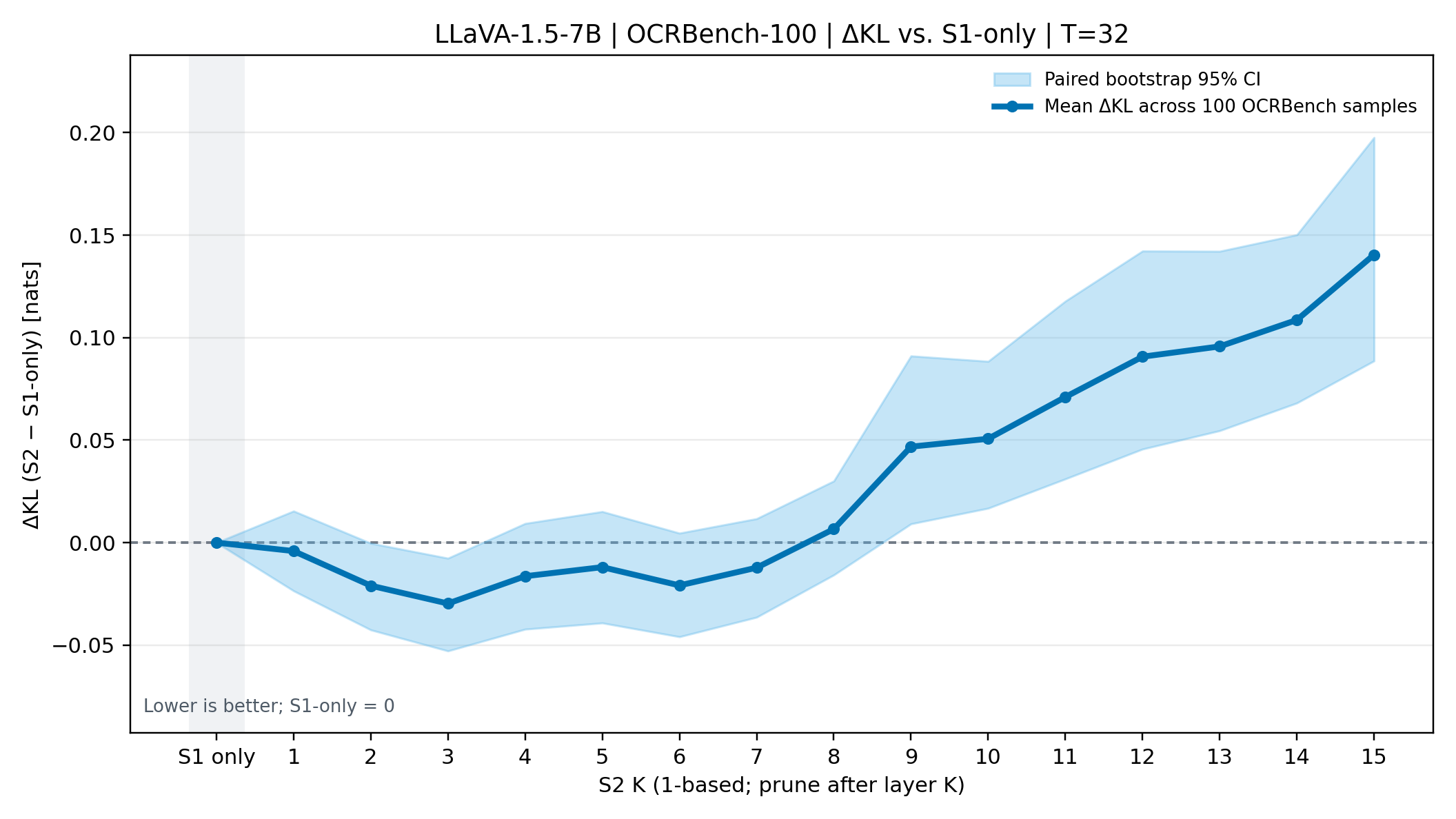}
        \caption{LLaVA-1.5-7B}
    \end{subfigure}\hfill
    \begin{subfigure}{0.24\linewidth}
        \centering
        \includegraphics[width=\linewidth]{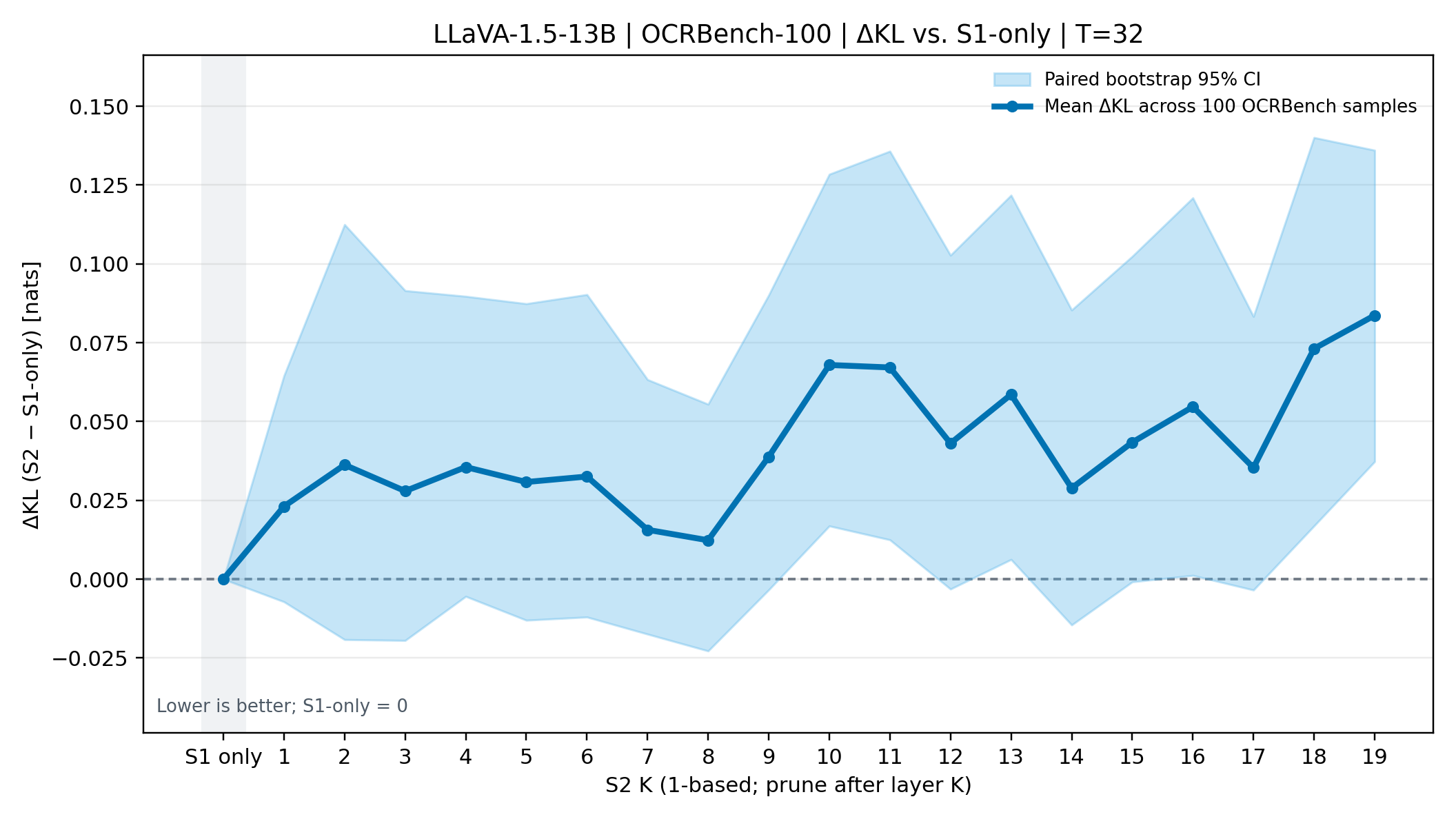}
        \caption{LLaVA-1.5-13B}
    \end{subfigure}\hfill
    \begin{subfigure}{0.24\linewidth}
        \centering
        \includegraphics[width=\linewidth]{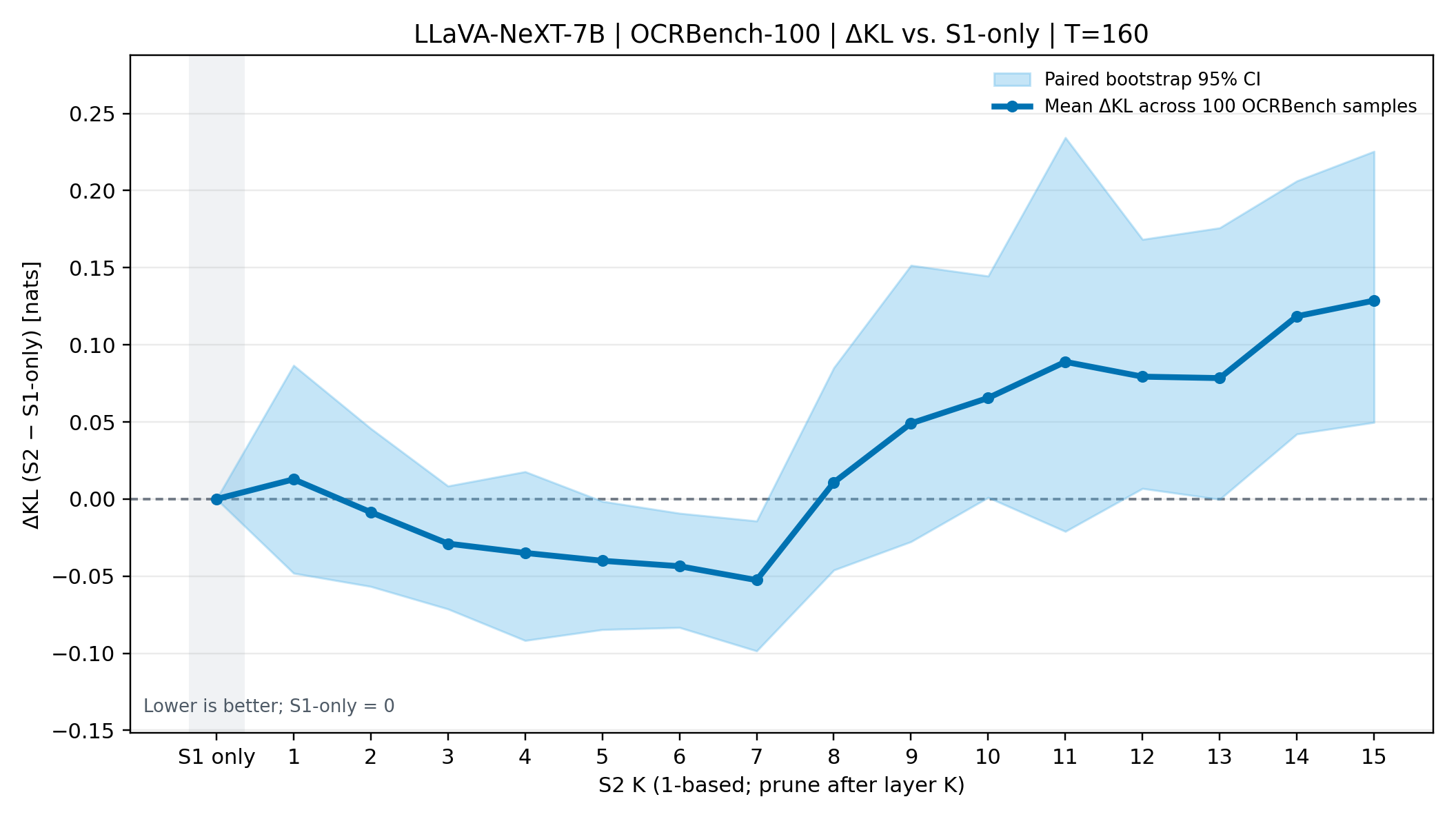}
        \caption{LLaVA-NeXT-7B}
    \end{subfigure}\hfill
    \begin{subfigure}{0.24\linewidth}
        \centering
        \includegraphics[width=\linewidth]{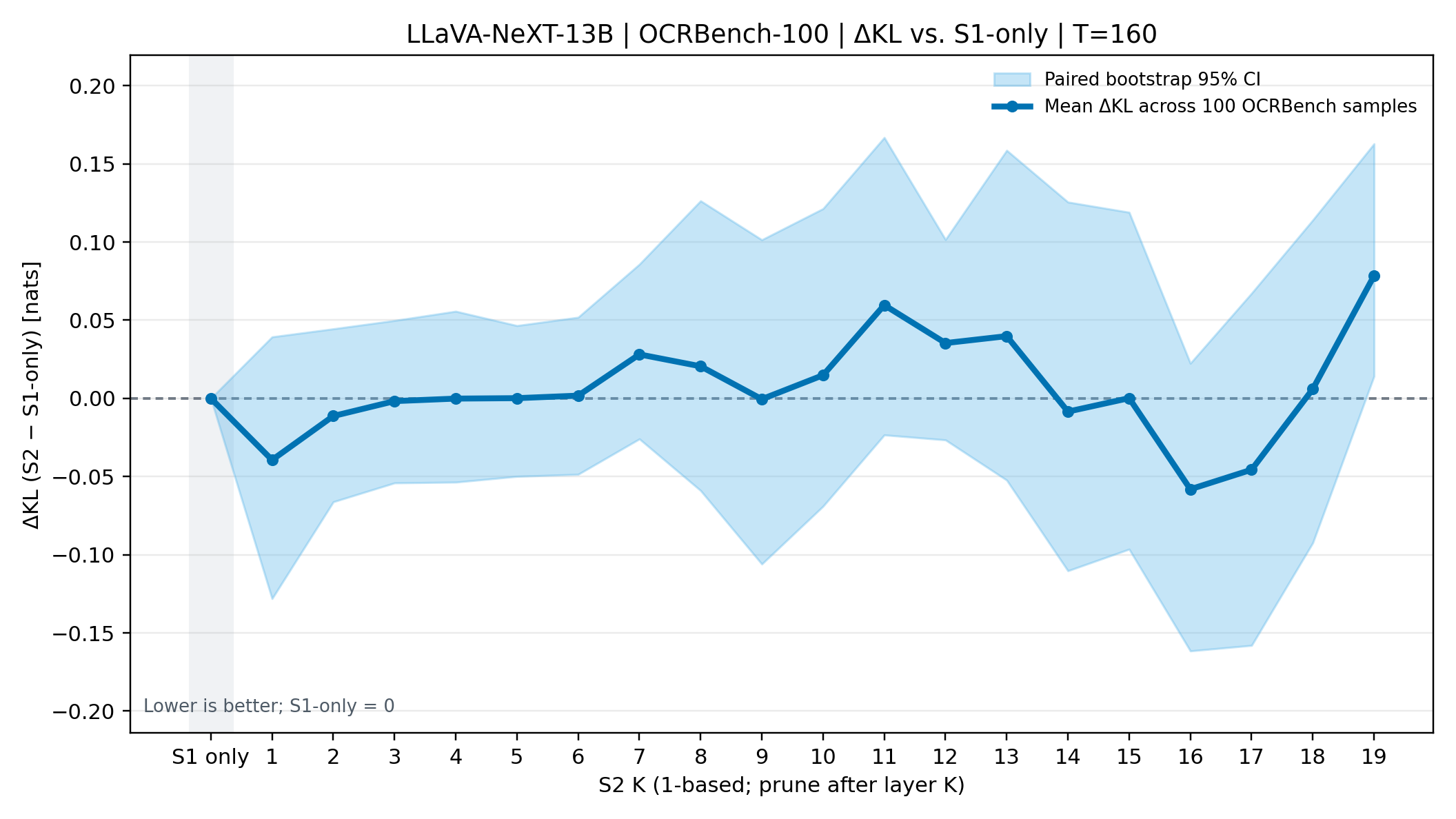}
        \caption{LLaVA-NeXT-13B}
    \end{subfigure}

    \vspace{4pt}
    \begin{subfigure}{0.48\linewidth}
        \centering
        \includegraphics[width=\linewidth]{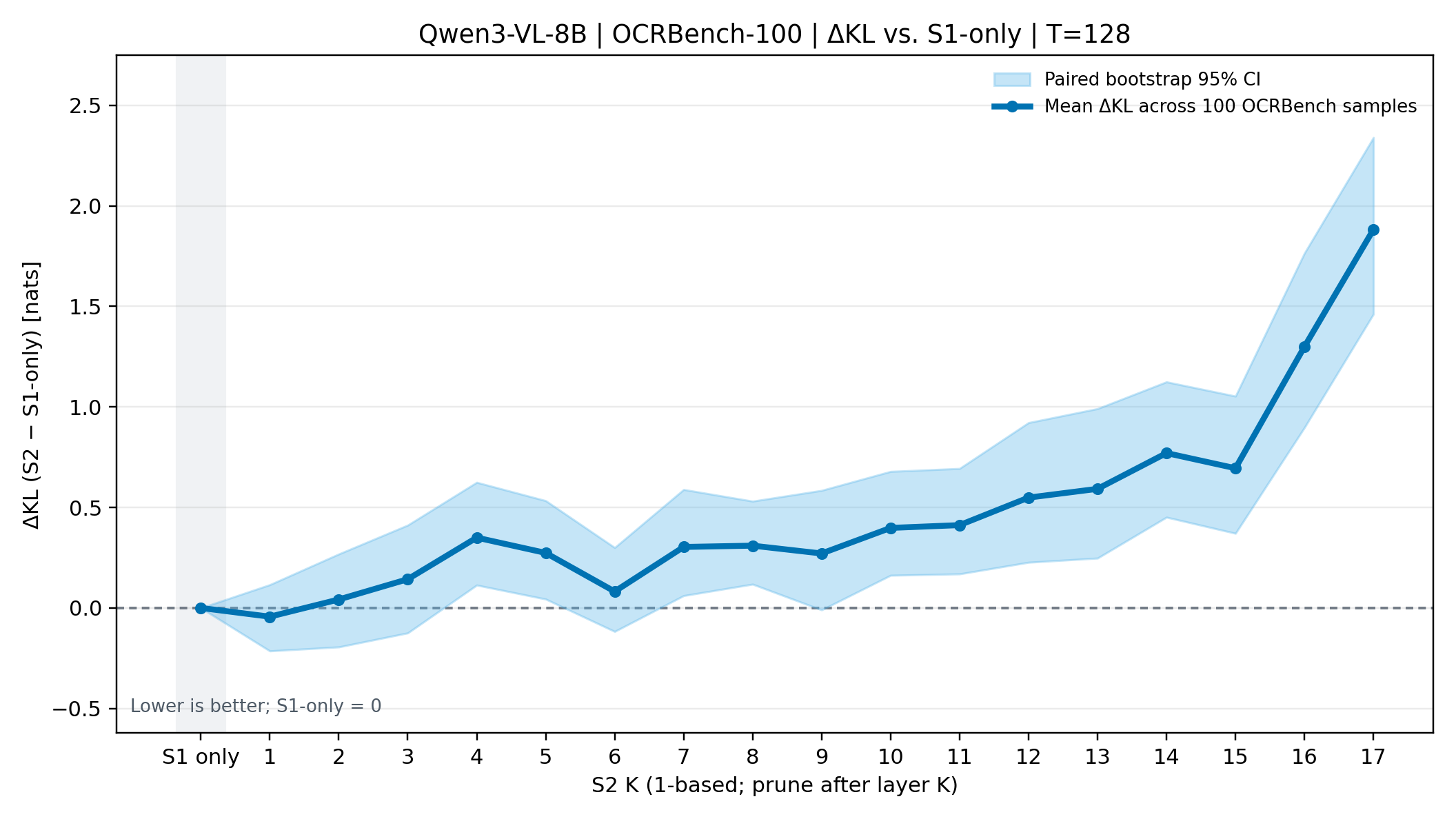}
        \caption{Qwen3-VL-8B}
    \end{subfigure}\hfill
    \begin{subfigure}{0.48\linewidth}
        \centering
        \includegraphics[width=\linewidth]{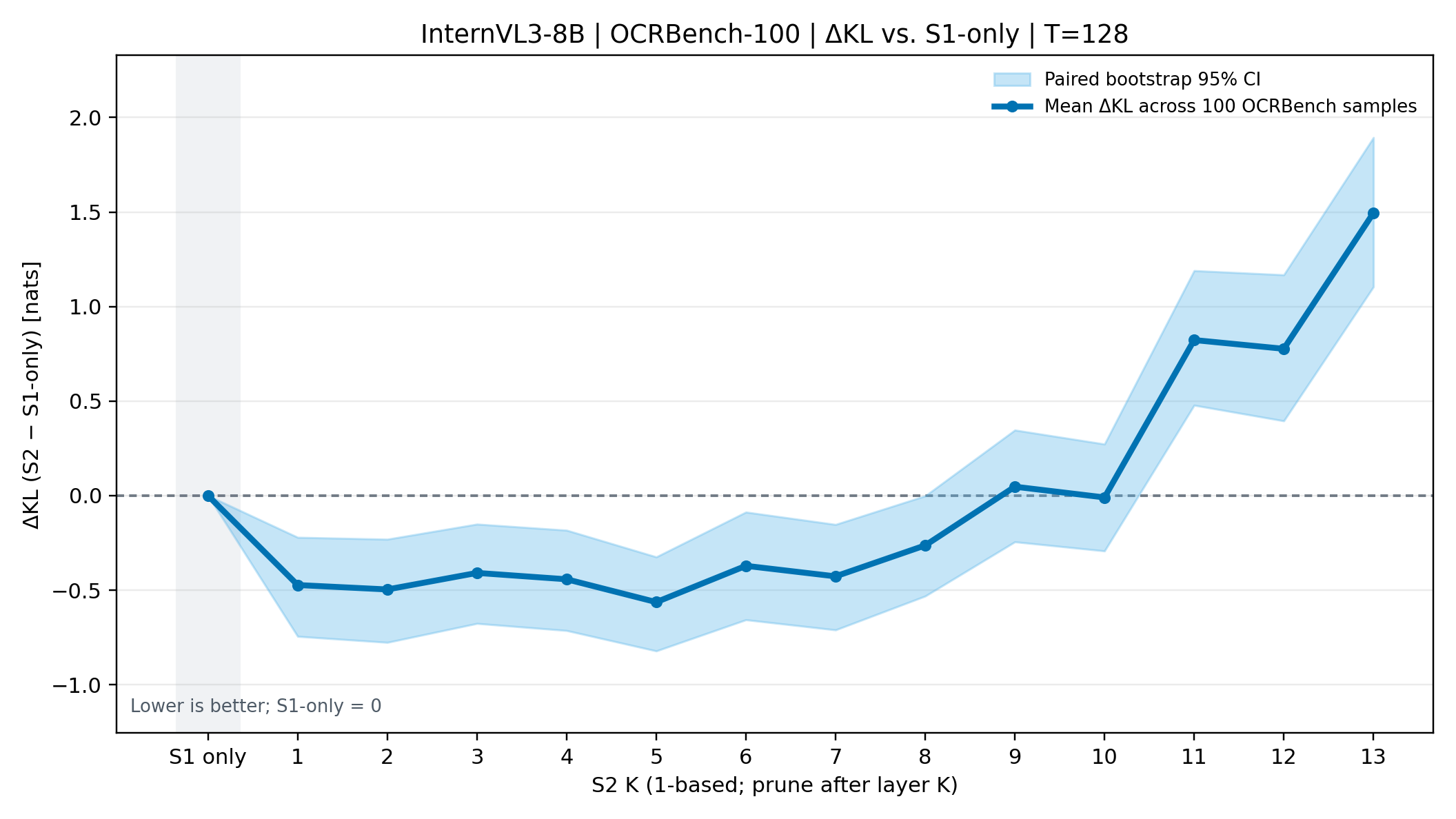}
        \caption{InternVL3-8B}
    \end{subfigure}
    \caption{Layer-wise $\Delta$KL on 100 OCRBench samples (lower is better); shading shows paired-bootstrap 95\% intervals.}
    \label{fig:ocrbench_layer_selection_kl}
\end{figure}
% No barrier here: the KL figure is tall, so let the short video subsection fill the page;
% the barrier in appendix.tex still flushes the figure before Appendix G.

\subsection{Video Refinement Layer}
\label{sec:appendix_video_layer}

For LLaVA-Video, we select the refinement layer once for the video setting and fix it at $K=6$ for all reported video budgets; it is not re-tuned per budget or per benchmark. The video results therefore use the same fixed configuration throughout the reported efficiency and performance comparisons.

\FloatBarrier

\section{Complete Image-Model Results}
\label{sec:appendix_full_results}

The main text reports detailed results for LLaVA-NeXT-13B and summarizes the remaining image models by family.  For lookup, this appendix keeps the complete benchmark scores in two cross-page tables: one for the LLaVA series and one for Qwen-VL and InternVL.  Each model is a labeled block within its family table, with the original budgets, methods, MiCo-S1 rows, and highlighting preserved.

\subsection{LLaVA Series}

The first longtable consolidates LLaVA-1.5-7B/13B and LLaVA-NeXT-7B/13B.  All four models share the same ten benchmark columns, followed by the aggregate Acc and relative retention (Rel).  FastV and SparseVLM keep every visual token through their first two decoder layers, so at the 94.4\% budget ($T=32$ on LLaVA-1.5, $T=160$ on LLaVA-NeXT) the layer-average accounting of Appendix~\ref{sec:appendix_budget} leaves them no feasible schedule on the 32-layer 7B models and only a degenerate one on the 13B models; they are therefore not reported at that budget, and Tables~\ref{tab:efficiency_next7b} and~\ref{tab:ablation_ocrbench} leave the same entries blank.

% Consolidated cross-page table for the complete LLaVA-series results.
\begingroup
\scriptsize
\setlength{\tabcolsep}{0pt}
\setlength{\LTleft}{0pt}
\setlength{\LTright}{0pt}
\setlength{\LTcapwidth}{\linewidth}
\renewcommand{\arraystretch}{1.02}
\begin{longtable}{@{}
>{\raggedright\arraybackslash}p{\dimexpr0.18\linewidth-2\tabcolsep\relax}%
*{10}{>{\centering\arraybackslash}p{\dimexpr0.067\linewidth-2\tabcolsep\relax}}%
>{\centering\arraybackslash}p{\dimexpr0.075\linewidth-2\tabcolsep\relax}%
>{\centering\arraybackslash}p{\dimexpr0.075\linewidth-2\tabcolsep\relax}@{}}
\caption{Complete benchmark results across the LLaVA series, including MiCo-S1.}\label{tab:appendix_llava_full}\\
\toprule
Method & GQA & SQA-IMG & TextVQA & POPE & MME & MMB-EN & MMB-CN & SEED & AI2D & MMMU & Acc & Rel \\
\midrule
\endfirsthead
\multicolumn{13}{l}{\textit{Table \thetable\ (continued): Complete benchmark results across the LLaVA series, including MiCo-S1}}\\
\toprule
Method & GQA & SQA-IMG & TextVQA & POPE & MME & MMB-EN & MMB-CN & SEED & AI2D & MMMU & Acc & Rel \\
\midrule
\endhead
\midrule
\multicolumn{13}{r}{\textit{Continued on next page}}\\
\endfoot
\bottomrule
\endlastfoot
\rowcolor{black!4}
\multicolumn{13}{c}{\textbf{LLaVA-1.5-7B} \label{tab:pruning_comparison_7b_full}}\\
\midrule
\multicolumn{13}{c}{Upper Bound --- 576 tokens (no pruning)} \\
Vanilla & 61.9 & 69.5 & 58.2 & 85.9 & 1508.8 & 64.7 & 58.1 & 66.0 & 55.5 & 35.0 & 63.0 & 100.0\% \\
\midrule
\multicolumn{13}{c}{Retain 128 tokens ($\downarrow$77.8\%)} \\
\midrule
FastV (ECCV'24) & 54.0 & \underline{69.2} & 56.4 & 68.2 & 1376.5 & \textbf{63.0} & 55.9 & 59.7 & 53.6 & 34.3 & 58.3 & 92.6\% \\
SparseVLM (ICML'25) & 57.3 & 69.0 & 56.3 & 83.1 & 1401.3 & \underline{62.6} & 56.9 & 61.7 & 54.7 & 35.9 & 60.8 & 96.4\% \\
DivPrune (CVPR'25) & \underline{59.4} & 68.6 & 55.9 & \textbf{87.0} & 1411.2 & 61.5 & 54.8 & 62.4 & 54.2 & 35.2 & 61.0 & 96.8\% \\
VisionZip (CVPR'25) & 57.6 & 68.7 & 56.9 & 83.3 & 1437.3 & 62.1 & \underline{57.0} & 61.6 & 54.5 & 35.9 & 60.9 & 96.7\% \\
\end{longtable}
\endgroup

\FloatBarrier
\subsection{Qwen-VL and InternVL Series}

The second longtable consolidates Qwen2.5-VL-7B, Qwen3-VL-8B, Qwen3.5-9B, and InternVL3-8B.  These models share the same eight benchmark columns, followed by Acc and Rel.  Qwen3.5 uses the fixed $K=12$ configuration of Appendix~\ref{sec:appendix_qwen35_implementation}.  VScan and SparseVLM are not reported on Qwen3.5: both prune inside the decoder at layers fixed by their released configurations, using those layers' self-attention maps, and in Qwen3.5 these layers fall on Gated DeltaNet (linear-attention) blocks that expose no attention maps. Moving their pruning layers to a full-attention block would change the methods as published, so we do not report them.

% Consolidated cross-page table for the complete Qwen-VL and InternVL results.
\begingroup
\scriptsize
\setlength{\tabcolsep}{0pt}
\setlength{\LTleft}{0pt}
\setlength{\LTright}{0pt}
\setlength{\LTcapwidth}{\linewidth}
\renewcommand{\arraystretch}{1.02}
\begin{longtable}{@{}
>{\raggedright\arraybackslash}p{\dimexpr0.18\linewidth-2\tabcolsep\relax}%
*{8}{>{\centering\arraybackslash}p{\dimexpr0.075\linewidth-2\tabcolsep\relax}}%
>{\centering\arraybackslash}p{\dimexpr0.11\linewidth-2\tabcolsep\relax}%
>{\centering\arraybackslash}p{\dimexpr0.11\linewidth-2\tabcolsep\relax}@{}}
\caption{Complete benchmark results for Qwen-VL and InternVL models, including MiCo-S1.}\label{tab:appendix_qwen_intern_full}\\
\toprule
Method & AI2D & POPE & HallB & MME & MMB-EN & MMB-CN & MMStar & SQA-IMG & Acc & Rel \\
\midrule
\endfirsthead
\multicolumn{11}{l}{\textit{Table \thetable\ (continued): Complete benchmark results for Qwen-VL and InternVL models, including MiCo-S1}}\\
\toprule
Method & AI2D & POPE & HallB & MME & MMB-EN & MMB-CN & MMStar & SQA-IMG & Acc & Rel \\
\midrule
\endhead
\midrule
\multicolumn{11}{r}{\textit{Continued on next page}}\\
\endfoot
\bottomrule
\endlastfoot
\rowcolor{black!4}
\multicolumn{11}{c}{\textbf{Qwen2.5-VL-7B} \label{tab:pruning_comparison_qwen25_full}}\\
\midrule
Vanilla & 84.9 & 87.7 & 55.9 & 2302.0 & 84.8 & 82.9 & 65.5 & 86.8 & 83.0 & 100.0\% \\
\midrule
\multicolumn{11}{c}{\textit{Retain 256 tokens (80.2\% pruned)}} \\
\midrule
FastV (ECCV'24) & 78.4 & 83.0 & 49.1 & 2169.0 & 80.5 & 78.8 & 55.5 & 83.6 & 77.2 & 93.0\% \\
SparseVLM (ICML'25) & 77.6 & 82.9 & 46.4 & 2207.5 & 80.9 & 79.4 & 55.5 & \textbf{87.7} & 77.6 & 93.5\% \\
DivPrune (CVPR'25) & 81.2 & 85.3 & 46.6 & 2167.0 & 81.8 & 80.9 & 57.9 & 84.8 & 78.4 & 94.5\% \\
VisionZip (CVPR'25) & 76.0 & 82.1 & 42.1 & 2101.2 & 78.7 & 78.6 & 54.1 & 82.8 & 74.9 & 90.3\% \\
VisPruner (ICCV'25) & 77.7 & 84.1 & 41.8 & 2084.1 & 80.7 & 79.7 & 56.3 & 84.6 & 76.1 & 91.8\% \\
VScan (TMLR'26) & 79.1 & \textbf{86.6} & \underline{49.7} & 2223.4 & 80.6 & 80.2 & 57.1 & 84.4 & 78.6 & 94.8\% \\
HoloV (NeurIPS'25) & \underline{82.5} & 86.2 & 47.5 & 2229.1 & 81.8 & \underline{81.0} & 58.9 & 85.1 & 79.3 & 95.6\% \\
MMTok (ICLR'26) & 81.0 & 85.7 & 47.9 & 2245.6 & \underline{83.1} & \underline{81.0} & 59.1 & 85.5 & 79.4 & 95.8\% \\
ApET (CVPR'26) & 77.8 & 82.6 & 44.4 & 2174.7 & 79.0 & 78.2 & 53.7 & 83.7 & 76.0 & 91.7\% \\
\rowcolor{micoPurple}
\textbf{MiCo-S1} & 81.7 & 86.1 & 48.6 & \underline{2282.2} & 82.0 & 80.5 & \underline{59.5} & 85.5 & \underline{79.8} & \underline{96.1\%} \\
\rowcolor{micoPurple}
\textbf{MiCo} & \textbf{83.8} & \underline{86.3} & \textbf{51.2} & \textbf{2291.0} & \textbf{83.2} & \textbf{81.3} & \textbf{60.1} & \underline{86.2} & \textbf{80.8} & \textbf{97.4\%} \\
\midrule
\multicolumn{11}{c}{\textit{Retain 128 tokens (90.1\% pruned)}} \\
\midrule
FastV (ECCV'24) & 69.9 & 67.7 & 41.0 & 1597.0 & 66.4 & 68.8 & 43.7 & 79.6 & 64.6 & 77.9\% \\
SparseVLM (ICML'25) & 70.7 & 72.1 & 37.7 & 1815.3 & 72.9 & 74.1 & 45.7 & 81.9 & 68.2 & 82.3\% \\
DivPrune (CVPR'25) & 75.9 & 83.9 & 44.5 & 2044.0 & 79.2 & \underline{78.6} & 52.3 & 82.8 & 74.9 & 90.3\% \\
VisionZip (CVPR'25) & 74.1 & 79.9 & 38.8 & 2013.4 & 79.3 & 76.4 & 52.7 & 80.9 & 72.8 & 87.8\% \\
VisPruner (ICCV'25) & 73.8 & 80.1 & 39.7 & 2007.8 & 78.7 & 77.1 & 52.3 & 82.1 & 73.0 & 88.0\% \\
VScan (TMLR'26) & 75.4 & \underline{84.8} & \underline{46.4} & 2024.5 & 79.2 & 78.0 & 53.5 & 82.9 & 75.2 & 90.6\% \\
HoloV (NeurIPS'25) & 74.3 & 80.3 & 39.3 & 1987.9 & 79.1 & 77.4 & 51.7 & 82.3 & 73.0 & 88.0\% \\
MMTok (ICLR'26) & 76.3 & 84.2 & 43.3 & 2120.7 & 79.8 & 77.5 & 53.7 & 81.4 & 75.3 & 90.8\% \\
ApET (CVPR'26) & 70.5 & 76.7 & 36.4 & 1911.3 & 74.6 & 73.4 & 47.7 & 79.3 & 69.3 & 83.5\% \\
\rowcolor{micoPurple}
\textbf{MiCo-S1} & \underline{78.8} & 84.3 & 46.3 & \underline{2176.4} & \underline{80.6} & 78.4 & \underline{54.6} & \underline{83.5} & \underline{76.9} & \underline{92.7\%} \\
\rowcolor{micoPurple}
\textbf{MiCo} & \textbf{80.4} & \textbf{85.5} & \textbf{48.6} & \textbf{2252.4} & \textbf{80.8} & \textbf{79.1} & \textbf{56.5} & \textbf{84.6} & \textbf{78.5} & \textbf{94.7\%} \\
\midrule
\rowcolor{black!4}
\multicolumn{11}{c}{\textbf{Qwen3-VL-8B} \label{tab:pruning_comparison_qwen3}}\\
\midrule
Vanilla & 84.1 & 89.7 & 56.2 & 2406.0 & 86.3 & 86.1 & 67.1 & 94.6 & 85.6 & 100.0\% \\
\midrule
\multicolumn{11}{c}{\textit{Retain 256 tokens}} \\
\midrule
FastV (ECCV'24) & 69.3 & 81.7 & 43.0 & 1913.0 & 77.3 & 76.7 & 51.1 & 82.0 & 72.1 & 84.3\% \\
SparseVLM (ICML'25) & 78.8 & 86.8 & 46.5 & 2304.5 & 83.8 & \textbf{83.7} & 58.3 & 91.2 & 80.5 & 94.1\% \\
DivPrune (CVPR'25) & 80.0 & 89.4 & 46.9 & 2236.0 & 83.6 & 82.3 & 58.3 & 88.0 & 80.0 & 93.6\% \\
VisionZip (CVPR'25) & 81.4 & 89.0 & 50.4 & 2334.2 & 83.8 & \underline{83.6} & 60.2 & \textbf{92.0} & 82.1 & 96.0\% \\
VisPruner (ICCV'25) & 79.9 & 89.2 & 47.7 & 2238.0 & 83.8 & 82.9 & 59.3 & 90.1 & 80.6 & 94.2\% \\
VScan (TMLR'26) & 80.2 & 87.8 & \textbf{51.8} & 2272.0 & \textbf{84.6} & 82.6 & 61.1 & 91.0 & 81.6 & 95.4\% \\
HoloV (NeurIPS'25) & \textbf{82.0} & 88.9 & 50.3 & 2305.2 & 83.4 & 83.5 & \underline{61.2} & \underline{91.6} & 82.0 & 95.9\% \\
MMTok (ICLR'26) & 81.6 & 89.3 & 49.6 & 2331.1 & \underline{84.5} & 82.6 & \textbf{61.4} & 90.2 & 82.0 & 95.8\% \\
ApET (CVPR'26) & 78.6 & \underline{89.5} & 47.0 & 2243.1 & 83.8 & 81.7 & 59.4 & 90.8 & 80.4 & 93.9\% \\
\rowcolor{micoPurple}
\textbf{MiCo-S1} & \underline{81.8} & 89.3 & 51.3 & \underline{2338.9} & 83.6 & 82.5 & 61.0 & 91.0 & \underline{82.2} & \underline{96.1\%} \\
\rowcolor{micoPurple}
\textbf{MiCo} & 81.6 & \textbf{89.6} & \underline{51.4} & \textbf{2375.3} & 84.3 & \underline{83.6} & 59.7 & 91.4 & \textbf{82.5} & \textbf{96.5\%} \\
\midrule
\multicolumn{11}{c}{\textit{Retain 128 tokens}} \\
\midrule
FastV (ECCV'24) & 66.3 & 57.2 & 30.3 & 1312.0 & 51.6 & 50.0 & 38.3 & 76.4 & 54.5 & 63.7\% \\
SparseVLM (ICML'25) & 72.2 & 81.2 & 40.1 & 1941.4 & 78.4 & 77.8 & 50.1 & 85.5 & 72.8 & 85.1\% \\
DivPrune (CVPR'25) & 74.4 & 88.3 & 42.6 & 2089.0 & 80.8 & 79.5 & 52.2 & 83.8 & 75.8 & 88.6\% \\
VisionZip (CVPR'25) & 75.8 & 86.7 & 41.5 & 2142.4 & 80.8 & 80.0 & 53.9 & 87.2 & 76.6 & 89.6\% \\
VisPruner (ICCV'25) & 72.4 & 86.3 & 40.5 & 1997.3 & 79.8 & 78.1 & 52.3 & 84.9 & 74.3 & 86.8\% \\
VScan (TMLR'26) & 76.2 & 85.6 & 44.1 & 2146.1 & 81.2 & \underline{80.1} & 55.3 & 86.8 & 77.1 & 90.1\% \\
HoloV (NeurIPS'25) & \textbf{79.5} & 87.2 & 42.1 & 2126.0 & 81.4 & 79.8 & \textbf{55.9} & 85.7 & 77.2 & 90.3\% \\
MMTok (ICLR'26) & 77.5 & \textbf{89.2} & 46.0 & 2223.7 & \textbf{82.3} & 79.0 & \underline{55.8} & 87.3 & 78.5 & 91.8\% \\
ApET (CVPR'26) & 71.5 & \underline{89.1} & 42.9 & 2044.0 & 79.2 & 77.7 & 53.5 & 87.3 & 75.4 & 88.2\% \\
\rowcolor{micoPurple}
\textbf{MiCo-S1} & 77.6 & 88.4 & \underline{46.5} & \textbf{2281.7} & \underline{81.8} & 79.3 & 55.7 & \textbf{88.2} & \underline{78.9} & \underline{92.3\%} \\
\rowcolor{micoPurple}
\textbf{MiCo} & \underline{78.7} & 88.6 & \textbf{48.3} & \underline{2268.6} & 81.5 & \textbf{81.0} & 55.4 & \underline{88.1} & \textbf{79.4} & \textbf{92.8\%} \\
\midrule
\rowcolor{black!4}
\multicolumn{11}{c}{\textbf{Qwen3.5-9B} \label{tab:pruning_comparison_qwen35}}\\
\midrule
Vanilla & 89.2 & 90.8 & 56.8 & 2381.5 & 86.4 & 85.7 & 69.8 & 95.2 & 86.6 & 100.0\% \\
\midrule
\multicolumn{11}{c}{\textit{Retain 256 tokens}} \\
\midrule
FastV (ECCV'24) & 73.9 & 84.1 & 42.8 & 1835.3 & 75.5 & 76.0 & 53.1 & 85.2 & 72.8 & 84.0\% \\
DivPrune (CVPR'25) & 86.6 & 90.4 & 50.6 & 2250.2 & 85.0 & 84.0 & 63.7 & 91.6 & 83.0 & 95.9\% \\
VisionZip (CVPR'25) & 87.1 & 90.4 & 54.0 & \textbf{2352.7} & 85.7 & 84.7 & 65.3 & \textbf{94.1} & \underline{84.9} & 98.0\% \\
VisPruner (ICCV'25) & 86.7 & 90.1 & 54.4 & 2286.9 & \underline{85.9} & \underline{85.1} & 62.9 & 92.9 & 84.1 & 97.0\% \\
HoloV (NeurIPS'25) & \textbf{87.4} & 90.2 & 54.3 & 2340.3 & 85.7 & 84.6 & 65.5 & \underline{93.8} & 84.8 & 97.9\% \\
MMTok (ICLR'26) & 87.0 & \underline{90.7} & 54.7 & 2334.5 & 85.1 & 83.2 & 64.7 & 91.9 & 84.3 & 97.3\% \\
ApET (CVPR'26) & \underline{87.3} & \textbf{90.8} & 53.4 & 2331.0 & 85.8 & 84.6 & \textbf{67.5} & 93.5 & \underline{84.9} & \underline{98.1\%} \\
\rowcolor{micoPurple}
\textbf{MiCo-S1} & 87.2 & 90.5 & \underline{54.9} & 2335.5 & 85.1 & 84.2 & 63.8 & 92.7 & 84.4 & 97.4\% \\
\rowcolor{micoPurple}
\textbf{MiCo} & 87.2 & \textbf{90.8} & \textbf{55.5} & \underline{2350.7} & \textbf{86.5} & \textbf{85.6} & \underline{67.1} & \underline{93.8} & \textbf{85.5} & \textbf{98.7\%} \\
\midrule
\multicolumn{11}{c}{\textit{Retain 128 tokens}} \\
\midrule
FastV (ECCV'24) & 67.0 & 62.7 & 31.5 & 1211.4 & 49.7 & 49.6 & 39.0 & 75.3 & 54.4 & 62.8\% \\
DivPrune (CVPR'25) & 82.0 & 89.7 & 44.6 & 2160.4 & \underline{83.8} & 81.6 & 57.4 & 89.2 & 79.5 & 91.8\% \\
VisionZip (CVPR'25) & 82.1 & 88.4 & 46.8 & 2140.3 & 83.1 & 81.8 & 58.0 & 90.4 & 79.7 & 92.0\% \\
VisPruner (ICCV'25) & 81.5 & 88.1 & 44.0 & 2069.2 & 81.7 & 79.8 & 56.2 & 89.2 & 78.0 & 90.1\% \\
HoloV (NeurIPS'25) & 82.4 & 88.7 & 48.8 & 2109.5 & 82.6 & \underline{82.5} & 59.2 & 89.7 & 79.9 & 92.3\% \\
MMTok (ICLR'26) & \underline{83.7} & \underline{90.5} & 47.8 & 2177.8 & 82.1 & 80.0 & \underline{60.5} & 89.5 & 80.4 & 92.8\% \\
ApET (CVPR'26) & 80.2 & \textbf{90.7} & 49.4 & 2141.8 & 83.0 & 81.9 & 59.7 & \underline{91.1} & 80.4 & 92.8\% \\
\rowcolor{micoPurple}
\textbf{MiCo-S1} & 83.6 & 89.7 & \underline{49.8} & \underline{2193.5} & 82.0 & 80.7 & 58.8 & 89.7 & \underline{80.5} & \underline{92.9\%} \\
\rowcolor{micoPurple}
\textbf{MiCo} & \textbf{85.6} & \textbf{90.7} & \textbf{51.4} & \textbf{2306.4} & \textbf{85.2} & \textbf{83.8} & \textbf{61.7} & \textbf{91.4} & \textbf{83.2} & \textbf{96.0\%} \\
\midrule
\rowcolor{black!4}
\multicolumn{11}{c}{\textbf{InternVL3-8B} \label{tab:pruning_comparison_internvl3_full}}\\
\midrule
Vanilla & 85.1 & 90.7 & 49.4 & 2369.0 & 85.7 & 85.1 & 68.3 & 97.8 & 85.1 & 100.0\% \\
\midrule
\multicolumn{11}{c}{\textit{Retain 256 tokens}} \\
\midrule
FastV (ECCV'24) & \underline{80.5} & 89.1 & 44.0 & \textbf{2289.0} & 83.6 & \underline{83.7} & \underline{61.2} & 93.3 & \underline{81.2} & \underline{95.5\%} \\
SparseVLM (ICML'25) & 75.4 & 86.3 & 36.7 & 2217.6 & 81.3 & 81.8 & 57.3 & 89.8 & 77.4 & 91.0\% \\
DivPrune (CVPR'25) & 80.3 & 89.8 & 43.0 & 2178.0 & 81.9 & 80.5 & 58.9 & 91.8 & 79.4 & 93.3\% \\
VisionZip (CVPR'25) & 76.0 & 87.0 & 40.0 & 2148.8 & 82.0 & 81.2 & 55.1 & 90.7 & 77.4 & 91.0\% \\
VisPruner (ICCV'25) & 72.3 & 88.9 & 37.6 & 2037.7 & 81.2 & 79.5 & 56.5 & 90.8 & 76.1 & 89.4\% \\
VScan (TMLR'26) & 76.7 & 87.6 & 41.4 & 2230.0 & 82.4 & 80.5 & 58.5 & 93.9 & 79.1 & 92.9\% \\
HoloV (NeurIPS'25) & 77.9 & 87.5 & 41.7 & 2194.0 & 82.1 & 80.8 & 57.0 & 91.8 & 78.6 & 92.4\% \\
MMTok (ICLR'26) & 80.2 & 90.0 & 41.1 & 2268.0 & 81.8 & 82.0 & 58.5 & \underline{94.0} & 80.1 & 94.2\% \\
ApET (CVPR'26) & 77.9 & \underline{90.5} & 42.5 & 2228.8 & 80.4 & 79.0 & 58.2 & 92.0 & 79.0 & 92.9\% \\
\rowcolor{micoPurple}
\textbf{MiCo-S1} & 80.1 & \textbf{90.8} & \underline{44.1} & \underline{2279.1} & \underline{83.7} & 83.2 & 59.7 & 93.4 & 81.1 & 95.4\% \\
\rowcolor{micoPurple}
\textbf{MiCo} & \textbf{82.9} & 90.2 & \textbf{44.4} & 2276.1 & \textbf{84.7} & \textbf{84.9} & \textbf{62.9} & \textbf{95.2} & \textbf{82.4} & \textbf{96.8\%} \\
\midrule
\multicolumn{11}{c}{\textit{Retain 128 tokens}} \\
\midrule
FastV (ECCV'24) & 68.4 & 78.7 & \underline{38.9} & 1807.0 & 73.7 & 73.5 & 47.3 & 83.8 & 69.3 & 81.5\% \\
SparseVLM (ICML'25) & 69.7 & 79.6 & 32.4 & 1849.4 & 75.3 & 75.1 & 49.0 & 83.5 & 69.6 & 81.9\% \\
DivPrune (CVPR'25) & 74.2 & 88.8 & 37.6 & 2051.0 & 78.3 & 75.7 & 52.0 & 87.5 & 74.6 & 87.7\% \\
VisionZip (CVPR'25) & 68.5 & 81.5 & 32.7 & 1864.7 & 74.7 & 73.5 & 49.2 & 81.7 & 69.4 & 81.5\% \\
VisPruner (ICCV'25) & 68.5 & 84.4 & 33.0 & 1769.8 & 72.9 & 72.8 & 49.0 & 82.8 & 69.0 & 81.1\% \\
VScan (TMLR'26) & 71.9 & 84.7 & 37.8 & 2028.0 & 78.8 & 76.9 & 53.9 & \underline{89.0} & 74.3 & 87.3\% \\
HoloV (NeurIPS'25) & 72.4 & 84.1 & 37.9 & 1968.0 & 78.4 & 77.4 & 53.2 & 86.7 & 73.6 & 86.5\% \\
MMTok (ICLR'26) & \underline{74.7} & 88.6 & 37.9 & \underline{2128.9} & 78.3 & 77.4 & 52.9 & 87.8 & 75.5 & 88.8\% \\
ApET (CVPR'26) & 74.3 & \textbf{89.7} & 38.1 & 2087.9 & 76.2 & 75.9 & 53.9 & 87.5 & 75.0 & 88.2\% \\
\rowcolor{micoPurple}
\textbf{MiCo-S1} & 73.9 & \underline{89.3} & 36.8 & 2112.5 & \underline{81.0} & \underline{79.0} & \underline{55.5} & 88.9 & \underline{76.2} & \underline{89.6\%} \\
\rowcolor{micoPurple}
\textbf{MiCo} & \textbf{77.4} & \textbf{89.7} & \textbf{40.7} & \textbf{2260.5} & \textbf{82.8} & \textbf{81.4} & \textbf{58.2} & \textbf{91.5} & \textbf{79.3} & \textbf{93.3\%} \\
\end{longtable}
\endgroup

\FloatBarrier
\subsection{Refinement-Layer Sweep Tables}

The following two long tables report all 204 LLaVA model--budget--layer settings: 15 layers for each 7B model and 19 for each 13B model, across three budgets. Acc$_{10}$ is the equal-weight mean of GQA, SQA-IMG, TextVQA, POPE, MME$/20$, MMBench-EN, MMBench-CN, SEED, AI2D, and MMMU. These sweeps use POPE average F1, MME Perception, circular MMBench scoring, image-only SEED, and the recorded single-image MMMU protocol. Rel normalizes Acc$_{10}$ by the corresponding unpruned reference. Aggregation preserves available source precision: reused scores may already be rounded, whereas newly measured scores retain evaluator precision. Small differences at the displayed precision should not be interpreted as distinct robust peaks. Pale red rows mark the original selected layers, and score highlighting marks the highest and second-highest distinct displayed values within each sweep.

% Regrouped from the twelve layer-sweep floats into VTC-style cross-page tables.
% Each section is one model/budget setting; values and highlighted rows are unchanged.
\begingroup
\scriptsize
\setlength{\tabcolsep}{1.25pt}
\setlength{\LTleft}{0pt}
\setlength{\LTright}{0pt}
\setlength{\LTcapwidth}{\linewidth}
\renewcommand{\arraystretch}{1.0}
\begin{longtable}{>{\centering\arraybackslash}p{\dimexpr0.05\linewidth-2\tabcolsep\relax}
  *{12}{>{\raggedleft\arraybackslash}p{\dimexpr0.0791667\linewidth-2\tabcolsep\relax}}}
\caption{Complete refinement-layer sweeps for LLaVA-1.5 models across the three reported budgets.}
\label{tab:layer_sweeps_l15}\\
\toprule
$K$ & GQA & \shortstack{SQA\\IMG} & TextVQA & POPE & MME & \shortstack{MMB\\EN} & \shortstack{MMB\\CN} & SEED & AI2D & MMMU & Acc & Rel \\
\midrule
\endfirsthead
\multicolumn{13}{l}{\textit{Table \thetable\ (continued): LLaVA-1.5 refinement-layer sweeps}}\\
\toprule
$K$ & GQA & \shortstack{SQA\\IMG} & TextVQA & POPE & MME & \shortstack{MMB\\EN} & \shortstack{MMB\\CN} & SEED & AI2D & MMMU & Acc & Rel \\
\midrule
\endhead
\midrule
\multicolumn{13}{r}{\textit{Continued on next page}}\\
\endfoot
\bottomrule
\endlastfoot
\rowcolor{black!4}
\multicolumn{13}{c}{\textbf{LLaVA-1.5-7B \quad $T=32$}}\\
\midrule
Vanilla & 61.9 & 69.5 & 58.2 & 85.9 & 1508.8 & 64.7 & 58.1 & 66.0 & 55.5 & 35.0 & 63.0 & 100.0\% \\
\midrule
1 & 56.1 & \underline{69.0} & 54.3 & 80.4 & 1295.0 & 58.9 & 52.2 & 57.5 & 52.7 & 33.3 & 57.9 & 91.9\% \\
2 & 55.6 & 68.6 & 54.7 & 79.7 & 1298.2 & 58.4 & 52.2 & 57.8 & 53.3 & 34.4 & 58.0 & 92.0\% \\
3 & 55.6 & 68.8 & 54.9 & 79.0 & 1336.5 & 59.7 & 53.0 & 57.9 & 53.0 & 34.2 & 58.3 & 92.5\% \\
4 & 55.5 & 68.8 & 54.6 & 78.7 & 1323.3 & 59.5 & 52.7 & 57.9 & 53.0 & \underline{35.2} & 58.2 & 92.4\% \\
5 & 55.7 & 68.8 & 54.6 & 78.7 & 1331.1 & 59.4 & 53.4 & 57.9 & 52.9 & 34.9 & 58.3 & 92.5\% \\
6 & 55.6 & 68.7 & \underline{55.0} & 78.8 & 1297.6 & 58.8 & 53.4 & 58.2 & 52.6 & 34.8 & 58.1 & 92.1\% \\
\rowcolor{micoPurple}
7 & 55.9 & 68.7 & \textbf{55.2} & 80.9 & 1355.0 & 59.2 & 54.0 & 58.2 & \underline{53.6} & 34.1 & 58.8 & 93.3\% \\
8 & 56.0 & 68.6 & 54.6 & 81.2 & 1332.6 & 59.0 & 53.7 & 58.2 & \textbf{53.7} & 34.9 & 58.7 & 93.1\% \\
9 & \underline{56.6} & 68.9 & 54.6 & 82.1 & \underline{1398.6} & 59.9 & 53.8 & 59.1 & 53.2 & 34.9 & 59.3 & 94.1\% \\
10 & 55.9 & 68.3 & 54.4 & 80.4 & 1357.4 & 59.9 & 52.9 & 59.1 & 53.1 & 34.0 & 58.6 & 93.0\% \\
11 & 56.3 & 68.9 & 54.6 & 83.2 & 1394.7 & 59.5 & 53.8 & 59.8 & 53.0 & 35.1 & 59.4 & 94.2\% \\
12 & \textbf{57.1} & 68.5 & 54.6 & \textbf{84.4} & \textbf{1405.3} & 59.7 & \underline{56.1} & \underline{60.6} & 53.2 & 34.9 & \textbf{59.9} & \textbf{95.1\%} \\
13 & 56.4 & 68.9 & 54.6 & 82.9 & 1398.2 & 60.2 & \textbf{56.3} & \underline{60.6} & \underline{53.6} & 35.0 & \underline{59.8} & \underline{94.9\%} \\
14 & 54.7 & \textbf{69.1} & 54.0 & 81.6 & 1395.8 & \textbf{60.5} & 56.0 & \textbf{60.7} & 53.3 & \textbf{35.4} & 59.5 & 94.4\% \\
15 & 55.1 & \underline{69.0} & 54.3 & \underline{83.4} & 1398.1 & \underline{60.3} & \underline{56.1} & \textbf{60.7} & \textbf{53.7} & 35.0 & \underline{59.8} & 94.8\% \\
\midrule
\rowcolor{black!4}
\multicolumn{13}{c}{\textbf{LLaVA-1.5-7B \quad $T=64$}}\\
\midrule
Vanilla & 61.9 & 69.5 & 58.2 & 85.9 & 1508.8 & 64.7 & 58.1 & 66.0 & 55.5 & 35.0 & 63.0 & 100.0\% \\
\midrule
1 & 58.4 & 69.0 & 55.8 & 84.7 & 1397.1 & 60.3 & 54.0 & 61.1 & 53.7 & 34.7 & 60.1 & 95.4\% \\
2 & 58.2 & \textbf{69.8} & 56.3 & 83.7 & 1390.1 & 59.8 & 54.5 & 60.7 & 53.6 & 34.7 & 60.1 & 95.3\% \\
3 & 57.7 & \underline{69.7} & \underline{56.5} & 82.7 & 1403.5 & 60.5 & 55.4 & 60.9 & 53.8 & \textbf{35.4} & 60.3 & 95.6\% \\
4 & 57.7 & \underline{69.7} & \textbf{56.6} & 83.3 & 1418.8 & 61.1 & 55.4 & 60.9 & 54.0 & 34.9 & 60.5 & 95.9\% \\
5 & 57.9 & 69.0 & 56.1 & 83.0 & 1411.5 & \textbf{61.6} & 55.8 & 60.9 & 53.9 & 34.3 & 60.3 & 95.7\% \\
6 & 57.7 & \underline{69.7} & 56.3 & 83.3 & 1394.5 & 61.1 & 55.2 & 60.5 & 54.4 & 34.4 & 60.2 & 95.6\% \\
\rowcolor{micoPurple}
7 & 58.1 & \underline{69.7} & 56.4 & 84.7 & \underline{1426.8} & 60.4 & 55.7 & 61.0 & 54.1 & 34.2 & 60.6 & 96.1\% \\
8 & 57.9 & 69.2 & 56.1 & 84.7 & 1421.5 & 61.2 & 56.0 & 60.9 & 54.3 & 34.1 & 60.6 & 96.1\% \\
9 & 57.9 & 69.0 & 56.2 & 85.0 & 1411.8 & 60.9 & 56.2 & 61.6 & \underline{54.7} & 34.6 & 60.7 & 96.3\% \\
10 & 57.9 & 68.8 & 56.1 & 84.5 & 1420.3 & 61.3 & 55.8 & 61.8 & 54.4 & 33.3 & 60.5 & 96.0\% \\
11 & 58.1 & 69.3 & 56.4 & \underline{86.2} & 1406.5 & \underline{61.4} & 56.2 & 62.2 & 54.0 & 34.6 & \underline{60.9} & 96.6\% \\
12 & \textbf{58.8} & 68.5 & \underline{56.5} & \textbf{86.6} & 1399.7 & 60.4 & 57.3 & \textbf{62.6} & 54.1 & 34.6 & \underline{60.9} & 96.7\% \\
13 & \underline{58.5} & 69.0 & 56.0 & 85.9 & 1406.6 & 61.1 & \underline{57.5} & \textbf{62.6} & 54.6 & \underline{35.0} & \textbf{61.0} & \textbf{96.9\%} \\
14 & 57.8 & 68.9 & 55.2 & 85.4 & \textbf{1434.2} & 61.0 & \textbf{57.6} & \underline{62.4} & \underline{54.7} & 34.8 & \underline{60.9} & 96.7\% \\
15 & 58.2 & 68.7 & 56.1 & 85.9 & 1420.7 & \underline{61.4} & 56.8 & \textbf{62.6} & \textbf{54.8} & 34.3 & \textbf{61.0} & \underline{96.8\%} \\
\midrule
\rowcolor{black!4}
\multicolumn{13}{c}{\textbf{LLaVA-1.5-7B \quad $T=128$}}\\
\midrule
Vanilla & 61.9 & 69.5 & 58.2 & 85.9 & 1508.8 & 64.7 & 58.1 & 66.0 & 55.5 & 35.0 & 63.0 & 100.0\% \\
\midrule
1 & 59.7 & 68.9 & 56.8 & 86.0 & 1445.0 & 62.7 & 56.5 & 63.8 & 54.6 & 34.9 & 61.6 & 97.8\% \\
2 & 59.5 & \textbf{69.5} & 57.5 & 85.0 & 1463.7 & 62.4 & 56.4 & 63.6 & 55.0 & 35.1 & 61.7 & 97.9\% \\
3 & 59.6 & \underline{69.4} & 57.2 & 84.8 & 1460.7 & 62.5 & 56.6 & 63.9 & 55.0 & 35.2 & 61.7 & 97.9\% \\
4 & 59.2 & \underline{69.4} & 57.5 & 84.8 & 1472.3 & 62.1 & 56.8 & 63.9 & \textbf{55.5} & 35.6 & 61.8 & 98.1\% \\
5 & 59.7 & 69.3 & 57.5 & 84.6 & 1466.9 & 62.3 & 57.2 & 63.4 & 55.3 & 35.9 & 61.9 & 98.2\% \\
6 & 59.2 & 69.1 & \textbf{57.7} & 85.6 & 1457.3 & 61.9 & 56.7 & 63.5 & 54.9 & 35.6 & 61.7 & 97.9\% \\
\rowcolor{micoPurple}
7 & 60.0 & 69.3 & \underline{57.6} & 85.9 & 1474.8 & 62.4 & 56.5 & 63.7 & 54.9 & 36.3 & 62.0 & 98.5\% \\
8 & 59.5 & 69.3 & \textbf{57.7} & 86.1 & \underline{1481.4} & 62.5 & 56.9 & 63.6 & 54.9 & \underline{36.6} & 62.1 & 98.6\% \\
9 & 59.9 & 68.9 & 57.2 & 86.2 & 1476.9 & 62.5 & 56.9 & 64.0 & 54.8 & 36.1 & 62.0 & 98.4\% \\
10 & 60.2 & 68.5 & 57.2 & 85.7 & 1463.2 & \textbf{63.0} & 56.6 & 64.1 & 54.7 & 36.2 & 61.9 & 98.3\% \\
11 & 60.0 & 68.6 & 57.1 & 86.6 & 1471.9 & 62.6 & 57.2 & 64.4 & 55.0 & \textbf{37.0} & 62.2 & 98.7\% \\
12 & \underline{60.4} & 68.5 & \textbf{57.7} & \textbf{87.0} & 1461.5 & \textbf{63.0} & \textbf{57.6} & \underline{64.6} & 55.2 & \textbf{37.0} & \textbf{62.4} & \textbf{99.0\%} \\
13 & \textbf{60.5} & 68.0 & 57.1 & \underline{86.8} & 1477.5 & \underline{62.9} & \underline{57.5} & \textbf{64.8} & \underline{55.4} & 36.1 & \underline{62.3} & \underline{98.8\%} \\
14 & 59.6 & 68.7 & 56.7 & 86.3 & \textbf{1487.3} & 62.6 & 57.4 & \textbf{64.8} & \underline{55.4} & 36.2 & 62.2 & 98.7\% \\
15 & 60.2 & 68.8 & 56.8 & 86.4 & 1465.9 & \underline{62.9} & 57.3 & \textbf{64.8} & \underline{55.4} & 36.3 & 62.2 & 98.7\% \\
\midrule
\rowcolor{black!4}
\multicolumn{13}{c}{\textbf{LLaVA-1.5-13B \quad $T=32$}}\\
\midrule
Vanilla & 63.3 & 72.8 & 61.2 & 86.0 & 1533.2 & 68.5 & 63.5 & 68.2 & 60.8 & 36.4 & 65.7 & 100.0\% \\
\midrule
1 & 56.7 & 72.3 & 56.7 & 79.7 & 1370.0 & 63.1 & 60.3 & 60.6 & 57.2 & 34.1 & 60.9 & 92.7\% \\
2 & 57.3 & 72.5 & 57.2 & 78.4 & 1413.2 & 64.4 & 60.1 & 60.8 & \underline{57.4} & 35.7 & 61.4 & 93.5\% \\
3 & 57.2 & 72.6 & 57.2 & 78.5 & 1417.1 & 65.0 & 59.6 & 60.6 & 56.7 & 34.7 & 61.3 & 93.2\% \\
4 & 57.1 & 72.4 & 57.0 & 78.0 & 1430.7 & 63.9 & 60.2 & 60.7 & 56.9 & 34.3 & 61.2 & 93.1\% \\
5 & 57.1 & 72.9 & 56.8 & 78.7 & 1416.5 & 64.3 & 60.1 & 60.7 & 56.5 & 35.3 & 61.3 & 93.3\% \\
6 & 56.9 & 72.8 & 56.9 & 79.9 & 1428.1 & 64.2 & 60.1 & 60.8 & \underline{57.4} & 35.0 & 61.5 & 93.6\% \\
7 & 57.2 & 72.2 & 57.5 & 80.1 & 1436.6 & 64.4 & 60.4 & 61.1 & 56.9 & 35.6 & 61.7 & 93.9\% \\
\rowcolor{micoPurple}
8 & 57.3 & 73.0 & 57.0 & 82.3 & \textbf{1466.2} & 64.6 & 61.0 & 61.5 & \underline{57.4} & 35.9 & 62.3 & 94.9\% \\
9 & 57.7 & \textbf{74.0} & 57.1 & 82.3 & 1418.8 & 64.9 & 60.3 & 61.9 & 57.0 & 35.7 & 62.2 & 94.6\% \\
10 & 57.2 & 72.7 & 57.4 & 82.1 & 1453.5 & 64.7 & 60.3 & 62.2 & \textbf{57.5} & \textbf{36.6} & 62.3 & 94.8\% \\
11 & \textbf{58.4} & 73.0 & 57.4 & 82.1 & \underline{1465.1} & \underline{65.6} & 61.7 & \underline{63.3} & 57.2 & \underline{36.0} & \textbf{62.8} & \textbf{95.5\%} \\
12 & \underline{58.1} & \underline{73.1} & \textbf{57.8} & 81.8 & 1444.1 & 65.3 & 61.8 & 63.2 & \textbf{57.5} & 35.6 & \underline{62.6} & \underline{95.3\%} \\
13 & 57.8 & 72.9 & 57.5 & \textbf{82.6} & 1448.6 & 65.4 & 61.2 & 63.1 & 57.3 & 35.7 & \underline{62.6} & 95.2\% \\
14 & 58.0 & 72.8 & \underline{57.7} & \underline{82.4} & 1452.9 & 65.1 & 61.7 & 63.2 & \underline{57.4} & 35.0 & \underline{62.6} & 95.2\% \\
15 & 57.6 & 73.0 & \textbf{57.8} & 82.3 & 1444.7 & 65.5 & \underline{61.9} & \underline{63.3} & \textbf{57.5} & 34.8 & \underline{62.6} & 95.2\% \\
16 & 57.7 & 72.9 & \textbf{57.8} & \textbf{82.6} & 1438.1 & \textbf{65.7} & \underline{61.9} & \underline{63.3} & \textbf{57.5} & 34.9 & \underline{62.6} & \underline{95.3\%} \\
17 & 57.7 & \underline{73.1} & \underline{57.7} & 82.3 & 1437.5 & \underline{65.6} & \textbf{62.0} & \underline{63.3} & \textbf{57.5} & 34.6 & \underline{62.6} & 95.2\% \\
18 & 57.4 & \underline{73.1} & 57.4 & \underline{82.4} & 1437.4 & 65.4 & \underline{61.9} & \textbf{63.4} & \underline{57.4} & 34.7 & 62.5 & 95.1\% \\
19 & 56.5 & \underline{73.1} & 56.2 & \underline{82.4} & 1444.9 & 65.5 & \underline{61.9} & \underline{63.3} & \underline{57.4} & 35.6 & 62.4 & 94.9\% \\
\midrule
\rowcolor{black!4}
\multicolumn{13}{c}{\textbf{LLaVA-1.5-13B \quad $T=64$}}\\
\midrule
Vanilla & 63.3 & 72.8 & 61.2 & 86.0 & 1533.2 & 68.5 & 63.5 & 68.2 & 60.8 & 36.4 & 65.7 & 100.0\% \\
\midrule
1 & 58.9 & 72.8 & 58.4 & 84.5 & 1468.0 & 66.3 & 61.0 & 63.7 & \underline{57.9} & 35.0 & 63.2 & 96.1\% \\
2 & 58.7 & 73.6 & 58.5 & 83.0 & 1467.5 & 65.9 & 61.3 & 63.3 & \textbf{58.0} & 36.2 & 63.2 & 96.1\% \\
3 & 58.7 & \underline{73.9} & 58.5 & 82.9 & 1455.7 & 65.6 & 61.7 & 63.4 & 57.7 & 36.0 & 63.1 & 96.0\% \\
4 & 58.6 & \textbf{74.2} & 58.5 & 82.7 & 1474.1 & 65.7 & 61.5 & 63.3 & 57.3 & 35.7 & 63.1 & 96.0\% \\
5 & 58.5 & 73.7 & \underline{58.6} & 83.6 & 1486.6 & 65.7 & 61.3 & 63.5 & 57.3 & 36.3 & 63.3 & 96.3\% \\
6 & 58.5 & 73.7 & 58.4 & 84.7 & \textbf{1498.8} & 65.8 & 61.4 & 63.2 & 57.3 & 36.4 & 63.4 & 96.5\% \\
7 & 58.7 & 73.2 & \textbf{58.8} & 84.6 & 1472.2 & 66.2 & 61.4 & 63.5 & 57.5 & 36.2 & 63.4 & 96.4\% \\
\rowcolor{micoPurple}
8 & 58.7 & 73.2 & 58.4 & \textbf{86.4} & \underline{1490.8} & 65.9 & 61.1 & 63.9 & 57.1 & 36.4 & 63.6 & 96.7\% \\
9 & 58.7 & 73.6 & 58.2 & \underline{86.0} & 1470.3 & 66.1 & 61.7 & 64.1 & 57.6 & 36.4 & 63.6 & 96.7\% \\
10 & 58.7 & 73.0 & \underline{58.6} & 85.5 & 1449.7 & 66.0 & 61.9 & 64.4 & \underline{57.9} & 36.8 & 63.5 & 96.6\% \\
11 & \textbf{59.5} & 73.5 & 58.5 & 85.7 & 1478.4 & 66.2 & 62.0 & \textbf{64.9} & 57.7 & \textbf{37.2} & \textbf{63.9} & \textbf{97.2\%} \\
12 & \underline{59.3} & 73.3 & 58.5 & 85.8 & 1464.4 & \underline{66.7} & 62.2 & 64.6 & 57.3 & 36.3 & 63.7 & 96.9\% \\
13 & 59.1 & \underline{73.9} & 58.3 & 85.6 & 1470.1 & 66.6 & 62.2 & 64.6 & 57.0 & 36.3 & 63.7 & 96.9\% \\
14 & \underline{59.3} & 73.7 & \underline{58.6} & 85.6 & 1455.9 & 66.6 & 62.3 & 64.6 & 57.1 & \textbf{37.2} & \underline{63.8} & \underline{97.0\%} \\
15 & 59.1 & \underline{73.9} & 58.4 & 85.5 & 1453.3 & \textbf{66.8} & \textbf{62.7} & 64.6 & 57.1 & 36.4 & 63.7 & 96.9\% \\
16 & 59.1 & 73.7 & 58.4 & 85.9 & 1448.3 & 66.6 & \underline{62.6} & 64.5 & 57.1 & \underline{36.9} & 63.7 & 96.9\% \\
17 & 58.8 & 73.7 & \underline{58.6} & 85.7 & 1461.6 & \underline{66.7} & 62.3 & 64.6 & 57.2 & 36.3 & 63.7 & 96.9\% \\
18 & 58.7 & 73.8 & 58.2 & 85.6 & 1446.6 & 66.6 & \underline{62.6} & 64.6 & 57.1 & 35.8 & 63.5 & 96.6\% \\
19 & 58.0 & 73.8 & 57.3 & 85.8 & 1458.3 & 66.6 & 62.4 & \underline{64.7} & 57.2 & 35.4 & 63.4 & 96.5\% \\
\midrule
\rowcolor{black!4}
\multicolumn{13}{c}{\textbf{LLaVA-1.5-13B \quad $T=128$}}\\
\midrule
Vanilla & 63.3 & 72.8 & 61.2 & 86.0 & 1533.2 & 68.5 & 63.5 & 68.2 & 60.8 & 36.4 & 65.7 & 100.0\% \\
\midrule
1 & 59.8 & 73.1 & 59.0 & 86.1 & 1515.3 & 66.6 & 62.2 & 65.3 & \underline{58.6} & 36.3 & 64.3 & 97.8\% \\
2 & 59.9 & \textbf{74.4} & 59.6 & 84.8 & \textbf{1524.0} & 67.0 & 62.5 & 65.4 & 57.8 & \textbf{36.9} & \underline{64.5} & 98.1\% \\
3 & 59.8 & \underline{74.3} & 59.2 & 84.9 & 1516.5 & 66.6 & 62.5 & 65.5 & 57.6 & 36.6 & 64.3 & 97.8\% \\
4 & 59.8 & \underline{74.3} & 59.4 & 84.9 & 1511.8 & 66.4 & 61.9 & 65.4 & 58.2 & 36.2 & 64.2 & 97.7\% \\
5 & 59.6 & 73.7 & 59.7 & 85.8 & 1518.7 & 66.6 & 62.3 & 65.6 & 58.0 & 36.3 & 64.4 & 97.9\% \\
6 & 59.9 & 73.5 & 59.6 & 86.1 & \underline{1523.6} & 67.0 & 62.6 & 65.3 & 57.8 & \underline{36.7} & \underline{64.5} & 98.1\% \\
7 & 59.9 & 73.3 & \textbf{59.9} & 86.1 & 1512.4 & \underline{67.4} & \textbf{63.1} & 65.5 & 58.1 & 36.2 & \underline{64.5} & 98.1\% \\
\rowcolor{micoPurple}
8 & \underline{60.3} & 73.7 & 59.6 & 86.6 & 1509.3 & 67.0 & 62.4 & 66.0 & 58.1 & 36.4 & \textbf{64.6} & \underline{98.3\%} \\
9 & 60.2 & 73.8 & 59.7 & \textbf{87.2} & 1482.1 & 67.2 & 62.8 & 66.1 & 58.5 & 36.1 & \textbf{64.6} & \underline{98.2\%} \\
10 & 60.1 & 73.7 & \textbf{59.9} & 86.6 & 1505.3 & 66.7 & 61.9 & 66.0 & \underline{58.6} & 35.9 & \underline{64.5} & 98.1\% \\
11 & 60.2 & 73.2 & \underline{59.8} & \underline{86.7} & 1500.5 & \textbf{67.5} & 62.5 & \underline{66.3} & \textbf{58.7} & 36.2 & \textbf{64.6} & \textbf{98.3\%} \\
12 & \textbf{60.5} & 73.0 & 59.7 & 86.5 & 1505.7 & 67.1 & 62.5 & 66.2 & 58.5 & 36.3 & \textbf{64.6} & \underline{98.2\%} \\
13 & 59.9 & 73.3 & 59.6 & \underline{86.7} & 1494.3 & 66.9 & 62.0 & 66.2 & 58.3 & 36.0 & 64.4 & 97.9\% \\
14 & 60.1 & 73.2 & 59.5 & \underline{86.7} & 1488.9 & 67.0 & 62.6 & \textbf{66.4} & 58.3 & 36.0 & 64.4 & 98.0\% \\
15 & 60.1 & 73.5 & 59.4 & 86.5 & 1494.1 & 67.2 & 62.7 & \textbf{66.4} & 58.4 & 36.6 & \textbf{64.6} & \underline{98.2\%} \\
16 & 59.8 & 73.5 & 59.4 & 86.6 & 1491.3 & 67.3 & 62.7 & \textbf{66.4} & 58.4 & 35.7 & 64.4 & 98.0\% \\
17 & 59.9 & 73.5 & 59.5 & 86.5 & 1491.1 & 67.3 & 62.7 & \textbf{66.4} & 58.5 & 36.2 & \underline{64.5} & 98.1\% \\
18 & 59.7 & 73.3 & 59.3 & 86.5 & 1490.4 & 67.3 & \underline{62.9} & \textbf{66.4} & 58.4 & 36.0 & 64.4 & 98.0\% \\
19 & 59.5 & 73.3 & 58.7 & 86.5 & 1495.9 & 67.3 & 62.7 & \textbf{66.4} & 58.4 & 36.0 & 64.4 & 97.9\% \\
\end{longtable}
\endgroup

% Regrouped from the twelve layer-sweep floats into VTC-style cross-page tables.
% Each section is one model/budget setting; values and highlighted rows are unchanged.
\begingroup
\scriptsize
\setlength{\tabcolsep}{1.25pt}
\setlength{\LTleft}{0pt}
\setlength{\LTright}{0pt}
\setlength{\LTcapwidth}{\linewidth}
\renewcommand{\arraystretch}{1.0}
\begin{longtable}{>{\centering\arraybackslash}p{\dimexpr0.05\linewidth-2\tabcolsep\relax}
  *{12}{>{\raggedleft\arraybackslash}p{\dimexpr0.0791667\linewidth-2\tabcolsep\relax}}}
\caption{Complete refinement-layer sweeps for LLaVA-NeXT models across the three reported budgets.}
\label{tab:layer_sweeps_next}\\
\toprule
$K$ & GQA & \shortstack{SQA\\IMG} & TextVQA & POPE & MME & \shortstack{MMB\\EN} & \shortstack{MMB\\CN} & SEED & AI2D & MMMU & Acc & Rel \\
\midrule
\endfirsthead
\multicolumn{13}{l}{\textit{Table \thetable\ (continued): LLaVA-NeXT refinement-layer sweeps}}\\
\toprule
$K$ & GQA & \shortstack{SQA\\IMG} & TextVQA & POPE & MME & \shortstack{MMB\\EN} & \shortstack{MMB\\CN} & SEED & AI2D & MMMU & Acc & Rel \\
\midrule
\endhead
\midrule
\multicolumn{13}{r}{\textit{Continued on next page}}\\
\endfoot
\bottomrule
\endlastfoot
\rowcolor{black!4}
\multicolumn{13}{c}{\textbf{LLaVA-NeXT-7B \quad $T=160$}}\\
\midrule
Vanilla & 62.5 & 67.5 & 60.3 & 86.8 & 1511.8 & 65.8 & 57.3 & 69.7 & 64.7 & 35.2 & 64.5 & 100.0\% \\
\midrule
1 & 59.4 & 68.0 & 54.3 & 82.3 & 1379.2 & 62.4 & 54.3 & 63.4 & 62.8 & \underline{36.8} & 61.3 & 94.9\% \\
2 & 59.5 & 67.3 & 55.9 & 81.9 & 1388.0 & 62.6 & 54.9 & 63.4 & 63.0 & \textbf{37.0} & 61.5 & 95.3\% \\
3 & 58.8 & 67.7 & 56.3 & 81.0 & 1442.7 & 62.8 & 55.6 & 63.5 & 62.6 & 35.3 & 61.6 & 95.4\% \\
4 & 59.1 & 68.5 & 56.1 & 81.5 & 1419.2 & 62.3 & 55.1 & 63.5 & 62.9 & 35.3 & 61.5 & 95.3\% \\
5 & 59.2 & 67.8 & 56.1 & 82.0 & 1420.9 & 63.0 & 55.9 & 63.3 & 62.4 & 34.8 & 61.6 & 95.4\% \\
6 & 59.1 & 67.9 & \underline{56.5} & 81.5 & 1442.7 & 63.4 & 55.8 & 63.7 & 63.0 & 34.8 & 61.8 & 95.7\% \\
\rowcolor{micoPurple}
7 & 59.5 & 68.2 & \textbf{56.8} & 83.8 & 1446.1 & 64.1 & \underline{56.6} & 64.4 & \underline{63.5} & 35.2 & 62.4 & 96.7\% \\
8 & 59.5 & \underline{68.6} & 56.1 & 84.2 & 1454.2 & 64.0 & \underline{56.6} & 64.2 & 63.3 & 35.7 & 62.5 & 96.8\% \\
9 & \underline{59.9} & \textbf{68.7} & 56.3 & 83.8 & 1434.7 & 63.4 & 55.9 & 64.3 & 63.3 & 35.7 & 62.3 & 96.5\% \\
10 & 59.6 & \underline{68.6} & 55.4 & 83.2 & 1432.8 & 63.7 & 55.9 & 64.4 & 63.4 & 35.1 & 62.1 & 96.2\% \\
11 & 59.6 & 67.6 & 55.6 & \textbf{85.5} & 1445.3 & 63.9 & 56.1 & 64.7 & \textbf{64.2} & 35.2 & 62.5 & 96.8\% \\
12 & \textbf{60.1} & 68.4 & 56.4 & \underline{85.0} & 1444.2 & 64.0 & 56.4 & \underline{65.4} & 63.4 & 35.2 & \underline{62.6} & \textbf{97.1\%} \\
13 & \textbf{60.1} & 68.4 & 56.0 & 84.4 & 1452.1 & \underline{64.2} & \underline{56.6} & 65.3 & 63.4 & 35.6 & \textbf{62.7} & \textbf{97.1\%} \\
14 & 59.4 & 68.5 & 54.2 & 83.6 & \underline{1465.3} & \textbf{64.3} & \underline{56.6} & 65.3 & 63.2 & 35.2 & 62.4 & 96.6\% \\
15 & 59.8 & 68.3 & 54.2 & 84.2 & \textbf{1467.6} & \textbf{64.3} & \textbf{57.2} & \textbf{65.5} & \underline{63.5} & 35.3 & \underline{62.6} & \underline{97.0\%} \\
\midrule
\rowcolor{black!4}
\multicolumn{13}{c}{\textbf{LLaVA-NeXT-7B \quad $T=320$}}\\
\midrule
Vanilla & 62.5 & 67.5 & 60.3 & 86.8 & 1511.8 & 65.8 & 57.3 & 69.7 & 64.7 & 35.2 & 64.5 & 100.0\% \\
\midrule
1 & 60.8 & 68.2 & 55.6 & 85.4 & 1407.0 & 64.1 & 55.7 & 65.5 & 64.1 & \underline{36.0} & 62.6 & 97.0\% \\
2 & 61.1 & 67.9 & 57.4 & 85.3 & 1420.2 & 65.6 & 57.0 & 65.7 & \textbf{64.9} & \textbf{36.4} & 63.2 & 98.0\% \\
3 & 60.4 & 67.5 & 57.3 & 84.6 & 1413.3 & 64.4 & 57.2 & 65.8 & 64.5 & 35.7 & 62.8 & 97.3\% \\
4 & 61.1 & 68.1 & 57.5 & 85.6 & 1419.0 & 64.5 & 57.5 & 65.8 & 64.5 & 35.1 & 63.1 & 97.7\% \\
5 & 60.9 & 67.3 & 57.4 & 85.2 & 1418.6 & 65.0 & 57.0 & 65.7 & 64.6 & 35.3 & 62.9 & 97.5\% \\
6 & 61.0 & 67.9 & 57.5 & 85.1 & 1437.6 & 65.6 & 57.4 & 65.9 & 64.5 & 35.8 & 63.3 & 98.0\% \\
\rowcolor{micoPurple}
7 & \underline{61.3} & 68.5 & \underline{58.3} & 86.4 & 1437.0 & 65.5 & 57.7 & 66.4 & \underline{64.8} & 35.1 & 63.6 & 98.5\% \\
8 & 61.1 & \underline{68.8} & \textbf{58.7} & \underline{86.9} & 1444.6 & 65.2 & 57.7 & 66.4 & 64.3 & 35.0 & 63.6 & 98.6\% \\
9 & \underline{61.3} & \underline{68.8} & 57.9 & 86.8 & 1443.9 & 65.1 & 57.5 & 66.6 & 64.6 & 34.6 & 63.5 & 98.4\% \\
10 & \textbf{61.5} & 68.4 & 57.2 & 85.8 & 1445.5 & 64.8 & 57.3 & 66.6 & 64.3 & 35.0 & 63.3 & 98.1\% \\
11 & \underline{61.3} & 68.1 & 58.0 & \textbf{87.5} & 1458.1 & 65.3 & 57.2 & 66.9 & 64.4 & 35.0 & 63.7 & 98.6\% \\
12 & \textbf{61.5} & 68.4 & 57.9 & \underline{86.9} & 1465.3 & \textbf{66.0} & \underline{57.9} & \underline{67.4} & \underline{64.8} & 35.6 & \textbf{64.0} & \textbf{99.1\%} \\
13 & \textbf{61.5} & \textbf{68.9} & 57.8 & 86.7 & 1455.9 & \textbf{66.0} & \textbf{58.2} & \underline{67.4} & 64.6 & 35.1 & \underline{63.9} & \underline{99.0\%} \\
14 & 60.9 & 68.7 & 56.6 & 86.1 & \textbf{1470.4} & \underline{65.9} & 57.6 & \underline{67.4} & 64.5 & 35.0 & 63.6 & 98.6\% \\
15 & 61.2 & 68.4 & 56.5 & 86.5 & \underline{1470.2} & 65.7 & \textbf{58.2} & \textbf{67.5} & \textbf{64.9} & 34.8 & 63.7 & 98.7\% \\
\midrule
\rowcolor{black!4}
\multicolumn{13}{c}{\textbf{LLaVA-NeXT-7B \quad $T=640$}}\\
\midrule
Vanilla & 62.5 & 67.5 & 60.3 & 86.8 & 1511.8 & 65.8 & 57.3 & 69.7 & 64.7 & 35.2 & 64.5 & 100.0\% \\
\midrule
1 & \textbf{62.3} & 68.3 & 56.5 & 87.3 & 1445.1 & 64.5 & 56.5 & 67.9 & 65.0 & 35.0 & 63.6 & 98.5\% \\
2 & 62.1 & 68.7 & 58.1 & 87.1 & 1455.7 & 65.3 & 56.9 & 68.2 & 64.7 & \textbf{35.3} & 63.9 & 99.0\% \\
3 & 62.0 & 68.5 & 58.0 & 87.1 & 1440.8 & 64.9 & 57.2 & 68.5 & 65.1 & \underline{35.1} & 63.8 & 98.9\% \\
4 & 62.1 & 68.9 & 57.8 & 87.3 & 1463.9 & 65.1 & \underline{58.2} & 68.4 & 65.5 & \textbf{35.3} & 64.2 & 99.4\% \\
5 & 61.9 & 68.9 & 57.9 & 87.1 & 1431.7 & 65.1 & 57.7 & 68.4 & 65.2 & 34.8 & 63.9 & 98.9\% \\
6 & 62.1 & 68.4 & 58.2 & 87.1 & 1455.9 & 65.1 & 57.7 & 68.1 & 65.2 & 35.0 & 64.0 & 99.1\% \\
\rowcolor{micoPurple}
7 & 61.9 & 68.2 & 58.9 & 87.3 & 1446.8 & 65.5 & 57.7 & 68.4 & 65.3 & 35.0 & 64.1 & 99.2\% \\
8 & 61.9 & 68.3 & \underline{59.2} & 87.8 & 1454.8 & 65.3 & 57.9 & 68.3 & \textbf{65.8} & 34.6 & 64.2 & 99.4\% \\
9 & \underline{62.2} & 68.7 & \textbf{59.3} & \underline{87.9} & 1436.3 & 65.4 & \underline{58.2} & 68.3 & \underline{65.6} & 34.9 & 64.2 & 99.5\% \\
10 & \underline{62.2} & \underline{69.1} & 58.2 & 87.5 & 1437.1 & 65.2 & 57.6 & 68.4 & \textbf{65.8} & 34.6 & 64.0 & 99.2\% \\
11 & 62.0 & 68.7 & 59.0 & \textbf{88.0} & 1465.9 & 65.6 & 58.1 & 68.7 & 65.5 & 34.1 & \underline{64.3} & 99.6\% \\
12 & \textbf{62.3} & \textbf{69.4} & \underline{59.2} & 87.8 & 1472.6 & 65.4 & \underline{58.2} & 69.0 & 65.3 & 34.0 & \textbf{64.4} & \textbf{99.8\%} \\
13 & \underline{62.2} & 68.9 & 58.9 & 87.5 & 1489.4 & \textbf{65.9} & 57.9 & \underline{69.2} & 65.5 & 33.8 & \textbf{64.4} & \textbf{99.8\%} \\
14 & 62.0 & 69.0 & 57.8 & 87.3 & \underline{1495.8} & \underline{65.8} & 57.7 & 69.1 & 65.0 & 34.0 & \underline{64.3} & 99.6\% \\
15 & 61.9 & 69.0 & 57.7 & 87.3 & \textbf{1504.0} & \textbf{65.9} & \textbf{58.4} & \textbf{69.3} & 65.3 & 33.7 & \textbf{64.4} & \underline{99.7\%} \\
\midrule
\rowcolor{black!4}
\multicolumn{13}{c}{\textbf{LLaVA-NeXT-13B \quad $T=160$}}\\
\midrule
Vanilla & 64.4 & 73.1 & 63.2 & 85.3 & 1539.5 & 68.5 & 61.2 & 71.6 & 70.1 & 35.7 & 67.0 & 100.0\% \\
\midrule
1 & 60.7 & \underline{72.3} & 56.9 & 83.8 & 1446.2 & 64.4 & 60.1 & 65.0 & 66.8 & 37.1 & 63.9 & 95.4\% \\
2 & 60.5 & 71.0 & \underline{58.4} & 82.3 & 1446.4 & 64.5 & 62.3 & 65.1 & 67.7 & 36.8 & 64.1 & 95.6\% \\
3 & 60.4 & 71.5 & 58.2 & 82.1 & 1453.7 & 64.9 & 61.9 & 65.4 & 67.7 & 37.1 & 64.2 & 95.8\% \\
4 & 60.4 & 71.1 & 58.2 & 82.7 & 1430.8 & 65.4 & 62.3 & 65.5 & 67.5 & 37.0 & 64.2 & 95.8\% \\
5 & 60.4 & 71.7 & 58.3 & 82.4 & 1469.6 & 65.4 & 62.5 & 65.8 & 67.6 & 37.3 & 64.5 & 96.3\% \\
6 & 60.8 & 72.2 & 58.3 & 83.1 & 1492.0 & \underline{66.3} & 62.3 & 66.0 & 67.8 & 36.6 & 64.8 & 96.7\% \\
7 & 60.7 & \textbf{72.4} & 57.9 & 83.7 & 1488.9 & 65.4 & 62.8 & 66.4 & 67.5 & \underline{37.4} & 64.9 & 96.8\% \\
8 & 61.1 & \underline{72.3} & 58.2 & 84.7 & 1482.5 & 65.7 & 62.3 & 66.5 & 67.4 & 36.9 & 64.9 & 96.9\% \\
9 & 61.6 & 72.0 & \textbf{58.5} & 85.0 & \underline{1500.6} & 66.1 & 63.0 & 67.3 & 67.7 & \textbf{37.6} & \textbf{65.4} & \textbf{97.6\%} \\
10 & 61.5 & 71.7 & 58.2 & \textbf{85.4} & \textbf{1505.0} & 65.7 & \underline{63.1} & 67.0 & 68.0 & 36.6 & \underline{65.3} & 97.4\% \\
11 & \underline{62.1} & 71.2 & \underline{58.4} & 84.9 & 1482.5 & 66.2 & 62.8 & 67.5 & 68.1 & 36.8 & 65.2 & 97.3\% \\
12 & 61.8 & 71.5 & \textbf{58.5} & 84.8 & 1491.9 & \underline{66.3} & 62.5 & 67.4 & 68.0 & 37.0 & \underline{65.3} & 97.4\% \\
13 & 61.6 & 71.6 & 58.2 & 84.8 & 1478.8 & \textbf{66.6} & \textbf{63.2} & 67.4 & 67.9 & 36.8 & 65.2 & 97.3\% \\
\rowcolor{micoPurple}
14 & 62.0 & 71.8 & \underline{58.4} & 85.0 & 1488.4 & 66.2 & \underline{63.1} & \underline{67.6} & 68.0 & 36.7 & \underline{65.3} & \underline{97.5\%} \\
15 & \textbf{62.2} & 71.7 & 58.2 & 85.0 & 1480.9 & 66.0 & 62.9 & \textbf{67.7} & 68.0 & 36.9 & \underline{65.3} & 97.4\% \\
16 & 62.0 & 71.7 & \textbf{58.5} & \underline{85.1} & 1486.4 & 66.0 & 62.8 & \textbf{67.7} & \underline{68.2} & 36.7 & \underline{65.3} & 97.4\% \\
17 & \underline{62.1} & 71.7 & \textbf{58.5} & 85.0 & 1489.6 & 66.0 & 63.0 & \textbf{67.7} & 68.1 & 36.9 & \underline{65.3} & \underline{97.5\%} \\
18 & \underline{62.1} & 71.5 & 58.3 & 84.9 & 1473.8 & 66.0 & 62.8 & \textbf{67.7} & \textbf{68.3} & 36.7 & 65.2 & 97.3\% \\
19 & 61.8 & 71.5 & 57.5 & 84.6 & 1482.6 & 66.0 & 62.8 & \textbf{67.7} & \underline{68.2} & 36.6 & 65.1 & 97.1\% \\
\midrule
\rowcolor{black!4}
\multicolumn{13}{c}{\textbf{LLaVA-NeXT-13B \quad $T=320$}}\\
\midrule
Vanilla & 64.4 & 73.1 & 63.2 & 85.3 & 1539.5 & 68.5 & 61.2 & 71.6 & 70.1 & 35.7 & 67.0 & 100.0\% \\
\midrule
1 & 62.3 & 71.4 & 59.1 & 85.4 & 1483.6 & 65.5 & 61.2 & 67.8 & 67.7 & \textbf{38.2} & 65.3 & 97.4\% \\
2 & 62.5 & 71.2 & 60.8 & 84.1 & 1506.8 & 66.6 & 62.4 & 68.1 & 67.8 & 37.4 & 65.6 & 97.9\% \\
3 & 62.4 & 71.0 & 60.6 & 84.5 & 1501.0 & 66.2 & 62.2 & 68.0 & 67.9 & 37.2 & 65.5 & 97.8\% \\
4 & 62.2 & 70.8 & 60.4 & 84.7 & 1507.8 & 66.2 & 62.0 & 67.8 & 67.8 & \underline{38.1} & 65.6 & 97.8\% \\
5 & 62.4 & 71.0 & 60.5 & 84.5 & 1507.5 & 66.5 & 62.8 & 68.2 & 67.5 & 37.9 & 65.6 & 98.0\% \\
6 & 62.4 & 71.7 & 60.5 & 85.3 & 1512.3 & 67.3 & 63.0 & 68.5 & 68.2 & 37.3 & 66.0 & 98.5\% \\
7 & 62.7 & \textbf{72.5} & 60.6 & 86.0 & 1513.7 & \textbf{67.9} & 63.0 & 68.7 & 68.2 & 37.1 & 66.2 & 98.9\% \\
8 & 62.8 & 71.9 & 60.1 & 86.2 & 1514.3 & 67.3 & 62.7 & 68.8 & 68.3 & \underline{38.1} & 66.2 & 98.8\% \\
9 & \textbf{63.3} & 72.0 & 60.6 & \underline{86.9} & 1526.7 & 67.6 & \underline{63.5} & 69.2 & 68.3 & 37.7 & \textbf{66.5} & \textbf{99.3\%} \\
10 & 62.8 & \underline{72.2} & 60.4 & \textbf{87.0} & 1541.0 & 67.7 & 62.7 & 69.3 & \underline{68.5} & 37.6 & \textbf{66.5} & \textbf{99.3\%} \\
11 & 63.0 & 71.6 & 60.4 & 86.7 & \textbf{1551.1} & 67.7 & \textbf{63.7} & \underline{69.4} & 68.2 & 37.0 & \textbf{66.5} & \textbf{99.3\%} \\
12 & \underline{63.2} & 71.4 & 60.6 & 86.5 & 1523.9 & 67.6 & \textbf{63.7} & 69.2 & \underline{68.5} & 36.9 & \underline{66.4} & 99.0\% \\
13 & 62.9 & 71.5 & 60.4 & 86.5 & 1533.5 & 67.7 & \underline{63.5} & 69.3 & 68.4 & 37.1 & \underline{66.4} & 99.1\% \\
\rowcolor{micoPurple}
14 & \textbf{63.3} & 71.3 & \underline{60.9} & 86.5 & 1543.3 & \textbf{67.9} & 63.2 & \textbf{69.5} & 68.3 & 37.0 & \textbf{66.5} & \textbf{99.3\%} \\
15 & \underline{63.2} & 71.8 & 60.7 & 86.5 & 1539.2 & \underline{67.8} & 63.4 & \underline{69.4} & 68.4 & 36.9 & \textbf{66.5} & \underline{99.2\%} \\
16 & 63.0 & 71.4 & 60.6 & 86.7 & 1537.1 & 67.7 & 63.1 & \textbf{69.5} & 68.4 & 36.6 & \underline{66.4} & 99.1\% \\
17 & 63.1 & 71.7 & \textbf{61.2} & 86.6 & 1536.6 & 67.4 & 63.1 & \textbf{69.5} & \textbf{68.6} & 37.2 & \textbf{66.5} & \textbf{99.3\%} \\
18 & 62.9 & 71.6 & 60.4 & 86.4 & 1535.8 & 67.3 & 63.1 & \textbf{69.5} & \underline{68.5} & 37.4 & \underline{66.4} & 99.1\% \\
19 & 63.0 & 71.5 & 59.8 & 86.5 & \underline{1547.4} & 67.5 & 63.1 & \textbf{69.5} & \underline{68.5} & 36.3 & 66.3 & 99.0\% \\
\midrule
\rowcolor{black!4}
\multicolumn{13}{c}{\textbf{LLaVA-NeXT-13B \quad $T=640$}}\\
\midrule
Vanilla & 64.4 & 73.1 & 63.2 & 85.3 & 1539.5 & 68.5 & 61.2 & 71.6 & 70.1 & 35.7 & 67.0 & 100.0\% \\
\midrule
1 & 63.6 & 71.7 & 60.2 & 86.2 & 1534.2 & 66.6 & 62.9 & 70.2 & 68.6 & \textbf{38.7} & 66.5 & 99.3\% \\
2 & 63.6 & 72.4 & 61.1 & 85.9 & 1514.2 & 67.9 & 62.5 & 70.4 & 69.1 & 37.6 & 66.6 & 99.4\% \\
3 & 63.4 & 72.2 & 61.3 & 85.6 & 1536.8 & 67.3 & 62.7 & 70.8 & 69.3 & 37.0 & 66.7 & 99.5\% \\
4 & 63.5 & 72.3 & 60.7 & 86.0 & 1537.3 & 68.0 & 63.0 & 70.3 & 69.0 & 37.7 & 66.7 & 99.6\% \\
5 & 63.4 & 72.3 & 61.2 & 85.9 & 1552.9 & 67.6 & 62.7 & 70.5 & 68.9 & 37.7 & 66.8 & 99.7\% \\
6 & 63.5 & 72.5 & 61.4 & 86.4 & 1543.4 & 68.6 & 62.9 & 70.4 & 69.1 & 37.6 & 67.0 & 99.9\% \\
7 & 63.7 & 72.7 & 61.4 & 86.8 & 1531.1 & 68.5 & 63.0 & 70.7 & 69.5 & \underline{37.8} & 67.1 & 100.1\% \\
8 & 63.7 & 72.5 & 61.2 & 86.9 & 1524.7 & 68.4 & 63.3 & 70.4 & 69.4 & 37.6 & 67.0 & 99.9\% \\
9 & 64.0 & 73.1 & 61.6 & \underline{87.1} & 1547.1 & 68.3 & 63.2 & 71.1 & 69.8 & 37.6 & 67.3 & 100.4\% \\
10 & 63.7 & 73.1 & 61.7 & \textbf{87.2} & 1545.8 & \underline{69.0} & 63.4 & 71.1 & 69.8 & 36.4 & 67.3 & 100.4\% \\
11 & 64.1 & \textbf{73.3} & 61.6 & \underline{87.1} & 1543.8 & 68.8 & \underline{63.5} & 71.2 & 69.9 & 36.4 & 67.3 & 100.5\% \\
12 & 64.1 & \textbf{73.3} & \textbf{61.9} & 86.8 & 1543.7 & 68.6 & 63.4 & 71.2 & 69.8 & 37.1 & 67.3 & 100.5\% \\
13 & 64.1 & 72.4 & \underline{61.8} & 87.0 & 1567.0 & \textbf{69.1} & 62.7 & 71.2 & 69.9 & 37.0 & 67.4 & 100.5\% \\
\rowcolor{micoPurple}
14 & \textbf{64.3} & \underline{73.2} & \textbf{61.9} & 87.0 & \textbf{1573.9} & \underline{69.0} & 63.4 & \underline{71.4} & 70.2 & 37.2 & \textbf{67.6} & \textbf{100.9\%} \\
15 & \underline{64.2} & 73.1 & 61.6 & 87.0 & 1558.0 & \underline{69.0} & 63.4 & \underline{71.4} & 70.1 & 37.6 & \underline{67.5} & \underline{100.8\%} \\
16 & 64.1 & 73.0 & 61.6 & 86.9 & \underline{1567.1} & 68.7 & \textbf{63.6} & \underline{71.4} & 70.2 & 37.4 & \underline{67.5} & \underline{100.8\%} \\
17 & \underline{64.2} & 73.1 & 61.6 & 86.8 & 1558.8 & \underline{69.0} & \underline{63.5} & \textbf{71.5} & 70.2 & 37.7 & \underline{67.5} & \underline{100.8\%} \\
18 & 64.0 & \underline{73.2} & 61.5 & 86.8 & 1566.4 & 68.8 & 63.4 & \textbf{71.5} & \underline{70.3} & 37.4 & \underline{67.5} & \underline{100.8\%} \\
19 & 63.9 & 73.1 & 60.7 & 86.8 & 1565.7 & 68.8 & \underline{63.5} & \textbf{71.5} & \textbf{70.4} & 37.3 & 67.4 & 100.6\% \\
\end{longtable}
\endgroup

\FloatBarrier

\section{Supplementary Benchmark Results}
\label{sec:appendix_specialized_results}

\subsection{VTC-Bench: Fixed-Reference Groups}
\label{sec:appendix_vtc_bench}

VTC-Bench evaluates visual-token compression by separating questions according to their sensitivity to image downsampling~\citep{liao2026vtcbench}. We use the author-released groups constructed with Qwen2-VL-7B-Instruct. For each reference downsampling factor $d$, Group A contains samples answered correctly at full resolution but incorrectly after downsampling, while Group B contains samples answered correctly in both settings. These groups are reused unchanged for Qwen2.5-VL-7B.

Tables~\ref{tab:vtc_bench_qwen25_t256} and~\ref{tab:vtc_bench_qwen25_t128} report Qwen2.5-VL-7B at $T=256$ and $128$, respectively. Following the original VTC-Bench layout, each cell shows Group A (Group B) on one line and their absolute gap $\Delta$ below. Scores and gaps are rounded independently from the unrounded results to one decimal place; gaps are in percentage points. Each table is organized by reference image-area reduction, with the two groups ranked independently. Both tables include all eight task scores, their unweighted mean Acc, and Rel normalized by vanilla for the same model, group, and $d$. Task scores, Acc, and Rel are percentages; Rel is computed before rounding, and score highlighting ranks distinct displayed values within each $T,d$ setting. MME uses per-question correctness rather than the native summed score, and MMBench uses per-question correctness rather than circular evaluation.

The reference factors $d=2,3,4,5,10$ correspond to image-area reductions of $75.00\%,88.89\%,93.75\%,96.00\%,99.00\%$ when constructing the groups; they do not specify the pruning ratios of our target models. We evaluate the target model on the original images with token budgets $T=256$ and $128$, using greedy decoding and a maximum of 2,048 generated tokens. Empty and length-truncated responses remain in the evaluation denominators.

% Generated from the original Group A/B CSVs; round once to one decimal.
% Within each cell: A (B), followed by the absolute unrounded-score gap.
% Rankings use distinct displayed scores among pruners, separately for A/B.
\begingroup
\scriptsize
\setlength{\tabcolsep}{1.5pt}
\setlength{\fboxsep}{1pt}
\setlength{\LTleft}{0pt}
\setlength{\LTright}{0pt}
\setlength{\LTcapwidth}{\linewidth}
\renewcommand{\arraystretch}{1.08}
\begin{longtable}{>{\raggedright\arraybackslash}m{\dimexpr0.095\linewidth-2\tabcolsep\relax}
*{9}{>{\centering\arraybackslash}m{\dimexpr0.0885\linewidth-2\tabcolsep\relax}}
>{\centering\arraybackslash}m{\dimexpr0.1085\linewidth-2\tabcolsep\relax}}
\caption{VTC-Bench scores (\%) on Qwen2.5-VL-7B at $T=256$. Cells show Group A (Group B), with the absolute gap $\Delta$ in percentage points below.}
\label{tab:vtc_bench_qwen25_t256}\\
\toprule
Method & \shortstack{Chart\\QA} & GQA & \shortstack{MMB\\CN} & \shortstack{MMB\\EN} & MME & MMStar & \shortstack{OCR\\Bench} & POPE & Acc & Rel \\
\midrule
\endfirsthead
\multicolumn{11}{l}{\textit{Table \thetable\ (continued): Qwen2.5-VL-7B, $T=256$, Group A (Group B)}}\\
\toprule
Method & \shortstack{Chart\\QA} & GQA & \shortstack{MMB\\CN} & \shortstack{MMB\\EN} & MME & MMStar & \shortstack{OCR\\Bench} & POPE & Acc & Rel \\
\midrule
\endhead
\midrule
\multicolumn{11}{r}{\textit{Continued on next page}}\\
\endfoot
\bottomrule
\endlastfoot
\rowcolor{black!4}
\multicolumn{11}{c}{\textbf{Reference image-area reduction: 75.00\%}} \\
Vanilla & \shortstack{88.5 (94.3)\\\colorbox{black!6}{\textcolor{black!65}{$\Delta\,5.8$}}} & \shortstack{65.0 (88.2)\\\colorbox{black!6}{\textcolor{black!65}{$\Delta\,23.2$}}} & \shortstack{67.4 (94.7)\\\colorbox{black!6}{\textcolor{black!65}{$\Delta\,27.3$}}} & \shortstack{76.5 (97.9)\\\colorbox{black!6}{\textcolor{black!65}{$\Delta\,21.4$}}} & \shortstack{70.4 (93.9)\\\colorbox{black!6}{\textcolor{black!65}{$\Delta\,23.5$}}} & \shortstack{75.4 (86.7)\\\colorbox{black!6}{\textcolor{black!65}{$\Delta\,11.3$}}} & \shortstack{83.8 (90.4)\\\colorbox{black!6}{\textcolor{black!65}{$\Delta\,6.6$}}} & \shortstack{81.0 (96.7)\\\colorbox{black!6}{\textcolor{black!65}{$\Delta\,15.6$}}} & \shortstack{76.0 (92.8)\\\colorbox{black!6}{\textcolor{black!65}{$\Delta\,16.8$}}} & \shortstack{100.0 (100.0)\\\colorbox{black!6}{\textcolor{black!65}{$\Delta\,0.0$}}} \\
\midrule
FastV & \shortstack{\underline{57.8} (\underline{83.7})\\\colorbox{black!6}{\textcolor{black!65}{$\Delta\,25.9$}}} & \shortstack{54.7 (77.3)\\\colorbox{black!6}{\textcolor{black!65}{$\Delta\,22.6$}}} & \shortstack{58.7 (93.5)\\\colorbox{black!6}{\textcolor{black!65}{$\Delta\,34.8$}}} & \shortstack{65.9 (95.9)\\\colorbox{black!6}{\textcolor{black!65}{$\Delta\,30.0$}}} & \shortstack{62.0 (91.8)\\\colorbox{black!6}{\textcolor{black!65}{$\Delta\,29.9$}}} & \shortstack{55.2 (76.2)\\\colorbox{black!6}{\textcolor{black!65}{$\Delta\,21.0$}}} & \shortstack{43.0 (64.8)\\\colorbox{black!6}{\textcolor{black!65}{$\Delta\,21.8$}}} & \shortstack{65.8 (91.5)\\\colorbox{black!6}{\textcolor{black!65}{$\Delta\,25.7$}}} & \shortstack{57.9 (84.3)\\\colorbox{black!6}{\textcolor{black!65}{$\Delta\,26.5$}}} & \shortstack{76.1 (90.8)\\\colorbox{black!6}{\textcolor{black!65}{$\Delta\,14.7$}}} \\
\end{longtable}
\endgroup
\begingroup
\scriptsize
\setlength{\tabcolsep}{1.5pt}
\setlength{\fboxsep}{1pt}
\setlength{\LTleft}{0pt}
\setlength{\LTright}{0pt}
\setlength{\LTcapwidth}{\linewidth}
\renewcommand{\arraystretch}{1.08}
\begin{longtable}{>{\raggedright\arraybackslash}m{\dimexpr0.095\linewidth-2\tabcolsep\relax}
*{9}{>{\centering\arraybackslash}m{\dimexpr0.0885\linewidth-2\tabcolsep\relax}}
>{\centering\arraybackslash}m{\dimexpr0.1085\linewidth-2\tabcolsep\relax}}
\caption{VTC-Bench scores (\%) on Qwen2.5-VL-7B at $T=128$. Cells show Group A (Group B), with the absolute gap $\Delta$ in percentage points below.}
\label{tab:vtc_bench_qwen25_t128}\\
\toprule
Method & \shortstack{Chart\\QA} & GQA & \shortstack{MMB\\CN} & \shortstack{MMB\\EN} & MME & MMStar & \shortstack{OCR\\Bench} & POPE & Acc & Rel \\
\midrule
\endfirsthead
\multicolumn{11}{l}{\textit{Table \thetable\ (continued): Qwen2.5-VL-7B, $T=128$, Group A (Group B)}}\\
\toprule
Method & \shortstack{Chart\\QA} & GQA & \shortstack{MMB\\CN} & \shortstack{MMB\\EN} & MME & MMStar & \shortstack{OCR\\Bench} & POPE & Acc & Rel \\
\midrule
\endhead
\midrule
\multicolumn{11}{r}{\textit{Continued on next page}}\\
\endfoot
\bottomrule
\endlastfoot
\rowcolor{black!4}
\multicolumn{11}{c}{\textbf{Reference image-area reduction: 75.00\%}} \\
Vanilla & \shortstack{88.5 (94.3)\\\colorbox{black!6}{\textcolor{black!65}{$\Delta\,5.8$}}} & \shortstack{65.0 (88.2)\\\colorbox{black!6}{\textcolor{black!65}{$\Delta\,23.2$}}} & \shortstack{67.4 (94.7)\\\colorbox{black!6}{\textcolor{black!65}{$\Delta\,27.3$}}} & \shortstack{76.5 (97.9)\\\colorbox{black!6}{\textcolor{black!65}{$\Delta\,21.4$}}} & \shortstack{70.4 (93.9)\\\colorbox{black!6}{\textcolor{black!65}{$\Delta\,23.5$}}} & \shortstack{75.4 (86.7)\\\colorbox{black!6}{\textcolor{black!65}{$\Delta\,11.3$}}} & \shortstack{83.8 (90.4)\\\colorbox{black!6}{\textcolor{black!65}{$\Delta\,6.6$}}} & \shortstack{81.0 (96.7)\\\colorbox{black!6}{\textcolor{black!65}{$\Delta\,15.6$}}} & \shortstack{76.0 (92.8)\\\colorbox{black!6}{\textcolor{black!65}{$\Delta\,16.8$}}} & \shortstack{100.0 (100.0)\\\colorbox{black!6}{\textcolor{black!65}{$\Delta\,0.0$}}} \\
\midrule
FastV & \shortstack{17.7 (40.1)\\\colorbox{black!6}{\textcolor{black!65}{$\Delta\,22.4$}}} & \shortstack{37.3 (60.8)\\\colorbox{black!6}{\textcolor{black!65}{$\Delta\,23.5$}}} & \shortstack{43.5 (82.3)\\\colorbox{black!6}{\textcolor{black!65}{$\Delta\,38.8$}}} & \shortstack{49.3 (82.5)\\\colorbox{black!6}{\textcolor{black!65}{$\Delta\,33.2$}}} & \shortstack{50.7 (77.8)\\\colorbox{black!6}{\textcolor{black!65}{$\Delta\,27.1$}}} & \shortstack{34.4 (55.9)\\\colorbox{black!6}{\textcolor{black!65}{$\Delta\,21.5$}}} & \shortstack{15.1 (26.9)\\\colorbox{black!6}{\textcolor{black!65}{$\Delta\,11.8$}}} & \shortstack{49.6 (69.8)\\\colorbox{black!6}{\textcolor{black!65}{$\Delta\,20.2$}}} & \shortstack{37.2 (62.0)\\\colorbox{black!6}{\textcolor{black!65}{$\Delta\,24.8$}}} & \shortstack{48.9 (66.8)\\\colorbox{black!6}{\textcolor{black!65}{$\Delta\,17.8$}}} \\
\end{longtable}
\endgroup

\FloatBarrier

\subsection{Generalized Referring Expression Comprehension}
\label{sec:appendix_grefcoco}
Table~\ref{tab:ablation_grefcoco} extends the evaluation to gRefCOCO~\cite{he2023grec}, which includes expressions referring to single, multiple, or no target objects. We report $\mathrm{Pr}@(F_1=1,\,\mathrm{GIoU}\geq0.5)$ on val, testA, and testB, computed with the released GREC code using a generalized-IoU matching threshold (a variant of the original $\mathrm{IoU}\geq0.5$ definition, see Appendix~\ref{sec:appendix_supplementary_image_benchmarks}). MiCo leads the displayed pruners in all six model--budget--split settings on Qwen2.5 and all four testA/testB settings on Qwen3-VL. At $T=128$ on Qwen3-VL, it reaches 36.3/40.0 on testA/testB, versus VisPruner's 31.8/34.7; FastV remains strongest on Qwen3-VL val at both budgets.
Parenthesized values report retention relative to the matching Vanilla score.

\begin{table}[!htbp]
    \centering
    % Source: author-supplied completed Feishu gRefCOCO/GREC results, 2026-09-20.
% 18 model/configuration groups, complete 49,492-prediction records per group;
% author-reported validation protocol: official-giou-flat-box-v2.
% Preserve the supplied score and Rel precision. Do not recalculate Rel from
% rounded display scores. val/testA/testB are splits of one dataset.
\centering
\footnotesize
\setlength{\tabcolsep}{4pt}
\renewcommand{\arraystretch}{1.05}
\caption{GREC scores (\%) on gRefCOCO.}
\label{tab:ablation_grefcoco}
\resizebox{\linewidth}{!}{%
\begin{tabular}{lcccccc}
\toprule
& \multicolumn{3}{c}{Qwen2.5-VL-7B} & \multicolumn{3}{c}{Qwen3-VL-8B} \\
\cmidrule(lr){2-4}\cmidrule(lr){5-7}
Method & val & testA & testB & val & testA & testB \\
\midrule
Vanilla & 66.8 (100.0\%) & 47.3 (100.0\%) & 51.9 (100.0\%) & 49.1 (100.0\%) & 47.6 (100.0\%) & 47.8 (100.0\%) \\
\midrule
\multicolumn{7}{c}{\textit{Retain 256 tokens (80.2\% pruned)}} \\
\midrule
FastV (ECCV'24) & 61.9 (92.8\%) & 27.4 (57.9\%) & 34.6 (66.7\%) & \textbf{55.3} (112.7\%) & 38.0 (79.9\%) & 39.6 (82.8\%) \\
VisPruner (ICCV'25) & \underline{64.0} (95.8\%) & \underline{35.8} (75.8\%) & \underline{40.8} (78.5\%) & 52.0 (106.0\%) & \underline{41.7} (87.5\%) & \underline{42.9} (89.6\%) \\
ApET (CVPR'26) & 62.3 (93.3\%) & 27.6 (58.3\%) & 33.0 (63.6\%) & 50.0 (101.9\%) & 31.4 (66.0\%) & 35.1 (73.4\%) \\
\rowcolor{micoPurple}
\textbf{MiCo} & \textbf{64.9} (97.2\%) & \textbf{38.8} (82.1\%) & \textbf{44.1} (84.8\%) & \underline{52.7} (107.4\%) & \textbf{43.8} (92.0\%) & \textbf{45.1} (94.3\%) \\
\midrule
\multicolumn{7}{c}{\textit{Retain 128 tokens (90.1\% pruned)}} \\
\midrule
FastV (ECCV'24) & \underline{62.4} (93.5\%) & 23.2 (49.0\%) & 29.1 (56.0\%) & \textbf{60.1} (122.4\%) & 26.5 (55.6\%) & 30.3 (63.3\%) \\
VisPruner (ICCV'25) & 62.2 (93.2\%) & \underline{29.3} (62.1\%) & \underline{33.2} (63.9\%) & \underline{54.3} (110.5\%) & \underline{31.8} (66.8\%) & \underline{34.7} (72.6\%) \\
ApET (CVPR'26) & 61.8 (92.5\%) & 24.5 (51.9\%) & 29.6 (56.9\%) & 52.4 (106.8\%) & 26.4 (55.4\%) & 30.1 (62.8\%) \\
\rowcolor{micoPurple}
\textbf{MiCo} & \textbf{62.5} (93.6\%) & \textbf{31.0} (65.6\%) & \textbf{35.2} (67.8\%) & 53.7 (109.4\%) & \textbf{36.3} (76.2\%) & \textbf{40.0} (83.6\%) \\
\bottomrule
\end{tabular}%
}

\end{table}

\subsection{Scene Text, Chart, and Document Understanding}
\label{sec:appendix_qwen_reading}

Table~\ref{tab:qwen_reading_benchmarks_full} gives the complete TextVQA, ChartQA, and DocVQA comparison; the main text reports a compact view with FastV, VisPruner, and ApET.
% Sources read 2026-09-22:
% Feishu Docx WH5rdMOmEoTMKPxKoyGcuCTMnvd, revision 601,
% ChartQA/DocVQA native table doxcnmqvy6lAAphGldo3HBxTvnb.
% TextVQA: embedded spreadsheet ZWr6sFcIphhA49tbghjc7ESunkh,
% sheet pMs3TL, A1:F45 (native precision retained in the CSV below).
% Data snapshot: tables/data/qwen_reading_benchmarks.csv.
\begin{table}[H]
\centering
\footnotesize
\setlength{\tabcolsep}{3.5pt}
\renewcommand{\arraystretch}{1.0}
\caption{Text (TextVQA), Chart (ChartQA), and Doc (DocVQA) scores (\%) on Qwen2.5-VL-7B and Qwen3-VL-8B at $T=256/128$.}
\label{tab:qwen_reading_benchmarks_full}
\micoTableFit{%
\begin{tabular}{l*{12}{c}}
\toprule
& \multicolumn{6}{c}{Qwen2.5-VL-7B} & \multicolumn{6}{c}{Qwen3-VL-8B} \\
\cmidrule(lr){2-7}\cmidrule(lr){8-13}
& \multicolumn{3}{c}{$T=256$} & \multicolumn{3}{c}{$T=128$}
& \multicolumn{3}{c}{$T=256$} & \multicolumn{3}{c}{$T=128$} \\
\cmidrule(lr){2-4}\cmidrule(lr){5-7}\cmidrule(lr){8-10}\cmidrule(lr){11-13}
Method & Text & Chart & Doc
& Text & Chart & Doc
& Text & Chart & Doc
& Text & Chart & Doc \\
\midrule
Vanilla & 85.06 & 87.28 & 94.86 & 85.06 & 87.28 & 94.86 & 84.28 & 83.04 & 95.72 & 84.28 & 83.04 & 95.72 \\
\midrule
FastV & \underline{78.44} & 68.28 & 38.22 & 46.45 & 33.20 & 14.54 & 62.18 & 28.64 & 14.79 & 32.88 & 16.24 & 9.83 \\
SparseVLM & 64.23 & 40.08 & 12.50 & 39.48 & 19.24 & 9.50 & 66.72 & 11.08 & 17.05 & 41.29 & 20.56 & 11.17 \\
DivPrune & 76.55 & \underline{70.44} & \underline{51.70} & \underline{66.69} & \underline{52.16} & \underline{35.56} & 74.61 & 59.32 & 47.91 & 64.72 & 43.76 & 34.14 \\
VisionZip & 61.87 & 55.56 & 27.44 & 49.50 & 36.64 & 20.01 & 74.68 & 69.76 & 49.68 & 57.88 & 47.00 & 25.10 \\
VisPruner & 64.02 & 56.08 & 27.66 & 49.40 & 36.68 & 19.69 & 71.96 & 63.52 & 39.63 & 54.58 & 39.64 & 23.00 \\
HoloV & 63.19 & 54.68 & 28.19 & 49.05 & 36.48 & 19.19 & 74.63 & \textbf{71.40} & 50.40 & 57.27 & \textbf{51.20} & 27.73 \\
MMTok & 75.29 & 67.68 & 39.31 & 61.23 & 45.08 & 22.25 & \underline{76.93} & \underline{70.84} & \textbf{66.95} & \underline{68.11} & 48.68 & \underline{45.77} \\
ApET & 69.17 & 50.60 & 23.12 & 45.82 & 31.16 & 14.35 & 59.54 & 54.52 & 39.25 & 41.27 & 35.44 & 25.59 \\
\midrule
\rowcolor{micoPurple}
\textbf{MiCo} & \textbf{80.96} & \textbf{81.16} & \textbf{57.44} & \textbf{71.74} & \textbf{70.68} & \textbf{38.58} & \textbf{79.69} & 66.80 & \underline{64.98} & \textbf{73.44} & \underline{50.12} & \textbf{46.35} \\
\bottomrule
\end{tabular}%
}
\par\vspace{3pt}
\begin{minipage}{\linewidth}
\footnotesize
\textit{Metrics.} TextVQA\_VAL: VQA consensus score over 5,000 questions; ChartQA\_TEST: relaxed accuracy over 2,500 questions; DocVQA\_VAL: ANLS over 5,349 questions. Each model's unpruned reference is repeated under both budgets.
\end{minipage}
\end{table}

\FloatBarrier
\subsection{OCRBench}
\label{sec:appendix_ocrbench}

Table~\ref{tab:ablation_ocrbench} reports the complete OCRBench comparison on four image models; this benchmark is not tabulated in the main text.

\begin{table}[!htbp]
    \centering
    % Source: ccfa-workfiles/experiments/feishu-table-sync/cache/ocrbench_ebKF4Y.json, revision 1176.
\centering
\footnotesize
\setlength{\tabcolsep}{3pt}
\renewcommand{\arraystretch}{1.08}
\caption{OCRBench scores (0--100, higher is better) on LLaVA-NeXT-7B/13B, Qwen2.5-VL-7B, and InternVL3-8B at token budget $T$.}
\label{tab:ablation_ocrbench}
\begin{tabular}{p{\dimexpr0.20\linewidth-2\tabcolsep\relax}
  *{6}{>{\centering\arraybackslash}p{\dimexpr0.07\linewidth-2\tabcolsep\relax}}
  *{2}{>{\centering\arraybackslash}p{\dimexpr0.105\linewidth-2\tabcolsep\relax}}
  *{2}{>{\centering\arraybackslash}p{\dimexpr0.085\linewidth-2\tabcolsep\relax}}}
\toprule
& \multicolumn{3}{c}{LLaVA-NeXT-7B} & \multicolumn{3}{c}{LLaVA-NeXT-13B} & \multicolumn{2}{c}{Qwen2.5-VL-7B} & \multicolumn{2}{c}{InternVL3-8B} \\
\cmidrule(lr){2-4}\cmidrule(lr){5-7}\cmidrule(lr){8-9}\cmidrule(lr){10-11}
Method / $T$ & 640 & 320 & 160 & 640 & 320 & 160 & 256 & 128 & 256 & 128 \\
\midrule
Vanilla & \multicolumn{3}{c}{50.7} & \multicolumn{3}{c}{51.1} & \multicolumn{2}{c}{87.7} & \multicolumn{2}{c}{88.1} \\
\midrule
FastV & 37.0 & 11.9 & -- & 38.9 & 18.9 & -- & 70.6 & 39.6 & \underline{55.0} & 25.0 \\
DivPrune & 38.2 & 30.1 & 23.8 & 39.5 & 31.7 & 26.2 & \underline{75.1} & \underline{66.7} & 44.9 & 32.5 \\
VisPruner & \textbf{43.6} & \underline{34.2} & \underline{25.8} & \underline{47.8} & \underline{37.9} & \underline{31.6} & 71.5 & 65.5 & 26.9 & 6.9 \\
MMTok & \underline{40.8} & 30.4 & 24.9 & 42.3 & 34.1 & 26.3 & 71.1 & 62.7 & 51.7 & 35.3 \\
ApET & 35.8 & 29.0 & 19.5 & 37.9 & 29.6 & 23.4 & 67.6 & 55.5 & 49.0 & \underline{36.6} \\
\rowcolor{micoPurple}
\textbf{MiCo} & \textbf{43.6} & \textbf{37.3} & \textbf{30.3} & \textbf{49.8} & \textbf{46.0} & \textbf{37.7} & \textbf{77.7} & \textbf{68.3} & \textbf{65.3} & \textbf{49.7} \\
\bottomrule
\end{tabular}
\par\vspace{2pt}
\begin{minipage}{\linewidth}
\footnotesize
\textit{Note.} On LLaVA-NeXT, VisPruner and ApET follow their authors' per-crop budgets of 128/64/32 tokens for $T=640/320/160$, so their realized layer-average token counts are not exactly matched to the other methods.
\end{minipage}

\end{table}

\FloatBarrier

\section{Limitations}
\label{sec:appendix_limitations}

\paragraph{Proxy and hyperparameter choices.}
MiCo estimates information, reliability, and relevance with simple statistics that are readily
available in the backbone, and uses one refinement layer and one Stage-1 pool size per
architecture. These choices work well in practice; more refined estimators and input-adaptive
settings are interesting directions for future work.

\paragraph{Selection cost.}
The coverage selector adds a small amount of computation for pairwise similarities. This overhead
is modest in our experiments and could be reduced further with approximate similarity
computation for very long inputs.

\paragraph{Evaluation scope.}
We focus on standard image and video benchmarks with short-answer prompts. Broader settings
such as long-form generation and interactive use are left for future work.

\end{document}